\pdfoutput=1
\documentclass{article}

\PassOptionsToPackage{numbers, sort&compress}{natbib}

\usepackage[main, final]{neurips_2026}

\usepackage[utf8]{inputenc} %
\usepackage[T1]{fontenc}    %
\usepackage{hyperref}       %
\usepackage{url}            %
\usepackage{booktabs}       %
\usepackage{array}          %
\usepackage{tabularx}
\newcolumntype{P}[1]{>{\raggedright\arraybackslash}p{#1}}
\newcolumntype{Y}{>{\raggedright\arraybackslash}X}
\newlength{\dashhdr}\newcommand{\dashunder}[1]{\settowidth{\dashhdr}{#1}\makebox[\dashhdr]{---}}
\newlength{\guarw}\newcommand{\guar}[2]{\settowidth{\guarw}{in exp.,}\makebox[\guarw][l]{#1,}\ #2}
\usepackage{amsfonts}       %
\usepackage{nicefrac}       %
\usepackage{microtype}      %
\usepackage{xcolor}         %
\definecolor{linknavy}{HTML}{1A3D6D}
\hypersetup{colorlinks=true, linkcolor=linknavy, citecolor=linknavy, urlcolor=linknavy,
  pdftitle={Sharp Convergence and Sample Complexity of Policy Mirror Descent for Average-Reward MDPs},
  pdfauthor={Enes Arda, Atilla Eryilmaz}}

\usepackage{amsmath}
\usepackage{amssymb}
\usepackage{amsthm}
\usepackage{mathtools}

\usepackage{graphicx}
\usepackage{float}
\usepackage{enumitem}
\usepackage{subcaption}
\usepackage{algorithm}
\usepackage{algorithmic}
\usepackage[capitalize,noabbrev]{cleveref}
\usepackage{enumitem}

\DeclareMathOperator*{\argmin}{arg\,min}

\newcommand{\E}{\mathbb{E}}

\newcommand{\St}{\mathcal{S}}

\theoremstyle{plain}
\newtheorem{theorem}{Theorem}[section]
\newtheorem{proposition}[theorem]{Proposition}
\newtheorem{lemma}[theorem]{Lemma}
\newtheorem{fact}[theorem]{Fact}
\newtheorem{corollary}[theorem]{Corollary}
\theoremstyle{definition}

\newtheorem{assumption}[theorem]{Assumption}
\theoremstyle{remark}
\newtheorem{remark}[theorem]{Remark}

\title{Sharp Convergence and Sample Complexity of Policy Mirror Descent for Average-Reward MDPs}

\author{%
  Enes Arda \\
  The Ohio State University \\
  \texttt{arda.2@osu.edu} \\
  \And
  Atilla Eryilmaz \\
  The Ohio State University \\
  \texttt{eryilmaz.2@osu.edu} \\
}

\begin{document}

\maketitle

\begin{abstract}
Policy mirror descent (PMD) has a mature finite-time theory in discounted Markov decision processes (MDPs), but less is known in the average-reward setting, a more natural objective for many control applications. We give a finite-time, finite-sample analysis of PMD in ergodic average-reward MDPs built around a single master recursion that governs convergence for any critic, without external regularization. Its specializations yield linear rates for exact, inexact-tabular, and linear function approximation (LFA) updates. For exact PMD, the rate is tight in the worst case, and convergence can become superlinear. We complement these convergence results with critics whose end-to-end sample complexities improve the prior best policy-gradient dependence from $|\mathcal S||\mathcal A|\,t_{\mathrm{mix}}^3$ to $|\mathcal S|\,t_{\mathrm{mix}}^3+|\mathcal S||\mathcal A|\,t_{\mathrm{mix}}^2$ in the tabular setting and from $|\mathcal S||\mathcal A|\,t_{\mathrm{mix}}^5$ to $d\,t_{\mathrm{mix}}^3$ under LFA. Matching information-theoretic lower bounds establish that both critics are optimal up to logarithmic factors.
\end{abstract}

\section{Introduction} \label{sec:intro}
Policy optimization, the process of updating an agent's decision rule to maximize expected long-term reward, is a workhorse of reinforcement learning applications \citep[e.g.,][]{adaGeneralizationTransferLearning2022, zamir_deep_2024}.
It underlies classical policy-gradient \citep{sutton_policy_1999} and natural-gradient methods \citep{kakade_natural_2001} as well as widely used algorithms such as TRPO \citep{schulman_trust_2017} and PPO \citep{schulman_proximal_2017}. A large and rapidly growing theory explains the behavior of these methods in \emph{discounted} Markov decision processes (MDPs), including finite-time finite-sample guarantees \citep{shani_adaptive_2020, agarwal_theory_2020, lan_policy_2022}, function approximation extensions \citep{agarwal_theory_2020, yuan_linear_2023, alfano_novel_2023, cayci_linear_2024}, and linear convergence results for policy mirror descent (PMD) and its natural-gradient instantiations \citep{cen_fast_2021, xiao_convergence_2022, lan_policy_2022, johnson_optimal_2023, zhan_policy_2023}. However, many continuing-control applications are more naturally modeled by the \emph{average-reward} criterion, where performance is determined by the long-run steady-state behavior rather than geometrically discounted rewards. This objective appears throughout ergodic control and operations research, including resource allocation, scheduling, queueing, and communication networks \citep{puterman_markov_1994, bertsekas_dynamic_2017, gosavi_reinforcement_2004, mahadevan_average_1996, srikant_communication_2014, neely_stochastic_2010, kumar_global_2024}.

\looseness=-1
To make the distinction between discounted and average-reward MDPs concrete, let $c_k$ denote the cost incurred by a policy at time $k$. The discounted and average objectives can be written informally as
\[
\rho_\gamma = \lim_{H\to\infty}\sum_{k=0}^{H-1}\gamma^k\,c_k,
\qquad
\rho = \lim_{H\to\infty}\frac{1}{H}\sum_{k=0}^{H-1} c_k,
\]
with $\gamma\in(0,1)$.
The key difference is the weighting: discounting weights early costs more heavily, so \emph{initial conditions} matter, while average reward weights all times equally and is governed by the \emph{stationary} behavior. Two consequences follow.
First, discounted PMD analyses exploit the initial state distribution to lower bound state visitation uniformly.
In average reward, initial conditions are transient, so controlling the stationary distribution shift becomes a new challenge.
Second, discounting yields a finite effective horizon $H_\gamma:=1/(1-\gamma)$ and admits unbiased value estimates via geometric stopping times. In average reward, all times are weighted equally, so there is no such stopping distribution, and truncating the horizon introduces a bias that must be controlled through mixing to stationarity.

A standard relationship that links the two objectives is
$\rho=\lim_{\gamma\uparrow 1}(1-\gamma)\rho_\gamma$ \citep{ross_introduction_1983, bertsekas_dynamic_2017}. This has led to the common workaround of analyzing a discounted surrogate with $\gamma$ close to $1$ and then sending $\gamma\uparrow 1$ \citep{murthy_performance_2023, kumar_global_2024}. While analytically appealing, this reduction can yield bounds and algorithmic intuition that break down in the average-reward regime. In particular, discounted finite-time guarantees typically scale polynomially in $H_\gamma$ (up to $H_\gamma^5$ for policy-gradient methods \citep{agarwal_theory_2020, xiao_convergence_2022, kumar_global_2024} and $H_\gamma^2$ for approximate policy iteration \citep{murthy_performance_2023}), which make the resulting bounds in the average-reward limit blow up since $\lim_{\gamma \uparrow 1} (1-\gamma)H_\gamma^p=\infty$ for $p>1$.
Average-reward MDPs therefore need a theory of their own, developed directly rather than inherited from discounted analyses.

\paragraph{Contributions.}
In this paper, we develop a finite-time, finite-sample theory of policy mirror descent (PMD) in \emph{average-reward} Markov decision processes (MDPs):
\begin{itemize}[leftmargin=*, itemsep=5pt, topsep=2pt, parsep=0pt]
\item \textbf{Master recursion.} We establish a master recursion (Lemma~\ref{lem:master}) that governs PMD for any critic, from which all downstream results follow as specializations.
\item \textbf{Linear and asymptotic convergence.} We prove linear convergence for exact, inexact-tabular, and linear function approximation (LFA) updates (Corollary~\ref{cor:exact-pmd}; Theorems~\ref{thm:avg-pmd-linear-inexact-limsup},~\ref{thm:non-tabular-linear-convergence}), at a rate that is tight in the worst case for exact PMD (Proposition~\ref{prop:ap-rate-lower-bound}). For all three updates, we further sharpen prior worst-case rates with an asymptotic-rate analysis. Our LFA result is, to our knowledge, the first non-tabular guarantee of linear convergence to the optimum up to critic error.
\item \textbf{End-to-end sample complexity.} For the tabular and LFA settings, we give critics whose end-to-end sample complexities (Thms.~\ref{thm:pmd_sample_complexity},~\ref{thm:non_tabular_pmd_sample_complexity_generative}) improve the prior best policy-gradient dependence from $|\mathcal S||\mathcal A|\,t_{\mathrm{mix}}^3$ to $|\mathcal S|\,t_{\mathrm{mix}}^3+|\mathcal S||\mathcal A|\,t_{\mathrm{mix}}^2$ in the tabular setting and from $|\mathcal S||\mathcal A|\,t_{\mathrm{mix}}^5$ to $d\,t_{\mathrm{mix}}^3$ in the LFA setting. Both guarantees hold with high probability for the last iterate rather than in expectation for a random iterate. The LFA guarantee needs only a single trajectory, a strictly weaker oracle than the generative model the tabular guarantee uses.
\item \textbf{Matching lower bounds.} We establish information-theoretic lower bounds via Assouad reductions (Propositions~\ref{prop:tabular_lower_bound},~\ref{prop:nontabular_lower_bound}), showing that both critics are optimal up to logarithmic factors.
\end{itemize}

\section{Related work} \label{sec:related-work}
Recent work on policy optimization for average-reward MDPs spans several partial results.
\citet{sun_policy_2024} prove linear convergence of robust PMD for average-reward MDPs under squared Euclidean geometry. Specialized to zero robustness radius, their result partially overlaps with ours for an exact critic (Theorem~\ref{thm:avg-pmd-linear-inexact-limsup} with $\varepsilon_Q=0$) but uses a stricter step-size growth condition and does not address general Bregman geometry, inexact critics, function approximation, or sample complexity.
\citet{li_stochastic_2024} give stochastic first-order methods with $O(1/\sqrt{T})$ unregularized rates and linear convergence up to an $\smash{\widetilde O}(1/T)$ stochastic term only under strongly convex regularization. Their LFA sample complexity on a single trajectory (their Corollary~3) has $|\mathcal S||\mathcal A|\,t_{\mathrm{mix}}^5$ dependence.
\citet{kumar_global_2024, murthy_convergence_2023} establish $O(1/T)$ rates for projected policy gradient and for natural policy gradient (NPG) and mirror-descent-like methods, \citet{grosof_convergence_2025} obtain an $O(1/\sqrt{T})$ rate for NPG on countably-infinite-state queueing systems, and \citet{murthy_performance_2023} develop performance bounds for approximate policy iteration; their bound for a fixed-step mirror-descent update keeps an error floor that does not vanish with the critic error.
\citet{bai_regret_2024, ganesh_order-optimal_2025, ganesh_sharper_2025} study general-parameterized average-reward policy gradient, achieving polynomial convergence and regret rates.
In concurrent work, \citet{lee_policy_2026} extend policy gradient to multichain MDPs via $\alpha$-clipped PMD, with sample complexity $\smash{\widetilde O}(\varepsilon^{-6})$.\footnote{Throughout, $\smash{\widetilde O}(\cdot)$ and $\smash{\widetilde\Theta}(\cdot)$ hide polylogarithmic factors in the problem parameters, the target accuracy, and the confidence.} This extension is complementary to our sharper $\smash{\widetilde O}(\varepsilon^{-2})$ analysis in the ergodic case.

Against this backdrop, we provide a unified PMD analysis under general Bregman geometry spanning exact, inexact-tabular, and LFA updates. We improve on prior tabular results along multiple axes (general Bregman geometry, inexact critics, asymptotic-rate analysis, sharper sample complexity, high-probability last-iterate guarantees) and, to our knowledge, establish the first linear convergence of non-tabular average-reward PMD to the optimum up to critic error, with a sharper sample complexity. Matching information-theoretic lower bounds establish that both critics are optimal up to logarithmic factors. Table~\ref{tab:comparison} summarizes these comparisons, and Appendix~\ref{ap:related-work} details them theorem by theorem.

\begin{table}[t]
\caption{Comparison with prior work. Rates concern the unregularized objective, up to critic-error floors. Sample complexities omit instance factors other than rate constants (App.~\ref{ap:related-work}). Guarantees hold w.h.p.\ or in expectation (in exp.), for the last or a uniformly random (rand.) iterate.}
\label{tab:comparison}
\centering

\begingroup
\footnotesize
\setlength{\tabcolsep}{2pt}
\renewcommand{\arraystretch}{1.15}
\begin{tabularx}{\textwidth}{@{}P{2.7cm}P{3.15cm}P{3.15cm}P{2.8cm}Y@{}}
\toprule
Work & Update & Convergence & Sample complexity & Guarantee \\
\midrule
\multicolumn{5}{@{}l}{\emph{Exact critic}} \\
\addlinespace[2pt]
\citet{murthy_convergence_2023} & PMD (KL) & $O(1/T)$ & \dashunder{Sample complexity} & \dashunder{Guarantee} \\
\citet{kumar_global_2024} & Projected PG & $O(1/T)$ & \dashunder{Sample complexity} & \dashunder{Guarantee} \\
\citet{sun_policy_2024} & PMD (Euclidean) & $O((1-1/M_0)^T)$ & \dashunder{Sample complexity} & \dashunder{Guarantee} \\
\textbf{Ours} & PMD (Bregman) & $O((1-1/\bar\chi)^T)$;\newline $\limsup_t \Delta_t^{1/t} \le 1-1/\chi$ & \dashunder{Sample complexity} & \dashunder{Guarantee} \\
\midrule
\multicolumn{5}{@{}l}{\emph{Tabular critic, generative-model oracle}} \\
\addlinespace[2pt]
\citet{lee_policy_2026}$^{\dagger}$ & $\alpha$-clipped PMD (Bregman) & $O((1-1/C_\alpha)^T)$ & $\widetilde O(t_{\mathrm{tar}}^3|\mathcal S||\mathcal A|C_\alpha^6/\varepsilon^6)$ & \guar{w.h.p.}{last} \\
\citet{li_stochastic_2024}$^{\ddagger}$ & PMD (Bregman) & $O(1/\sqrt T)$ & $\widetilde O(t_{\mathrm{mix}}^3|\mathcal S||\mathcal A|/\varepsilon^2)$ & \guar{in exp.}{rand.} \\
\textbf{Ours} & PMD (Bregman) & $O((1-1/\bar\chi)^T)$ & \hspace*{-26pt}$\widetilde O(t_{\mathrm{mix}}^2|\mathcal S|(|\mathcal A|+t_{\mathrm{mix}})\bar\chi^3/\varepsilon^2)$ & \guar{w.h.p.}{last} \\
\emph{Lower bound (critic)} & & & \hspace*{-26pt}$\Omega(t_{\mathrm{mix}}^2|\mathcal S|(|\mathcal A|+t_{\mathrm{mix}})/\varepsilon^2)$ & \\
\midrule
\multicolumn{5}{@{}l}{\emph{Linear function approximation, single-trajectory oracle}} \\
\addlinespace[2pt]
\citet{murthy_performance_2023} & PMD (KL) & linear, non-vanishing floor & \dashunder{Sample complexity} & \guar{in exp.}{last} \\
\citet{li_stochastic_2024} & PMD (KL) & $O(1/\sqrt T)$ & $\widetilde O(t_{\mathrm{mix}}^5|\mathcal S||\mathcal A|/\varepsilon^2)$ & \guar{in exp.}{rand.} \\
\textbf{Ours} & PMD (Bregman) & $O((1-1/\bar\chi)^T)$ & $\widetilde O(t_{\mathrm{mix}}^3 d\,\bar\chi^3/\varepsilon^2)$ & \guar{w.h.p.}{last} \\
\emph{Lower bound (critic)} & & & $\Omega(t_{\mathrm{mix}}^3 d/\varepsilon^2)$ & \\
\bottomrule
\end{tabularx}

\vspace{4pt}
\parbox{\linewidth}{\footnotesize
$M_0\ge\bar\chi$ (Section~\ref{sec:pmd-algorithms}).
$^{\dagger}$Multichain, targeting the best $\alpha$-clipped policy, not the optimum; $t_{\mathrm{tar}}$ and $C_\alpha$ are their analogues of $t_{\mathrm{mix}}$ and $\bar\chi$, respectively.
$^{\ddagger}$Mixing unichain; App.~\ref{ap:mixing-assumption-details} discusses this extension of our results. Other rows are ergodic.}
\endgroup
\end{table}

\section{Problem setup and preliminaries} \label{sec:preliminaries}

We consider a Markov decision process (MDP) $(\mathcal S,\mathcal A, P, c),$
where $\mathcal S$ is a finite state space, $\mathcal A$ is a finite action space,
$P(\cdot\mid s,a)$ is the transition kernel, and $c:\mathcal S\times\mathcal A\to[0,1]$
is the cost function. Let $\mathcal{P}(\mathcal{X})$ denote the probability simplex over the finite set $\mathcal{X}$. A policy $\pi \in \mathcal{P}(\mathcal{A})^\mathcal{S} =: \Pi$ is a collection
of conditional distributions $\pi(\cdot\mid s)\in\mathcal{P}(\mathcal A)$ for all $s\in\mathcal S$.
Each policy $\pi$ induces a Markov kernel $P_\pi(\cdot\mid s) := \sum_{a\in\mathcal A} \pi(a\mid s)\,P(\cdot\mid s,a)$. For all objects that live in finite-dimensional spaces, we freely switch between function and vector/matrix notation (e.g., $\pi(\cdot\mid s)=\pi_s\in\mathbb{R}^{|\mathcal A|}$ is a stochastic row vector and $P_\pi\in\mathbb{R}^{|\mathcal S|\times|\mathcal S|}$ a row-stochastic matrix). We assume throughout:

\begin{assumption}[Ergodicity]
\label{as:ergodic}
For every policy $\pi\in\Pi$, the Markov chain with kernel $P_\pi$ is irreducible and aperiodic, hence admits a unique stationary distribution $d^\pi\in\mathcal{P}(\mathcal S)$ satisfying $d^\pi = d^\pi P_\pi$.
\end{assumption}

Assumption~\ref{as:ergodic} is a standard regularity condition in average-reward settings, as it makes the steady-state objective well defined \citep{kumar_global_2024, li_stochastic_2024, murthy_convergence_2023, murthy_performance_2024}; App.~\ref{ap:mixing-assumption-details} discusses its relaxation to unichain MDPs. A key consequence of Asm.~\ref{as:ergodic} is that the stationary distributions are uniformly bounded away from zero:
\begin{lemma}
\label{lem:uniform-d-bound}
The constant $\underline d := \inf_{\pi\in\Pi,\, s\in\mathcal S} d^\pi(s)$ is positive. Under a uniform minorization condition, $\underline d$ admits an explicit lower bound (Appendix~\ref{ap:mixing-assumption-details}).
\end{lemma}

\paragraph{Objective.}
The average cost of $\pi$ is defined as $\rho(\pi) := \mathbb E_{s\sim d^\pi,\,a\sim\pi_s}\!\bigl[c(s,a)\bigr]$.
Our goal is to find an optimal policy $\pi^\star \in \argmin_{\pi \in \Pi} \rho(\pi)$ (fixed but arbitrary), with $\rho^\star := \rho(\pi^\star)$ and $d^\star := d^{\pi^\star}$.

Policy optimization requires knowing which actions are better than others in each state. In the average-reward setting, starting from any state--action pair $(s,a)$ and following $\pi$ afterwards leaves the long-run average at $\rho(\pi)$, so actions must be compared by the cost they incur along the way. This is measured by the \emph{differential action-value function} \citep{kumar_global_2024}, the expected total cost in excess of $\rho(\pi)$:
\begin{equation} \label{eq:diff-Q-rolling}
Q^\pi(s,a) \;:=\; \lim_{H\to\infty} \mathbb{E}\!\left[\sum_{k=0}^{H-1}\bigl(c(s_k,a_k)-\rho(\pi)\bigr) \;\Bigg|\; s_0 = s,\, a_0 = a\right].
\end{equation}
Under Assumption~\ref{as:ergodic}, $P_\pi^k$ converges geometrically, so the limit exists (see Appendix~\ref{ap:preliminaries}). Centering by $\rho(\pi)$ is necessary since the uncentered sum
$\mathbb{E}\sum_{k<H} c(s_k,a_k)$ grows linearly in $H$ at rate
$\rho(\pi)$. In the discounted setting, geometric weighting plays the analogous
role.

The following facts are the average-reward counterparts of two standard discounted tools, often used in the convergence analysis of policy gradient methods \citep{sutton_policy_1999, agarwal_theory_2020, xiao_convergence_2022, yuan_linear_2023}. See App.~\ref{ap:preliminaries} for proofs.
\begin{fact}[Performance difference lemma]
\label{fact:avg-pdl}
For any policies $\pi,\pi'$,
\[
    \rho(\pi') - \rho(\pi)
    \;=\;
    \E_{s\sim d^{\pi'}}
    \Bigl[
        \bigl\langle Q^\pi_s,\,\pi'_s - \pi_s\bigr\rangle
    \Bigr].
\]
\end{fact}
\begin{fact}[Policy gradient theorem]
\label{fact:avg-policy-gradient}
As a gradient on the simplex, $\nabla_{\pi_s}\rho(\pi)$ can be represented as
\[
\nabla_{\pi_s}\rho(\pi)\;=\;d^\pi(s)Q^\pi_s.
\]
\end{fact}

\subsection{Policy mirror descent algorithms} \label{sec:pmd-algorithms}
Mirror descent \citep{nemirovskij_problem_1983} minimizes a function $f$ through the steps $x^{t+1} \in \argmin_{x}\bigl\{\eta_t\langle \nabla f(x^t), x\rangle + D(x, x^t)\bigr\}$, where $D(x, x') := F(x) - F(x') - \langle \nabla F(x'), x - x'\rangle$ is the Bregman divergence of a Legendre function $F$ \citep{lattimore_bandit_2020}. Following \citep{shani_adaptive_2020, xiao_convergence_2022}, we apply it to $\rho$ over $\Pi$ with the state-weighted divergence $D^{d^t}(\pi, \pi^t)$, where $D^{d}(\pi, \pi') := \sum_s d(s)\, D(\pi_s, \pi'_s)$. With $Q^t := Q^{\pi^t}$ and $d^t := d^{\pi^t}$, Fact~\ref{fact:avg-policy-gradient} gives $\nabla_{\pi_s}\rho(\pi^t) = d^t(s)\,Q^t_s$, so the update splits across states and the positive weights $d^t(s)$ drop out. Replacing $Q^t$ by a critic $\tilde Q^t$ yields the \emph{policy mirror descent} (PMD) update
\begin{equation}
\label{eq:pmd-update}
\pi^{t+1}_s \;\in\; \argmin_{p\in\mathcal{P}(\mathcal A)} \Bigl\{\eta_t\,\langle \tilde Q^t_s,\,p\rangle + D\bigl(p,\pi^t_s\bigr)\Bigr\}.
\end{equation}
\begin{remark}
When $D$ is the KL divergence, which induces the local Fisher geometry on the simplex, \eqref{eq:pmd-update} reduces to a natural policy gradient step \citep{agarwal_theory_2020, xiao_convergence_2022}. When $D$ is the squared Euclidean distance, it reduces to projected $Q$-descent \citep{xiao_convergence_2022, sun_policy_2024}. We keep the general Bregman geometry in this paper.
\end{remark}

Three instantiations of the critic $\smash{\tilde Q}^t$ cover our analyses: $\smash{\tilde Q}^t = Q^t$, the exact critic (Section~\ref{sec:tabular}); $\smash{\tilde Q}^t = \smash{\widehat Q}^t$, an estimate of each $Q^t(s,a)$ with a uniform error bound (Section~\ref{sec:tabular}); and $\smash{\tilde Q}^t(s,a) = \phi_{s,a}^\top w_t$, a linear function of features $\phi_{s,a}$ with a mean-squared-error (MSE) bound (Section~\ref{sec:non-tabular}).

The most important property of mirror descent is the three-point inequality, which relates one-step progress to a difference of Bregman divergences and enables telescoping across iterations.

\begin{fact}[Three-point inequality for PMD]
\label{fact:three-point}
For any $p\in\mathcal{P}(\mathcal A)$, the PMD update~\eqref{eq:pmd-update} satisfies
\[
\eta_t\langle \tilde Q^t_s,\,\pi_s^{t+1}-p\rangle
\le
D(p,\pi_s^t) - D(p,\pi_s^{t+1}) - D(\pi_s^{t+1},\pi_s^t).
\]
\end{fact}

Setting $p = \pi^t_s$ in Fact~\ref{fact:three-point} yields, for any critic $\smash{\tilde Q}^t$, the descent-direction property
\begin{equation} \label{eq:pmd-sign}
\langle \tilde Q^t_s,\,\pi^{t+1}_s - \pi^t_s\rangle \;\le\; 0, \qquad \forall\, s \in \mathcal S.
\end{equation}
Specializing to the exact case $\smash{\tilde Q}^t = Q^t$, taking expectation over $s\sim
d^{t+1}$, and applying Fact~\ref{fact:avg-pdl} gives the descent property
$\rho(\pi^{t+1}) \le \rho(\pi^t)$, so each exact PMD update decreases the
objective \emph{without} requiring any restriction on the step size $\eta_t$.
Moreover, \eqref{eq:pmd-update} shows that larger $\eta_t$ pushes
$\pi_s^{t+1}$ closer to a greedy update, recovering policy iteration in the
limit. In the discounted setting, aggressively growing step sizes thus inherit the linear convergence of policy iteration
\citep{xiao_convergence_2022, yuan_linear_2023}. We show the same in the average-reward setting, at a rate governed by the distribution mismatch introduced next.

\paragraph{Stationary distribution mismatch.}
Let
\[
\chi_t
\;:=\;
\left\|
\frac{d^{\star}}{d^{t}}
\right\|_\infty
\;=\;
\max_{s\in\mathcal S}\frac{d^{\star}(s)}{d^{t}(s)}
\]
measure the mismatch between the stationary distributions of the PMD iterate $\pi^t$
and an optimal policy $\pi^\star$.
Its discounted analogue, the distribution-mismatch coefficient \citep{agarwal_theory_2020, xiao_convergence_2022, yuan_linear_2023}, is well-studied, but its analysis exploits a structural property that the average-reward setting lacks.
In the discounted setting, the analogue of $d^\pi$ is the discounted occupancy $d^\pi_\gamma := (1-\gamma)\sum_{k\ge 0}\gamma^k \mu_k^\pi$, with $\mu_k^\pi$ the state distribution at time $k$ under $\pi$ from an initial $\mu_0$. Dropping all $k\ge 1$ terms gives $d^\pi_\gamma(s) \geq (1-\gamma)\mu_0(s)$, so the discounted distribution-mismatch coefficients are uniformly bounded via a freely chosen full-support $\mu_0$. In the average-reward setting, this trick is not available: the objective depends on the stationary distribution $d^\pi$, which forgets $\mu_0$, so no choice of $\mu_0$ can bound it from below. Controlling $\chi_t$ along the iterates therefore becomes a central technical challenge.
Nevertheless, by Lem.~\ref{lem:uniform-d-bound} we have the loose bound $\sup_t \chi_t \leq \underline{d}^{-1}$,
which characterizes the convergence rate in the analysis of \citet{li_stochastic_2024}.
\citet{sun_policy_2024} use the tighter intermediate constant $M_0$ satisfying $\sup_t\chi_t \leq M_0 \leq
\underline{d}^{-1}$.
Our asymptotic analysis below is sharper than both, and we show empirically that its rate takes over after only a brief warm-up (Fig.~\ref{fig:rate-gap-main}).

\subsection{Master recursion for inexact PMD} \label{sec:master}
Write $\Delta_t := \rho(\pi^t) - \rho^\star$ for the suboptimality gap at iteration $t$ and $D_t^\star := D^{d^\star}(\pi^\star, \pi^t)$ for the optimal-policy weighted divergence. The critic error in~\eqref{eq:pmd-update} enters the analysis through the quantities
\begin{equation} \label{eq:AB-def}
A_t := \Bigl|\mathbb{E}_{s\sim d^{t+1}}\bigl[\langle \tilde Q^t_s - Q^t_s, \pi^{t+1}_s - \pi^t_s\rangle\bigr]\Bigr|, \quad B_t := \Bigl|\mathbb{E}_{s\sim d^\star}\bigl[\langle \tilde Q^t_s - Q^t_s, \pi^t_s - \pi^\star_s\rangle\bigr]\Bigr|.
\end{equation}
The following lemma combines the descent-direction property~\eqref{eq:pmd-sign} with the bounds on $A_t, B_t$ to deliver both a non-asymptotic linear rate and a sharper asymptotic rate, governing all PMD specializations.

\begin{lemma}[Master convergence for inexact PMD] \label{lem:master}
Let $\{\pi^t\}$ be generated by~\eqref{eq:pmd-update} with $\pi^0_s \in \operatorname{rint}(\mathcal{P}(\mathcal A))$ for all $s \in \mathcal S$, and suppose $A_t \leq \bar A$, $B_t \leq \bar B$ for all $t \geq 0$ with constants $\bar A, \bar B \geq 0$.

\emph{(i) Non-asymptotic linear rate.} Suppose $\sup_t \chi_t \leq \bar\chi$ for some $\bar\chi \in (1, \infty)$ and the step sizes satisfy $\eta_t \geq \tfrac{\bar\chi}{\bar\chi - 1}\, \eta_{t-1}$ for all $t \geq 0$ with some $\eta_{-1} > 0$. Then for all $t \geq 0$, with $(x)_+ := \max\{x, 0\}$,
\begin{equation} \label{eq:master-nonasymp}
\bigl(\Delta_t - (\bar\chi\, \bar A + \bar B)\bigr)_+ \;\leq\; \Bigl(1 - \tfrac{1}{\bar\chi}\Bigr)^t \Bigl(\Delta_0 + \tfrac{D_0^\star}{\eta_{-1}\bar\chi}\Bigr).
\end{equation}

\emph{(ii) Asymptotic rate.} Suppose $D_0^\star < \infty$ and the step sizes satisfy $\lim_{t\to\infty} \eta_t/\eta_{t-1} = \infty$ with some $\eta_{-1} > 0$, and let $\chi := \limsup_t \chi_t$. Then for every $\zeta > 0$,
\begin{equation} \label{eq:master-asymp}
\limsup_{t\to\infty}\, \bigl(\Delta_t - ((\chi+\zeta)\, \bar A + \bar B)\bigr)_+^{1/t} \;\leq\; 1 - \tfrac{1}{\chi+\zeta};
\end{equation}
hence, $\limsup_{t\to\infty}\Delta_t \leq \chi\, \bar A + \bar B$. If in addition $\bar A = 0$, then $\limsup_{t\to\infty}(\Delta_t - \bar B)_+^{1/t} \leq 1 - 1/\chi$.
\end{lemma}

\begin{remark} \label{rmk:initial-divergence}
Lemma~\ref{lem:master}(i) is informative whenever $D_0^\star<\infty$, which part~(ii) assumes. Since the lemma assumes an interior $\pi^0$, $D_0^\star$ is finite whenever $F$ is finite on the closed simplex, even for a deterministic $\pi^\star$. This covers the two standard choices: the KL divergence and the squared Euclidean distance. The exception is a geometry whose $F$ diverges at the boundary, such as a log barrier. There, a deterministic $\pi^\star$ gives $D_0^\star=\infty$, and the bound of part~(i), though still valid, is vacuous.
\end{remark}

Both conclusions arise from the same one-step Lyapunov recursion $\Psi_{t+1} \leq (1 - 1/\bar\chi)\Psi_t + (\bar\chi A_t + B_t)/\bar\chi$ for $\Psi_t := \Delta_t + D_t^\star/(\eta_{t-1}\bar\chi)$: part~(i) iterates it from $t = 0$, and part~(ii) on a tail with $\bar\chi$ replaced by $\chi + \zeta$ (Appendix~\ref{ap:master-lemma}).
Crucially, the lemma separates the \emph{actor's machinery} (step-size schedule, mismatch coefficient) from the \emph{critic's error} (the bounds $\bar A, \bar B$), so each of our convergence theorems reduces to instantiating $\smash{\tilde Q}^t$ and bounding $A_t, B_t$.

\textbf{On the step size schedule.}
A natural concern is that the aggressively growing step sizes of Lem.~\ref{lem:master} would amplify critic errors. But step size and critic error decouple in PMD: unlike stochastic gradient descent (SGD), where each step injects noise scaled by $\eta_t$, here the critic's error is fixed, and $\eta_t$ only controls how aggressively each update commits to that (possibly biased) direction. Larger $\eta_t$ speeds up convergence but leaves the error floor unchanged. Lemma~\ref{lem:master} formalizes this, and Fig.~\ref{fig:app:pmd-inexact}(a) (App.~\ref{ap:exp-inexact}) verifies it empirically. Numerical overflow is also mild: with $T = \Theta(\bar\chi\log(1/\varepsilon))$ iterations, the growth condition requires only $\eta_T = \mathrm{poly}(1/\varepsilon)$. With a constant step size, the one-step inequality behind Lem.~\ref{lem:master} gives an $O(1/T)$ rate in expectation to the same error floor for a uniformly random iterate; with an exact critic, the same holds for the last iterate (App.~\ref{ap:constant-step}). This matches prior $O(1/T)$ rates \citep{kumar_global_2024, murthy_convergence_2023}. Growing step sizes instead make the rate linear.

Importantly, Lem.~\ref{lem:master}'s guarantees do not require knowing $\chi$ or $\bar\chi$: doubling the step size every iteration gives part~(i) with $\bar\chi$ replaced by $\max\{2,\bar\chi\}$, and part~(ii) needs only $\eta_t/\eta_{t-1}\to\infty$. Constants enter only to state finite-time or finite-sample guarantees, and only as valid upper bounds.

\section{Analysis in the tabular setting} \label{sec:tabular}
\vspace{-0.03em}
We first instantiate Lem.~\ref{lem:master} in the tabular setting, where the critic estimates each $Q^t(s,a)$ with a uniform error bound. Section~\ref{sec:tab-linconv} derives linear convergence of inexact PMD, including the exact case. Section~\ref{sec:sample-complexity-tabular} gives an end-to-end sample complexity and a matching lower bound for the critic.
\subsection{Linear convergence with inexact updates} \label{sec:tab-linconv}
Since $Q^t$ is not available in general, we study the PMD update~\eqref{eq:pmd-update} with $\smash{\tilde Q}^t = \smash{\widehat Q}^t \approx Q^t$, under the following critic-error assumption, which we satisfy with finite samples in Section~\ref{sec:sample-complexity-tabular}:

\begin{assumption}[Uniform critic error]
\label{as:uniform-critic-error}
There exists $\varepsilon_Q \ge 0$ such that $\sup_{t \geq 0}\|\smash{\widehat Q}^t - Q^t\|_\infty \le \varepsilon_Q$.
\end{assumption}
Under Asm.~\ref{as:uniform-critic-error}, $\|\pi_s - \pi'_s\|_1 \le 2$ gives $A_t, B_t \le 2\varepsilon_Q$ in~\eqref{eq:AB-def}, so Lem.~\ref{lem:master} with $\bar A = \bar B = 2\varepsilon_Q$ yields:

\begin{theorem}[Linear convergence of tabular inexact PMD]
\label{thm:avg-pmd-linear-inexact-limsup}
Suppose Assumption~\ref{as:uniform-critic-error} holds, and let $\{\pi^t\}$ be generated by~\eqref{eq:pmd-update} with $\smash{\tilde Q}^t = \smash{\widehat Q}^t$. Under the hypotheses of Lemma~\ref{lem:master}(i),
\begin{equation} \label{eq:inexact-nonasymptotic-linear}
(\Delta_t - 2(\bar\chi + 1)\varepsilon_Q)_+ \;\le\; \Bigl(1 - \tfrac{1}{\bar\chi}\Bigr)^t \Bigl(\Delta_0 + \tfrac{D_0^\star}{\eta_{-1}\bar\chi}\Bigr) \qquad \forall\, t \ge 0;
\end{equation}
and under the hypotheses of Lemma~\ref{lem:master}(ii), for every $\zeta > 0$,
\begin{equation} \label{eq:inexact-asymptotic-linear}
\limsup_{t\to\infty}\bigl(\Delta_t - 2(\chi+\zeta+1)\varepsilon_Q\bigr)_+^{1/t} \;\le\; 1 - \tfrac{1}{\chi+\zeta}.
\end{equation}
\end{theorem}

\begin{corollary}[Exact PMD and superlinear convergence] \label{cor:exact-pmd}
For exact PMD ($\smash{\tilde Q}^t = Q^t$), setting $\varepsilon_Q = 0$ in~\eqref{eq:inexact-asymptotic-linear} removes the $\zeta$-dependence from the floor, so taking $\zeta \downarrow 0$ gives the sharper asymptotic rate
\[
\limsup_{t\to\infty} \Delta_t^{1/t} \;\leq\; 1 - \tfrac{1}{\chi}.
\]
Moreover, $\Delta_t \to 0$, so by compactness of $\Pi$ and continuity of $\pi \mapsto \rho(\pi)$, every limit point of $\{\pi^t\}$ is optimal. Consequently, if $\{\pi^t\}$ converges to an optimal policy (e.g., when the optimal policy is unique), then taking that limit as $\pi^\star$ gives $\chi_t \to \chi = 1$, so exact PMD is \emph{superlinear}: $\limsup_{t\to\infty}\Delta_t^{1/t} = 0$.
\end{corollary}

\begin{remark} \label{rmk:exact-geometries}
For the KL and squared Euclidean geometries, Corollary~\ref{cor:exact-pmd}'s convergence assumption can be dropped (App.~\ref{ap:exact-geometries}, adapting \citet{lin_convergence_2024}): with KL and any nondecreasing $\eta_t \to \infty$, $\Delta_{t+1}/\Delta_t \to 0$ whenever $\Delta_0 > 0$, and with the squared Euclidean distance and any nondecreasing $\eta_t > 0$, exact PMD reaches an optimal policy in finitely many iterations and stays there.
\end{remark}

\begin{figure}[h]
\centering
\vspace{-0.2cm}
\begin{minipage}[t]{0.42\linewidth}
\vspace{0pt}
We empirically validate that the asymptotic rate of Corollary~\ref{cor:exact-pmd} is typically much faster than the worst-case finite-time rate driven by $\bar\chi$ in~\eqref{eq:inexact-nonasymptotic-linear}. Across four diverse average-reward MDPs (random Dirichlet MDPs, access-control queuing, $5{\times}5$ gridworld, Jack's car rental, $80$ runs each), the median $\chi_t/\bar\chi$ (Fig.~\ref{fig:rate-gap-main}) collapses by $10$--$1000\times$ within roughly $20$--$40$ iterations and stays there for the rest of the trajectory.
\end{minipage}\hfill
\begin{minipage}[t]{0.54\linewidth}
\vspace{0pt}
\centering
\includegraphics[width=\linewidth]{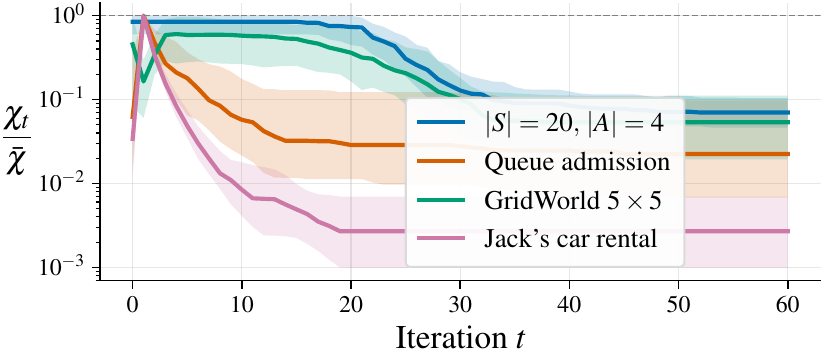}
\vspace{-0.5cm}
\caption{Median $\chi_t/\bar\chi$ trajectories. IQR shaded.}
\label{fig:rate-gap-main}
\end{minipage}
\vspace{-0.2cm}
\end{figure}

In other words, the trajectory enters Cor.~\ref{cor:exact-pmd}'s asymptotic regime (where the rate is governed by $\chi \ll \bar\chi$ rather than by $\bar\chi$ itself) after only a brief warm-up. The two-part decomposition of Lem.~\ref{lem:master} therefore delivers a substantively sharper convergence guarantee than the worst-case bounds of \citet{li_stochastic_2024} and \citet{sun_policy_2024}: empirically, $\bar\chi$ controls only this warm-up, and the asymptotic rate governs convergence from then on. Details and further experiments are in App.~\ref{ap:numerical-experiments}.

In the worst case, however, the rate of Lemma~\ref{lem:master}(i) is tight up to a constant factor over any fixed number of iterations: for every $\bar\chi > 1$ and $T \ge 1$, some MDP and initial policy meeting the hypotheses of Lemma~\ref{lem:master}(i) give $\Delta_t \ge \frac12(1-1/\bar\chi)^t\Delta_0 > 0$ for all $t \le T$, for exact PMD with any Bregman divergence and any positive step sizes (App.~\ref{ap:rate-lower-bound}, adapting \citet[Theorem~4.2]{johnson_optimal_2023}).
\subsection{Sample complexity} \label{sec:sample-complexity-tabular}
We now satisfy Assumption~\ref{as:uniform-critic-error} with transitions sampled from a generative model.
Our critic rests on the identity $Q^\pi(s,a) = c(s,a) - \rho(\pi) + P(\cdot\mid s,a)\,h^\pi$ (App.~\ref{ap:poisson-existence}), where $h^\pi(s) := \langle Q^\pi_s, \pi_s\rangle$ is the differential value function. The critic first estimates $h^\pi$ at every state, then shares these estimates across all state--action pairs through one-step samples. For each state $s$, we draw $B$ rollouts $\{(s_k^{(i)}, a_k^{(i)})\}_{k=0}^{H-1}$, $i = 1, \ldots, B$, of length $H$ that start at $s_0^{(i)} = s$ and evolve under $\pi$.

Estimating $h^\pi$ from such rollouts presents new challenges that have no discounted analog:
in a discounted MDP, truncating the geometric sum $\sum_{k=0}^\infty\gamma^k c_k$ at horizon $H$ yields a tail bias bounded by $\gamma^H/(1-\gamma)$, which decays geometrically with $H$, so larger $H$ always helps \citep[Lemma~15]{xiao_convergence_2022}. In contrast, $h^\pi$ is an \emph{equal-weighted} infinite-horizon sum of expected costs \emph{centered by the unknown average cost} $\rho(\pi)$, and these two features together make finite-horizon estimation nontrivial. Given an estimate $\hat\rho_s(\pi)$, the natural truncated centered-return estimator
\begin{equation}\label{eq:q-estimation}
\hat h^\pi(s) \;:=\; \frac{1}{B}\sum_{i=1}^B\sum_{k=0}^{H-1}\bigl(c(s_k^{(i)},a_k^{(i)}) - \hat\rho_s(\pi)\bigr)
\end{equation}
decomposes as
\[
\hat h^\pi(s) \;=\; \underbrace{\frac{1}{B}\sum_{i=1}^B\sum_{k=0}^{H-1}\bigl(c(s_k^{(i)},a_k^{(i)}) - \rho(\pi)\bigr)}_{\text{tail bias decays with $H$ via mixing}} \;-\; \underbrace{H\bigl(\hat\rho_s(\pi) - \rho(\pi)\bigr)}_{\text{grows linearly in $H$}},
\]
revealing two bias sources that pull $H$ in \emph{opposite directions}: the tail bias of the first shrinks as $H\to\infty$ via mixing to $d^\pi$, while the second amplifies any error in $\hat\rho_s(\pi)$ by a factor of $H$. Taking $H$ large without correspondingly sharpening $\hat\rho_s(\pi)$ therefore \emph{worsens} $\smash{\hat h}^\pi$, whereas in the discounted setting larger $H$ always helps. A second, subtler pitfall: if $\hat\rho_s(\pi)$ is computed as the empirical mean of the same $BH$ costs that appear in~\eqref{eq:q-estimation}, the two sums cancel exactly and give $\smash{\hat h}^\pi \equiv 0$.
We resolve these issues by estimating $\rho(\pi)$ from the \emph{tail} of the rollouts after a burn-in cutoff $\tau \in \{1,\dots,H-1\}$:
\begin{equation} \label{eq:rho-estimation}
\hat\rho_s(\pi) \;:=\; \frac{1}{B(H-\tau)}\sum_{i=1}^B\sum_{k=\tau}^{H-1} c(s_k^{(i)},a_k^{(i)}),
\end{equation}
which avoids the exact cancellation and concentrates around $\rho(\pi)$ once $\tau$ exceeds the mixing time. To quantify how large $H$ and $\tau$ must be, we use the following uniform mixing property:

\begin{lemma}[Uniform mixing]
\label{lem:uniform-mix}
There exists an integer $t_{\mathrm{mix}} \geq 1$ such that $\sup_{\mu_0 \in \mathcal{P}(\mathcal S)} \|\mu_0\, P_\pi^{t_{\mathrm{mix}}} - d^\pi\|_{\mathrm{TV}} \leq 1/4$ for every $\pi \in \Pi$, and consequently
\begin{equation} \label{eq:tmix-decay}
\sup_{\mu_0 \in \mathcal{P}(\mathcal S)} \bigl\|\mu_0\, P_\pi^k - d^\pi\bigr\|_{\mathrm{TV}} \;\leq\; 2^{-\lfloor k/t_{\mathrm{mix}}\rfloor} \qquad \forall\, k \geq 0,\; \forall\, \pi \in \Pi.
\end{equation}
\end{lemma}

\begin{remark} \label{rmk:uniform-geometric-mixing}
Uniform mixing is usually posed as an assumption in average-reward reinforcement learning \citep[e.g.,][]{li_stochastic_2024, wang_optimal_2024}. In our finite setting, it follows from Assumption~\ref{as:ergodic} by compactness of $\Pi$. Appendix~\ref{ap:mixing-assumption-details} gives the proof and a sufficient condition under which $t_{\mathrm{mix}}$ is explicit.
\end{remark}

For each pair $(s,a)$, we then draw $n$ fresh one-step samples $s'_j \sim P(\cdot\mid s,a)$ and set
\begin{equation} \label{eq:q-hat}
\widehat Q^\pi(s,a) \;:=\; c(s,a) - \hat\rho_s(\pi) + \frac{1}{n}\sum_{j=1}^n \hat h^\pi(s'_j),
\end{equation}
where $\hat\rho_s(\pi)$ is the estimate~\eqref{eq:rho-estimation} from the rollouts at $s$ and $\smash{\hat h}^\pi$ is clipped to $[-2t_{\mathrm{mix}}, 2t_{\mathrm{mix}}]$, which contains $h^\pi$ by~\eqref{eq:tmix-decay}. Combined with a Bernstein-type bound for the rollouts, the mixing decay~\eqref{eq:tmix-decay} yields the following finite-sample guarantee:
\begin{lemma}[Tabular critic upper bound]
\label{lem:Q_estimation_burnin}
Fix a policy $\pi$, and let $\smash{\widehat Q}^\pi$ be computed via~\eqref{eq:q-hat}. Then for any target critic accuracy $\varepsilon_Q \in (0, 1)$ and confidence $\delta \in (0, 1)$, the number of generative-model transitions
\[
N \;=\; \widetilde O\bigl((|\mathcal S|\,t_{\mathrm{mix}}^3 + |\mathcal S||\mathcal A|\,t_{\mathrm{mix}}^2)/\varepsilon_Q^2\bigr)
\]
suffices to guarantee $\|\smash{\widehat Q}^\pi - Q^\pi\|_\infty \leq \varepsilon_Q$ with probability at least $1 - \delta$.
\end{lemma}

Rollouts started at every state--action pair would cost $t_{\mathrm{mix}}^3/\varepsilon_Q^2$ transitions per pair. Our critic instead runs them per state and shares the resulting estimates across all pairs through the one-step samples in~\eqref{eq:q-hat}, which cost only $t_{\mathrm{mix}}^2/\varepsilon_Q^2$ per pair. Both terms are information-theoretically necessary: an Assouad reduction perturbs the transitions of the action that $\pi$ takes at every state, which forces the first term, and those of the other actions at every pair, which forces the second.
\begin{proposition}[Tabular critic lower bound]
\label{prop:tabular_lower_bound}
For every $|\mathcal S| \geq 2$, $p \in (0, 1/4)$, $\varepsilon_Q \in (0, 1/(128 p)]$, and every (possibly adaptive) generative-model algorithm, there are an average-reward MDP with state space $\mathcal S$ and action space $\mathcal A$ satisfying Assumption~\ref{as:ergodic} with $t_{\mathrm{mix}} = \Theta(1/p)$ and a deterministic policy $\pi$ such that outputting $\smash{\widehat Q} \in \mathbb{R}^{\mathcal S\times\mathcal A}$ with $\Pr(\|\smash{\widehat Q} - Q^\pi\|_\infty \leq \varepsilon_Q) \geq 2/3$ requires
\[
N \;=\; \Omega\bigl((|\mathcal S|\,t_{\mathrm{mix}}^3 + |\mathcal S||\mathcal A|\,t_{\mathrm{mix}}^2)/\varepsilon_Q^2\bigr)
\]
generative-model transitions. The explicit construction and constants are given in Appendix~\ref{ap:proofs-tabular}.
\end{proposition}

Lemma~\ref{lem:Q_estimation_burnin} and Proposition~\ref{prop:tabular_lower_bound} match up to logarithmic factors for accuracies covered by both and confidence $\delta \leq 1/3$, so the critic's worst-case sample complexity is $\smash{\widetilde\Theta}\bigl((|\mathcal S|\,t_{\mathrm{mix}}^3 + |\mathcal S||\mathcal A|\,t_{\mathrm{mix}}^2)/\varepsilon_Q^2\bigr)$.

Composing Lem.~\ref{lem:Q_estimation_burnin}'s critic with Thm.~\ref{thm:avg-pmd-linear-inexact-limsup}'s actor yields the following end-to-end sample complexity:
\begin{theorem}[Tabular sample complexity]
\label{thm:pmd_sample_complexity}
Suppose $D_0^\star < \infty$. For any target suboptimality $\varepsilon \in (0, 1)$ and confidence $\delta \in (0, 1)$, using the critic of Lemma~\ref{lem:Q_estimation_burnin} at accuracy $\varepsilon_Q = \varepsilon/(4(\bar\chi + 1))$ together with the actor of Theorem~\ref{thm:avg-pmd-linear-inexact-limsup} in \eqref{eq:inexact-nonasymptotic-linear} guarantees $\Delta_T \leq \varepsilon$ with probability at least $1 - \delta$ after $T = \smash{\widetilde\Theta}(\bar\chi)$ iterations, using a total of
\[
N \;=\; \widetilde O\bigl((|\mathcal S|\,t_{\mathrm{mix}}^3 + |\mathcal S||\mathcal A|\,t_{\mathrm{mix}}^2)\,\bar\chi^3/\varepsilon^2\bigr)
\]
generative-model transitions.
\end{theorem}

\begin{remark} \label{rmk:tabular-sample-complexity}
With the same generative-model oracle as \citet[Cor.~1]{li_stochastic_2024}, Thm.~\ref{thm:pmd_sample_complexity} improves their $|\mathcal S||\mathcal A|\,t_{\mathrm{mix}}^3$ dependence to $|\mathcal S|\,t_{\mathrm{mix}}^3 + |\mathcal S||\mathcal A|\,t_{\mathrm{mix}}^2$, a factor of order $\min\{|\mathcal A|, t_{\mathrm{mix}}\}$, and strengthens their in-expectation, random-iterate guarantee to a high-probability, last-iterate one. By Prop.~\ref{prop:tabular_lower_bound}, this dependence of the critic's sample complexity is information-theoretically tight for any policy-gradient method whose critic satisfies Assumption~\ref{as:uniform-critic-error} for every policy. In comparison, the $\widetilde O(|\mathcal S||\mathcal A|\,t_{\mathrm{mix}}/\varepsilon^2)$ rate of \citet{wang_optimal_2024} is achieved by model-based planning, which does not need a critic satisfying Asm.~\ref{as:uniform-critic-error} and is not directly applicable to the policy-gradient methods we analyze.
\end{remark}

\section{Analysis under linear function approximation (LFA)} \label{sec:non-tabular}
Section~\ref{sec:tabular}'s tabular sample complexity (Thm.~\ref{thm:pmd_sample_complexity}) scales with $|\mathcal S|$ and $|\mathcal A|$. We now instantiate Lem.~\ref{lem:master} with a linear critic whose error is bounded in MSE rather than uniformly (Section~\ref{sec:non-tabular-linconv}). Its $d$-dimensional weights replace the $|\mathcal S||\mathcal A|$ entries of a tabular critic, and an on-policy MSE bound can be met from a single trajectory instead of a generative model (Section~\ref{sec:sample-complexity-non-tabular}).
\subsection{Linear convergence with inexact updates} \label{sec:non-tabular-linconv}
To parameterize the critic, we assume access to a feature vector $\phi_{s,a} \in \mathbb{R}^d$ for every $(s,a)$ and use a linear critic $\phi_{s,a}^\top w \approx Q^\pi(s,a)$. Since the features are fixed, we identify the critic with its weights $w \in \mathbb{R}^d$.
Write $\nu^\pi(s,a) := d^\pi(s)\pi(a\mid s)$ for the stationary state--action distribution of $\pi$, and $\nu^t := \nu^{\pi^t}$, $\nu^\star := \nu^{\pi^\star}$. The critic's target at time $t$ is the best linear fit to $Q^t$ under $\nu^t$:
\begin{equation}\label{eq:w-star-def}
w_t^\star \in \argmin_{w\in\mathbb R^d} \tfrac12\,\mathbb{E}_{(s,a)\sim \nu^t}\!\left[\bigl(w^\top \phi_{s,a} - Q^{t}(s,a)\bigr)^2\right].
\end{equation}
Section~\ref{sec:sample-complexity-non-tabular} constructs an estimate $w_t \approx w_t^\star$ from samples, and the PMD update is then~\eqref{eq:pmd-update} with the LFA critic $\smash{\tilde Q}^t(s,a) = \phi_{s,a}^\top w_t$.
\begin{remark}
    Although~\eqref{eq:pmd-update} is written state by state, the actor in this section need not be tabular. For example, specializing $D$ to KL divergence yields the multiplicative update $\pi^{t+1}(a\mid s) \propto \pi^t(a\mid s)\exp(-\eta_t \phi_{s,a}^\top w_t)$, so from a uniform $\pi^0$ every iterate is a log-linear policy in the critic's features, recovering the policy class of \citet{yuan_linear_2023}. Our results hold for any Bregman divergence $D$.
\end{remark}

To bound $A_t, B_t$ in Lemma~\ref{lem:master}, we make three assumptions. First, no linear critic can beat the best linear fit, so we bound its approximation error:

\begin{assumption}[Approximation error]\label{as:approx-error}
There exists $\varepsilon_{\mathrm{app}}\ge 0$ such that $\mathbb{E}_{(s,a)\sim\nu^t}[(w_t^{\star\top}\phi_{s,a}-Q^t(s,a))^2] \le \varepsilon_{\mathrm{app}}^2$ for all $t\ge 0$.
\end{assumption}

Second, the critic's statistical error, measured in the feature-covariance norm $\|u\|_{\Sigma_t} := \sqrt{u^\top \Sigma_t u}$ with $\Sigma_t := \mathbb{E}_{(s,a)\sim\nu^t}[\phi_{s,a}\phi_{s,a}^\top]$, must be controllable (we show how in Section~\ref{sec:sample-complexity-non-tabular}):

\begin{assumption}[Statistical error]\label{as:stat-error}
There exists $\varepsilon_{\mathrm{stat}}\ge 0$ such that $\sup_{t\geq0}\|w_t - w_t^\star\|_{\Sigma_t} \le \varepsilon_{\mathrm{stat}}$.
\end{assumption}

Third, alongside the uniform mismatch coefficient $\chi_t$, the MSE analysis requires a second-moment form of distribution mismatch:

\begin{assumption}[Second-moment distribution mismatch] \label{as:L2-mismatch-bound}
There exist finite deterministic constants $C_{+},C_{\star}<\infty$ such that, almost surely for all $t\ge 0$, $\nu^{t+1}$ and $\nu^\star$ are absolutely continuous with respect to $\nu^t$ and
\[
\mathbb{E}_{(s,a)\sim \nu^t}\!\left[\left(\tfrac{\nu^{t+1}(s,a)}{\nu^t(s,a)}\right)^2\right]
\le C_{+},
\qquad
\mathbb{E}_{(s,a)\sim \nu^t}\!\left[\left(\tfrac{\nu^{\star}(s,a)}{\nu^t(s,a)}\right)^2\right]
\le C_{\star}.
\]
\end{assumption}

\begin{remark} \label{rmk:eps-exploration}
A simple sufficient condition to satisfy Assumption~\ref{as:L2-mismatch-bound} is to enforce an $\varepsilon$-exploration on every iterate so that $\pi^t(a\mid s)\ge \varepsilon/|\mathcal{A}|$ for all $(s,a)$. Then $\nu^t(s,a)\ge (\underline{d}\varepsilon)/|\mathcal{A}|$ and the assumption holds with $C_{+},C_{\star}= |\mathcal{A}|/(\underline{d}\varepsilon)$, which can be loose and incurs an $O(\varepsilon)$ optimization bias.
\end{remark}

Under these assumptions, a change-of-measure argument (Appendix~\ref{ap:proofs-non-tabular}) bounds the critic-error quantities~\eqref{eq:AB-def} by $A_t \leq 2\sqrt{C_+}(\varepsilon_{\mathrm{stat}} + \varepsilon_{\mathrm{app}})$ and $B_t \leq 2\sqrt{C_\star}(\varepsilon_{\mathrm{stat}} + \varepsilon_{\mathrm{app}})$, so Lemma~\ref{lem:master} with these bounds as $\bar A$ and $\bar B$ yields:

\begin{theorem}[Linear convergence of inexact PMD under LFA]\label{thm:non-tabular-linear-convergence}
Suppose Assumptions~\ref{as:approx-error}--\ref{as:L2-mismatch-bound} hold, and let $\{\pi^t\}$ be generated by~\eqref{eq:pmd-update} with $\smash{\tilde Q}^t(s,a) = \phi_{s,a}^\top w_t$. Set $\varepsilon_{\mathrm{floor}}(x) := 2(x\sqrt{C_+} + \sqrt{C_\star})(\varepsilon_{\mathrm{stat}} + \varepsilon_{\mathrm{app}})$. Under the hypotheses of Lemma~\ref{lem:master}(i),
\begin{equation} \label{eq:lfa-nonasymptotic-linear}
(\Delta_t - \varepsilon_{\mathrm{floor}}(\bar\chi))_+ \;\le\; \Bigl(1 - \tfrac{1}{\bar\chi}\Bigr)^t\,\Bigl(\Delta_0 + \tfrac{D_0^\star}{\eta_{-1}\bar\chi}\Bigr) \qquad \forall\, t \ge 0;
\end{equation}
and under the hypotheses of Lemma~\ref{lem:master}(ii), for every $\zeta > 0$,
\begin{equation} \label{eq:lfa-asymptotic-linear}
\limsup_{t\to\infty}\bigl(\Delta_t - \varepsilon_{\mathrm{floor}}(\chi+\zeta)\bigr)_+^{1/t} \;\le\; 1 - \tfrac{1}{\chi+\zeta}.
\end{equation}
\end{theorem}

\subsection{Sample complexity} \label{sec:sample-complexity-non-tabular}
We now satisfy Assumption~\ref{as:stat-error} with transitions sampled along a single trajectory, by minimizing the objective in~\eqref{eq:w-star-def} with SGD. Two obstacles arise when adapting the closest prior work.

\paragraph{Obstacle 1: unbiased samples without a discount factor.} In the closest discounted work, \citet{yuan_linear_2023} fit the critic by averaged SGD, which needs independent state--action pairs with unbiased $Q$-labels. They obtain both from geometric stopping, which samples exactly from the discounted occupancy. Average reward weights all times equally, so it has no such stopping distribution.

\textit{Our approach: sampling through mixing.} We approximate independent pairs and unbiased labels on a single trajectory that alternates between burn-ins of $m$ transitions and rollouts of $H$ transitions. For the pairs, we take $(s_k, a_k)$ at the end of the $k$-th burn-in, which makes each pair's distribution given the past close to $\nu^t$ by~\eqref{eq:tmix-decay}.
A coupling argument then lets the analysis treat the pairs as independent, up to the burn-in error (App.~\ref{ap:proofs-non-tabular}).
For the labels, we reuse the truncated centered-return estimator of Section~\ref{sec:sample-complexity-tabular}, applied to the rollout that starts at $(s_k, a_k)$. By the bounds of Lem.~\ref{lem:ap-Q_estimation_burnin}, each label's bias decays geometrically in $H$, like the burn-in error in $m$, so $m, H = \smash{\widetilde O}(t_{\mathrm{mix}})$ transitions per labeled pair make both errors arbitrarily small.

\paragraph{Obstacle 2: accurate critics for a linear-rate actor.} In the closest average-reward work, the sublinear actor of \citet[Cor.~3]{li_stochastic_2024} needs $\smash{\widetilde O}(\kappa^2/\varepsilon^2)$ steps, where $\kappa^2$ bounds the second moment of their critic outputs and scales as $t_{\mathrm{mix}}^2$. The many small actor steps average out the critic noise, so each critic call only needs to control bias. Each call of their exploratory variance-reduced temporal difference (EVRTD) critic does this on a single trajectory at a cost of order $t_{\mathrm{mix}}^3$, adding up to $t_{\mathrm{mix}}^5$. Our actor (Thm.~\ref{thm:non-tabular-linear-convergence}) instead converges linearly, in $\widetilde O(\bar\chi)$ steps with growing step sizes. To retain the linear rate with few, large actor steps, each critic call must also control its variance to meet Asm.~\ref{as:stat-error}.

\textit{Our approach: averaging inside the critic.} We move the averaging from the actor into the critic. Each critic call runs SGD on the labeled pairs with the tail averaging of \citet{jain_parallelizing_2018} instead of the full averaging used by \citet{yuan_linear_2023}. The labels' second moment is $\smash{\widetilde O}(t_{\mathrm{mix}}^2)$ by Lem.~\ref{lem:ap-Q_estimation_burnin} and plays the role of $\kappa^2$, so this tail-averaged SGD needs $\smash{\widetilde O}(d\,t_{\mathrm{mix}}^2/\varepsilon_{\mathrm{stat}}^2)$ such pairs to leading order. Since a labeled pair's cost scales only as $t_{\mathrm{mix}}$ rather than the $t_{\mathrm{mix}}^3$ of an EVRTD call, the end-to-end dependence (Thm.~\ref{thm:non_tabular_pmd_sample_complexity_generative}) drops from $t_{\mathrm{mix}}^5$ to $t_{\mathrm{mix}}^3$.

Algorithm~\ref{alg:end2end_w_oracle} (App.~\ref{ap:proofs-non-tabular}) combines the two approaches into our critic: tail-averaged SGD on the quadratic loss~\eqref{eq:w-star-def} over inner iterations $k=1,\ldots,K$. Iteration $k$ takes the next pair $(s_k, a_k)$ and its truncated centered-return label $\smash{\widehat Q}_k$ along the trajectory, then applies the stochastic gradient of~\eqref{eq:w-star-def} at that pair, $(w_{t,k-1}^\top \phi_{s_k,a_k} - \smash{\widehat Q}_k)\,\phi_{s_k,a_k}$. The algorithm returns the average of the last $\lceil K/2\rceil$ iterates, which has a small expected error. To make the error small with high probability, Algorithm~\ref{alg:boost} runs Algorithm~\ref{alg:end2end_w_oracle} along the trajectory $O(\log(2/\delta))$ times and returns as $w_t$ the output closest to a majority of the outputs in the empirical feature-covariance norm, the selection rule of \citet{hsu_loss_2016}. The same trajectory continues across outer iterations, each critic call following the current policy $\pi^t$. The resulting critic satisfies the following finite-sample guarantee:

\begin{lemma}[LFA critic upper bound]
\label{lem:w-sample-complexity}
Suppose Assumption~\ref{as:approx-error} holds, $\|\phi_{s,a}\|_2\le L$ for all $(s,a)$, and $\Sigma_t \succeq \mu I_d$ for all $t$, for some $\mu>0$. Let $w_t$ be the output of Algorithm~\ref{alg:boost} with step size
$\alpha=\frac{1}{4L^2}$. Then for any target critic accuracy $\varepsilon_{\mathrm{stat}}\in(0,1)$ and confidence $\delta\in(0,1)$, the number of transitions
\[
N
\;=\;
\widetilde{O}\!\left(
\frac{d\,t_{\mathrm{mix}}^3}{\varepsilon_{\mathrm{stat}}^2}+\frac{L^2}{\mu}\,t_{\mathrm{mix}}\Bigl(1+\frac{\varepsilon_{\mathrm{app}}^2}{\varepsilon_{\mathrm{stat}}^2}\Bigr)
\right)
\]
suffices to guarantee $\|w_t-w_t^\star\|_{\Sigma_t} \le \varepsilon_{\mathrm{stat}}$ with probability at least $1-\delta$.
\end{lemma}

\begin{remark} \label{rmk:covariance-condition}
The covariance condition $\Sigma_t \succeq \mu I_d$ of Lem.~\ref{lem:w-sample-complexity} can be guaranteed by the features alone. It holds with a single $\mu > 0$ for all policies if $\operatorname{span}\{\phi_{s,a_s} : s \in \mathcal S\} = \mathbb{R}^d$ for every choice of one action $a_s$ per state. When $d \le |\mathcal S|$, this is true for almost every choice of features. If the features lack this property, the $\varepsilon$-exploration of Rem.~\ref{rmk:eps-exploration} enforces the condition whenever $\operatorname{span}\{\phi_{s,a} : (s,a) \in \mathcal S \times \mathcal A\} = \mathbb{R}^d$. Appendix~\ref{ap:covariance-condition} gives the details.
\end{remark}

The $d\,t_{\mathrm{mix}}^3/\varepsilon_{\mathrm{stat}}^2$ term of Lemma~\ref{lem:w-sample-complexity} is information-theoretically tight. With one action and $2d$ states, the construction of Proposition~\ref{prop:tabular_lower_bound} is linearly realizable in dimension $d$, and the same Assouad reduction in the covariance norm yields the dimension factor:

\begin{proposition}[LFA critic lower bound]
\label{prop:nontabular_lower_bound}
For every $d \geq 1$, $p \in (0, 1/4)$, $\varepsilon_{\mathrm{stat}} \in (0, 1/(256\,p)]$, and every (possibly adaptive) generative-model algorithm, there is a one-action average-reward MDP on $2d$ states satisfying Assumption~\ref{as:ergodic} with $t_{\mathrm{mix}} = \Theta(1/p)$ and exactly realizable linear differential $Q$ on which outputting $\widehat w \in \mathbb{R}^d$ with $\Pr(\|\widehat w - w^\star\|_\Sigma \leq \varepsilon_{\mathrm{stat}}) \geq 2/3$ requires
\[
N \;=\; \Omega\bigl(d\,t_{\mathrm{mix}}^3/\varepsilon_{\mathrm{stat}}^2\bigr)
\]
generative-model transitions. The explicit construction and constants are given in Appendix~\ref{ap:proofs-non-tabular}.
\end{proposition}

The upper bound (Lem.~\ref{lem:w-sample-complexity}) and the lower bound (Prop.~\ref{prop:nontabular_lower_bound}) match up to logarithmic factors when $\varepsilon_{\mathrm{app}} \le \varepsilon_{\mathrm{stat}} \le t_{\mathrm{mix}}\sqrt{d\mu}/L$, for accuracies covered by both and confidence $\delta\le1/3$. In this regime, the critic's worst-case sample complexity is $\smash{\widetilde\Theta}(d\,t_{\mathrm{mix}}^3/\varepsilon_{\mathrm{stat}}^2)$, since the lower bound extends to every $L$ and $\mu$ allowed by Lem.~\ref{lem:w-sample-complexity} and to a single trajectory (App.~\ref{ap:covariance-condition}). The ratio $L^2/\mu$, which bounds the condition number of $\Sigma_t$, enters the upper bound only through a lower-order term.

Composing Lem.~\ref{lem:w-sample-complexity}'s critic with Thm.~\ref{thm:non-tabular-linear-convergence}'s actor yields the following end-to-end sample complexity:

\begin{theorem}[LFA sample complexity]
\label{thm:non_tabular_pmd_sample_complexity_generative}
Suppose Assumption~\ref{as:L2-mismatch-bound} and the hypotheses of Lemma~\ref{lem:w-sample-complexity} hold, and $D_0^\star < \infty$. For any target suboptimality $\varepsilon \in (0,1)$ with $\varepsilon \ge 8(\bar\chi\sqrt{C_+}+\sqrt{C_\star})\varepsilon_{\mathrm{app}}$ and confidence $\delta \in (0,1)$, using the critic of Lemma~\ref{lem:w-sample-complexity} at accuracy $\varepsilon_{\mathrm{stat}} = \varepsilon/(8(\bar\chi\sqrt{C_+}+\sqrt{C_\star}))$ together with the actor of Theorem~\ref{thm:non-tabular-linear-convergence} in~\eqref{eq:lfa-nonasymptotic-linear} guarantees $\Delta_T \le \varepsilon$ with probability at least $1-\delta$ after $T = \smash{\widetilde\Theta}(\bar\chi)$ outer iterations, using a total of
\[
N
\;=\;
\widetilde O\!\left(
\frac{d\,t_{\mathrm{mix}}^3\,\bar\chi\,(\bar\chi^2C_{+}+C_{\star})}{\varepsilon^{2}}+\frac{L^2}{\mu}\,t_{\mathrm{mix}}\,\bar\chi
\right)
\]
transitions of a single trajectory.
\end{theorem}
\begin{remark}
With the same single-trajectory oracle as \citet[Corollary~3]{li_stochastic_2024}, Theorem~\ref{thm:non_tabular_pmd_sample_complexity_generative} improves their $|\mathcal S||\mathcal A|\,t_{\mathrm{mix}}^5$ dependence to $d\,t_{\mathrm{mix}}^3$ through the trade-off of Obstacle~2: a linear-rate actor with few accurate critic calls in place of a sublinear actor with many bias-controlled ones. By Proposition~\ref{prop:nontabular_lower_bound}, the $d\,t_{\mathrm{mix}}^3/\varepsilon^2$ dependence of the critic's sample complexity in Theorem~\ref{thm:non_tabular_pmd_sample_complexity_generative} is information-theoretically tight for any policy-gradient method whose critic satisfies Assumption~\ref{as:stat-error}, even one with a generative model. The remaining factor $\bar\chi(\bar\chi^2C_++C_\star)$ collects standard PMD-path and distribution-mismatch quantities outside the critic's sample complexity.
\end{remark}

\section{Conclusion}
We developed a finite-time, finite-sample theory of policy mirror descent for ergodic average-reward MDPs around a master recursion (Lemma~\ref{lem:master}) whose two-part structure separates a transient $\bar\chi$-regime from an asymptotic $\chi$-regime, the latter empirically reached after a brief warm-up. Our tabular and LFA sample complexities (Theorems~\ref{thm:pmd_sample_complexity}, \ref{thm:non_tabular_pmd_sample_complexity_generative}) improve the prior best policy-gradient dependence, with guarantees that hold with high probability for the last iterate. The tabular critic uses a generative model and pays $t_{\mathrm{mix}}^3$ per state and $t_{\mathrm{mix}}^2$ per state--action pair. The LFA critic uses a single trajectory and pays $d\,t_{\mathrm{mix}}^3$ in its leading term, free of $|\mathcal S|$ and $|\mathcal A|$. Matching lower bounds (Propositions~\ref{prop:tabular_lower_bound},~\ref{prop:nontabular_lower_bound}) establish that both critics are optimal up to logarithmic factors. Because the master recursion holds for any critic, neural critics are a natural next step, as are multichain MDPs and regret guarantees.

\begin{ack}
We thank the anonymous reviewers and the area chair for their thoughtful comments and helpful suggestions, which improved the paper. This work was supported in part by the NSF AI Institute (AI-EDGE) 2112471 and the ARO Grant W911NF-24-1-0103.
\end{ack}

\bibliographystyle{plainnat}
\bibliography{references}

\begin{thebibliography}{49}
\providecommand{\natexlab}[1]{#1}
\providecommand{\url}[1]{\texttt{#1}}
\expandafter\ifx\csname urlstyle\endcsname\relax
  \providecommand{\doi}[1]{doi: #1}\else
  \providecommand{\doi}{doi: \begingroup \urlstyle{rm}\Url}\fi

\bibitem[Ada et~al.(2022)Ada, Ugur, and
  Akin]{adaGeneralizationTransferLearning2022}
Suzan~Ece Ada, Emre Ugur, and H.~Levent Akin.
\newblock Generalization in {Transfer} {Learning}: {Robust} {Control} of
  {Robot} {Locomotion}.
\newblock \emph{Robotica}, 40\penalty0 (11):\penalty0 3811--3836, 2022.
\newblock \doi{10.1017/S0263574722000625}.

\bibitem[Agarwal et~al.(2020)Agarwal, Kakade, Lee, and
  Mahajan]{agarwal_theory_2020}
Alekh Agarwal, Sham~M. Kakade, Jason~D. Lee, and Gaurav Mahajan.
\newblock On the {Theory} of {Policy} {Gradient} {Methods}: {Optimality},
  {Approximation}, and {Distribution} {Shift}, 2020.
\newblock arXiv:1908.00261.

\bibitem[Alfano et~al.(2023)Alfano, Yuan, and Rebeschini]{alfano_novel_2023}
Carlo Alfano, Rui Yuan, and Patrick Rebeschini.
\newblock A {Novel} {Framework} for {Policy} {Mirror} {Descent} with {General}
  {Parameterization} and {Linear} {Convergence}.
\newblock In \emph{Advances in {Neural} {Information} {Processing} {Systems}},
  2023.

\bibitem[Bach and Moulines(2013)]{bach_non-strongly-convex_2013}
Francis Bach and Eric Moulines.
\newblock Non-{Strongly}-{Convex} {Smooth} {Stochastic} {Approximation} with
  {Convergence} {Rate} {O}(1/n).
\newblock In \emph{Advances in {Neural} {Information} {Processing} {Systems}},
  2013.

\bibitem[Bai et~al.(2024)Bai, Mondal, and Aggarwal]{bai_regret_2024}
Qinbo Bai, Washim~Uddin Mondal, and Vaneet Aggarwal.
\newblock Regret {Analysis} of {Policy} {Gradient} {Algorithm} for {Infinite}
  {Horizon} {Average} {Reward} {Markov} {Decision} {Processes}, 2024.
\newblock arXiv:2309.01922.

\bibitem[Bertsekas(2017)]{bertsekas_dynamic_2017}
Dimitri~P. Bertsekas.
\newblock \emph{Dynamic {Programming} and {Optimal} {Control}}, volume~1.
\newblock Athena Scientific, fourth edition, 2017.

\bibitem[Cayci et~al.(2024)Cayci, He, and Srikant]{cayci_linear_2024}
Semih Cayci, Niao He, and R.~Srikant.
\newblock Linear {Convergence} of {Entropy}-{Regularized} {Natural} {Policy}
  {Gradient} with {Linear} {Function} {Approximation}, 2024.
\newblock arXiv:2106.04096.

\bibitem[Cen et~al.(2021)Cen, Cheng, Chen, Wei, and Chi]{cen_fast_2021}
Shicong Cen, Chen Cheng, Yuxin Chen, Yuting Wei, and Yuejie Chi.
\newblock Fast {Global} {Convergence} of {Natural} {Policy} {Gradient}
  {Methods} with {Entropy} {Regularization}, 2021.
\newblock arXiv:2007.06558.

\bibitem[Daley(2021)]{daley2021gym}
Brett Daley.
\newblock Gym {Classics}.
\newblock \url{https://github.com/brett-daley/gym-classics}, 2021.

\bibitem[Ganesh et~al.(2025{\natexlab{a}})Ganesh, Mondal, and
  Aggarwal]{ganesh_order-optimal_2025}
Swetha Ganesh, Washim~Uddin Mondal, and Vaneet Aggarwal.
\newblock Order-{Optimal} {Regret} with {Novel} {Policy} {Gradient}
  {Approaches} in {Infinite}-{Horizon} {Average} {Reward} {MDPs}.
\newblock In \emph{International {Conference} on {Artificial} {Intelligence}
  and {Statistics}}, 2025{\natexlab{a}}.

\bibitem[Ganesh et~al.(2025{\natexlab{b}})Ganesh, Mondal, and
  Aggarwal]{ganesh_sharper_2025}
Swetha Ganesh, Washim~Uddin Mondal, and Vaneet Aggarwal.
\newblock A {Sharper} {Global} {Convergence} {Analysis} for {Average} {Reward}
  {Reinforcement} {Learning} via an {Actor}-{Critic} {Approach},
  2025{\natexlab{b}}.
\newblock arXiv:2407.18878.

\bibitem[Gosavi(2004)]{gosavi_reinforcement_2004}
Abhijit Gosavi.
\newblock A {Reinforcement} {Learning} {Algorithm} {Based} on {Policy}
  {Iteration} for {Average} {Reward}: {Empirical} {Results} with {Yield}
  {Management} and {Convergence} {Analysis}.
\newblock \emph{Machine Learning}, 55\penalty0 (1):\penalty0 5--29, 2004.
\newblock \doi{10.1023/B:MACH.0000019802.64038.6c}.

\bibitem[Grosof et~al.(2025)Grosof, Maguluri, and
  Srikant]{grosof_convergence_2025}
Isaac Grosof, Siva~Theja Maguluri, and R.~Srikant.
\newblock Convergence of {Natural} {Policy} {Gradient} for a {Family} of
  {Infinite}-{State} {Queueing} {MDPs}, 2025.
\newblock arXiv:2402.05274.

\bibitem[Hsu and Sabato(2016)]{hsu_loss_2016}
Daniel Hsu and Sivan Sabato.
\newblock Loss {Minimization} and {Parameter} {Estimation} with {Heavy}
  {Tails}.
\newblock \emph{Journal of Machine Learning Research}, 2016.

\bibitem[Jain et~al.(2018)Jain, Kakade, Kidambi, Netrapalli, and
  Sidford]{jain_parallelizing_2018}
Prateek Jain, Sham~M. Kakade, Rahul Kidambi, Praneeth Netrapalli, and Aaron
  Sidford.
\newblock Parallelizing {Stochastic} {Gradient} {Descent} for {Least} {Squares}
  {Regression}: {Mini}-{Batching}, {Averaging}, and {Model} {Misspecification}.
\newblock \emph{Journal of Machine Learning Research}, 2018.

\bibitem[Johnson et~al.(2023)Johnson, Pike-Burke, and
  Rebeschini]{johnson_optimal_2023}
Emmeran Johnson, Ciara Pike-Burke, and Patrick Rebeschini.
\newblock Optimal {Convergence} {Rate} for {Exact} {Policy} {Mirror} {Descent}
  in {Discounted} {Markov} {Decision} {Processes}, 2023.
\newblock arXiv:2302.11381.

\bibitem[Kakade(2001)]{kakade_natural_2001}
Sham~M. Kakade.
\newblock A {Natural} {Policy} {Gradient}.
\newblock In \emph{Advances in {Neural} {Information} {Processing} {Systems}},
  2001.

\bibitem[Khodadadian et~al.(2021)Khodadadian, Jhunjhunwala, Varma, and
  Maguluri]{khodadadian_linear_2021}
Sajad Khodadadian, Prakirt~Raj Jhunjhunwala, Sushil~Mahavir Varma, and
  Siva~Theja Maguluri.
\newblock On the {Linear} {Convergence} of {Natural} {Policy} {Gradient}
  {Algorithm}, 2021.
\newblock arXiv:2105.01424.

\bibitem[Kumar et~al.(2024)Kumar, Murthy, Shufaro, Levy, Srikant, and
  Mannor]{kumar_global_2024}
Navdeep Kumar, Yashaswini Murthy, Itai Shufaro, Kfir~Y. Levy, R.~Srikant, and
  Shie Mannor.
\newblock On the {Global} {Convergence} of {Policy} {Gradient} in {Average}
  {Reward} {Markov} {Decision} {Processes}, 2024.
\newblock arXiv:2403.06806.

\bibitem[Lan(2022)]{lan_policy_2022}
Guanghui Lan.
\newblock Policy {Mirror} {Descent} for {Reinforcement} {Learning}: {Linear}
  {Convergence}, {New} {Sampling} {Complexity}, and {Generalized} {Problem}
  {Classes}, 2022.
\newblock arXiv:2102.00135.

\bibitem[Lattimore and Szepesvári(2020)]{lattimore_bandit_2020}
Tor Lattimore and Csaba Szepesvári.
\newblock \emph{Bandit {Algorithms}}.
\newblock Cambridge University Press, 2020.
\newblock \doi{10.1017/9781108571401}.

\bibitem[Lee and Ryu(2026)]{lee_policy_2026}
Jongmin Lee and Ernest~K. Ryu.
\newblock Policy {Gradient} {Algorithms} in {Average}-{Reward} {Multichain}
  {MDPs}, 2026.
\newblock arXiv:2602.18003.

\bibitem[Lee et~al.(2026)Lee, Ryu, and Aggarwal]{lee_learning_2026}
Jongmin Lee, Ernest~K. Ryu, and Vaneet Aggarwal.
\newblock Learning {Policy} from a {Single} {Trajectory} in {Average}-{Reward}
  {Markov} {Decision} {Process}, 2026.
\newblock arXiv:2606.16729.

\bibitem[Levin and Peres(2017)]{levin_markov_2017}
David~A. Levin and Yuval Peres.
\newblock \emph{Markov {Chains} and {Mixing} {Times}}.
\newblock American Mathematical Society, second edition, 2017.

\bibitem[Li et~al.(2024)Li, Wu, and Lan]{li_stochastic_2024}
Tianjiao Li, Feiyang Wu, and Guanghui Lan.
\newblock Stochastic {First}-{Order} {Methods} for {Average}-{Reward} {Markov}
  {Decision} {Processes}, 2024.
\newblock arXiv:2205.05800.

\bibitem[Lin and Zhang(2024)]{lin_convergence_2024}
Dachao Lin and Zhihua Zhang.
\newblock On the {Convergence} of {Policy} in {Unregularized} {Policy} {Mirror}
  {Descent}, 2024.
\newblock arXiv:2205.08176.

\bibitem[Liu et~al.(2025)Liu, Li, Lin, Wei, and Zhang]{liu_convergence_2025}
Jiacai Liu, Wenye Li, Dachao Lin, Ke~Wei, and Zhihua Zhang.
\newblock On the {Convergence} of {Projected} {Policy} {Gradient} for {Any}
  {Constant} {Step} {Sizes}.
\newblock \emph{Journal of Machine Learning Research}, 2025.

\bibitem[Mahadevan(1996)]{mahadevan_average_1996}
Sridhar Mahadevan.
\newblock Average {Reward} {Reinforcement} {Learning}: {Foundations},
  {Algorithms}, and {Empirical} {Results}.
\newblock \emph{Machine Learning}, 22\penalty0 (1-3):\penalty0 159--195, 1996.
\newblock \doi{10.1007/BF00114727}.

\bibitem[Murthy and Srikant(2023)]{murthy_convergence_2023}
Yashaswini Murthy and R.~Srikant.
\newblock On the {Convergence} of {Natural} {Policy} {Gradient} and {Mirror}
  {Descent}-{Like} {Policy} {Methods} for {Average}-{Reward} {MDPs}.
\newblock In \emph{2023 62nd {IEEE} {Conference} on {Decision} and {Control}
  ({CDC})}, pages 1979--1984, 2023.
\newblock \doi{10.1109/CDC49753.2023.10383691}.

\bibitem[Murthy et~al.(2023)Murthy, Moharrami, and
  Srikant]{murthy_performance_2023}
Yashaswini Murthy, Mehrdad Moharrami, and R.~Srikant.
\newblock Performance {Bounds} for {Policy}-{Based} {Average} {Reward}
  {Reinforcement} {Learning} {Algorithms}.
\newblock In \emph{Advances in {Neural} {Information} {Processing} {Systems}},
  2023.

\bibitem[Murthy et~al.(2024)Murthy, Grosof, Maguluri, and
  Srikant]{murthy_performance_2024}
Yashaswini Murthy, Isaac Grosof, Siva~Theja Maguluri, and R.~Srikant.
\newblock Performance of {NPG} in {Countable} {State}-{Space} {Average}-{Cost}
  {RL}, 2024.
\newblock arXiv:2405.20467.

\bibitem[Neely(2010)]{neely_stochastic_2010}
Michael~J. Neely.
\newblock \emph{Stochastic {Network} {Optimization} with {Application} to
  {Communication} and {Queueing} {Systems}}.
\newblock Morgan \& Claypool Publishers, 2010.

\bibitem[Nemirovsky and Yudin(1983)]{nemirovskij_problem_1983}
A.~S. Nemirovsky and D.~B. Yudin.
\newblock \emph{Problem {Complexity} and {Method} {Efficiency} in
  {Optimization}}.
\newblock Wiley, 1983.

\bibitem[Puterman(1994)]{puterman_markov_1994}
Martin~L. Puterman.
\newblock \emph{Markov {Decision} {Processes}: {Discrete} {Stochastic}
  {Dynamic} {Programming}}.
\newblock Wiley, 1994.
\newblock \doi{10.1002/9780470316887}.

\bibitem[Ross(1983)]{ross_introduction_1983}
Sheldon~M. Ross.
\newblock \emph{Introduction to {Stochastic} {Dynamic} {Programming}}.
\newblock Academic Press, 1983.

\bibitem[Schulman et~al.(2017{\natexlab{a}})Schulman, Levine, Moritz, Jordan,
  and Abbeel]{schulman_trust_2017}
John Schulman, Sergey Levine, Philipp Moritz, Michael~I. Jordan, and Pieter
  Abbeel.
\newblock Trust {Region} {Policy} {Optimization}, 2017{\natexlab{a}}.
\newblock arXiv:1502.05477.

\bibitem[Schulman et~al.(2017{\natexlab{b}})Schulman, Wolski, Dhariwal,
  Radford, and Klimov]{schulman_proximal_2017}
John Schulman, Filip Wolski, Prafulla Dhariwal, Alec Radford, and Oleg Klimov.
\newblock Proximal {Policy} {Optimization} {Algorithms}, 2017{\natexlab{b}}.
\newblock arXiv:1707.06347.

\bibitem[Shani et~al.(2020)Shani, Efroni, and Mannor]{shani_adaptive_2020}
Lior Shani, Yonathan Efroni, and Shie Mannor.
\newblock Adaptive {Trust} {Region} {Policy} {Optimization}: {Global}
  {Convergence} and {Faster} {Rates} for {Regularized} {MDPs}.
\newblock \emph{Proceedings of the AAAI Conference on Artificial Intelligence},
  34\penalty0 (04):\penalty0 5668--5675, 2020.
\newblock \doi{10.1609/aaai.v34i04.6021}.

\bibitem[Srikant and Ying(2014)]{srikant_communication_2014}
R.~Srikant and Lei Ying.
\newblock \emph{Communication {Networks}: {An} {Optimization}, {Control}, and
  {Stochastic} {Networks} {Perspective}}.
\newblock Cambridge University Press, 2014.

\bibitem[Sun et~al.(2024)Sun, He, Miao, and Zou]{sun_policy_2024}
Zhongchang Sun, Sihong He, Fei Miao, and Shaofeng Zou.
\newblock Policy {Optimization} for {Robust} {Average} {Cost} {MDPs}.
\newblock In \emph{Advances in {Neural} {Information} {Processing} {Systems}},
  2024.

\bibitem[Sutton and Barto(2018)]{sutton_reinforcement_2020}
Richard~S. Sutton and Andrew~G. Barto.
\newblock \emph{Reinforcement {Learning}: {An} {Introduction}}.
\newblock The MIT Press, second edition, 2018.

\bibitem[Sutton et~al.(1999)Sutton, McAllester, Singh, and
  Mansour]{sutton_policy_1999}
Richard~S. Sutton, David McAllester, Satinder Singh, and Yishay Mansour.
\newblock Policy {Gradient} {Methods} for {Reinforcement} {Learning} with
  {Function} {Approximation}.
\newblock In \emph{Advances in {Neural} {Information} {Processing} {Systems}},
  1999.

\bibitem[Tropp(2012)]{tropp_user-friendly_2012}
Joel~A. Tropp.
\newblock User-{Friendly} {Tail} {Bounds} for {Sums} of {Random} {Matrices}.
\newblock \emph{Foundations of Computational Mathematics}, 12\penalty0
  (4):\penalty0 389--434, 2012.
\newblock \doi{10.1007/s10208-011-9099-z}.

\bibitem[Wang et~al.(2024)Wang, Blanchet, and Glynn]{wang_optimal_2024}
Shengbo Wang, Jose Blanchet, and Peter Glynn.
\newblock Optimal {Sample} {Complexity} for {Average} {Reward} {Markov}
  {Decision} {Processes}, 2024.
\newblock arXiv:2310.08833.

\bibitem[Xiao(2022)]{xiao_convergence_2022}
Lin Xiao.
\newblock On the {Convergence} {Rates} of {Policy} {Gradient} {Methods}.
\newblock \emph{Journal of Machine Learning Research}, 2022.

\bibitem[Yu(1997)]{yu_assouad_1997}
Bin Yu.
\newblock Assouad, {Fano}, and {Le} {Cam}.
\newblock In David Pollard, Erik Torgersen, and Grace~L. Yang, editors,
  \emph{Festschrift for {Lucien} {Le} {Cam}}, pages 423--435. Springer, 1997.

\bibitem[Yuan et~al.(2023)Yuan, Du, Gower, Lazaric, and Xiao]{yuan_linear_2023}
Rui Yuan, Simon~S. Du, Robert~M. Gower, Alessandro Lazaric, and Lin Xiao.
\newblock Linear {Convergence} of {Natural} {Policy} {Gradient} {Methods} with
  {Log}-{Linear} {Policies}, 2023.
\newblock arXiv:2210.01400.

\bibitem[Zamir and Hou(2024)]{zamir_deep_2024}
Nida Zamir and I.-Hong Hou.
\newblock Deep {Index} {Policy} for {Multi}-{Resource} {Restless} {Matching}
  {Bandit} and {Its} {Application} in {Multi}-{Channel} {Scheduling}, 2024.
\newblock arXiv:2408.07205.

\bibitem[Zhan et~al.(2023)Zhan, Cen, Huang, Chen, Lee, and
  Chi]{zhan_policy_2023}
Wenhao Zhan, Shicong Cen, Baihe Huang, Yuxin Chen, Jason~D. Lee, and Yuejie
  Chi.
\newblock Policy {Mirror} {Descent} for {Regularized} {Reinforcement}
  {Learning}: {A} {Generalized} {Framework} with {Linear} {Convergence}, 2023.
\newblock arXiv:2105.11066.

\end{thebibliography}

\appendix
\setcounter{theorem}{0}

\section*{Supplementary Material}
\paragraph{Organization.} This appendix contains detailed related-work comparisons, deferred proofs, and additional technical details. Appendix~\ref{ap:related-work} provides detailed comparisons to closely related work. Appendix~\ref{ap:preliminaries} collects the background facts of Section~\ref{sec:preliminaries}, with references, and proves Lemma~\ref{lem:master}. Appendix~\ref{ap:proofs-tabular} contains the convergence and sample complexity proofs for the tabular case. Appendix~\ref{ap:proofs-non-tabular} contains the convergence and sample complexity proofs under linear function approximation. Appendix~\ref{ap:mixing-assumption-details} proves Lemma~\ref{lem:uniform-mix} and provides additional details on mixing times. Appendix~\ref{ap:numerical-experiments} presents numerical experiments.

\section{Detailed related work} \label{ap:related-work}

This appendix expands on the landscape summarized in Section~\ref{sec:related-work}. We discuss three closest comparators in Appendix~\ref{ap:rw-closest} (Li et al., Sun et al., Lee and Ryu) and broader policy-optimization lines in Appendix~\ref{ap:rw-broader}. Each entry uses the same two-part format: a statement of the work's results, followed by a comparison to ours.

\subsection{Closest comparators} \label{ap:rw-closest}

\paragraph{Li, Wu, Lan \citep{li_stochastic_2024}.}
\emph{Results.} \citet{li_stochastic_2024} analyze stochastic policy mirror descent for average-reward MDPs. Their Theorem~1 gives an $O(1/\sqrt{K})$ rate for a uniformly random iterate of the \emph{unregularized} objective with a constant step size. Under strongly convex regularization, their Theorem~2 gives $O(1/K)$ with decaying step sizes, and their Theorem~3, with a constant step size, bounds the expected divergence of the last iterate from the regularized optimum by a geometrically decaying term, an $\widetilde O(1/K)$ stochastic term, and a bias floor. With a generative model, Corollary~1 translates Theorem~1 into a tabular sample-complexity guarantee of $\widetilde O(|\mathcal S||\mathcal A|\,t_{\mathrm{mix}}^3/\varepsilon^2)$ in expectation. In the unregularized LFA case, Corollary~3 uses a single trajectory and runs $\widetilde O(\kappa^2/\varepsilon^2)$ actor iterations, where $\kappa^2$ bounds the second moment of the critic outputs (their Assumption~3), each with an EVRTD critic call whose bias control costs $\widetilde O(t_{\mathrm{mix}}^3/\bar\mu)$ transitions ($\bar\mu$ is their lower bound, over all policies, on the smallest eigenvalue of the feature covariance), for $\widetilde O(t_{\mathrm{mix}}^3\kappa^2/(\bar\mu\,\varepsilon^2))$ in total; their bound $\kappa^2 \lesssim t_{\mathrm{mix}}^2$, up to conditioning and approximation terms, gives an effective $t_{\mathrm{mix}}^5$ mixing dependence. Its conditioning term divides by their smallest exploratory state--action occupancy over all policies (their (4.22) and Corollary~3), which is at most $1/(|\mathcal S||\mathcal A|)$, so their stated bound contains a factor of at least $|\mathcal S||\mathcal A|$, the value Table~\ref{tab:comparison} lists.

\emph{Comparison.} Our analysis is complementary along three axes. First, we obtain linear convergence to the optimum up to critic error, without external regularization, using geometric step-size growth rather than constant steps (Theorem~\ref{thm:avg-pmd-linear-inexact-limsup}, Corollary~\ref{cor:exact-pmd}, Theorem~\ref{thm:non-tabular-linear-convergence}); the geometric term of their Theorem~3 concerns the regularized objective. Second, with the same generative-model oracle as their Corollary~1, our tabular sample complexity improves their $|\mathcal S||\mathcal A|\,t_{\mathrm{mix}}^3$ dependence to $|\mathcal S|\,t_{\mathrm{mix}}^3+|\mathcal S||\mathcal A|\,t_{\mathrm{mix}}^2$ (Theorem~\ref{thm:pmd_sample_complexity}): our critic pays $t_{\mathrm{mix}}^3$ per state, for the rollouts that estimate $h^\pi$, and only $t_{\mathrm{mix}}^2$ per state--action pair (Section~\ref{sec:sample-complexity-tabular}). Our LFA sample complexity, also on a single trajectory, has $t_{\mathrm{mix}}^3$ rather than their $t_{\mathrm{mix}}^5$ dependence, through the trade-off described in Obstacle~2 of Section~\ref{sec:sample-complexity-non-tabular} (Theorem~\ref{thm:non_tabular_pmd_sample_complexity_generative}): our linear-rate actor needs only $\widetilde O(\bar\chi)$ critic calls, each made accurate with high probability by tail-averaged SGD with confidence boosting on $\widetilde O(d\,t_{\mathrm{mix}}^2/\varepsilon_{\mathrm{stat}}^2)$ rollout labels when $\varepsilon_{\mathrm{app}}\le\varepsilon_{\mathrm{stat}}$, up to a lower-order term (Lemma~\ref{lem:ap-w-sample-complexity}), costing $\widetilde O(t_{\mathrm{mix}})$ transitions each, whereas their sublinear actor needs $\widetilde O(\kappa^2/\varepsilon^2)$ bias-controlled calls costing $\widetilde O(t_{\mathrm{mix}}^3/\bar\mu)$ transitions each. Our labels pay the same quadratic second moment as their $\kappa^2$, but inside the critic. Our bound carries no such occupancy factor. Both our sample complexities, with the same oracles as theirs, are also high-probability and last-iterate rather than in expectation for a random iterate (Theorems~\ref{thm:pmd_sample_complexity}, \ref{thm:non_tabular_pmd_sample_complexity_generative}). Third, Propositions~\ref{prop:tabular_lower_bound} and~\ref{prop:nontabular_lower_bound} establish that both critics are optimal up to logarithmic factors, the LFA one in its leading term.

\paragraph{Sun, He, Miao, Zou \citep{sun_policy_2024}.}
\emph{Results.} \citet{sun_policy_2024} study robust PMD for ergodic average-cost MDPs under rectangular transition uncertainty, with Bregman divergence specialized to squared Euclidean. Their Theorem~4.6 establishes linear convergence at rate $(1 - 1/M_0)^t$ for an increasing step-size schedule $\eta_t/\eta_{t-1} \geq M_0^2/(M_0 - 1)$, where $M_0 := \sup_{\pi \in \Pi} \|d^{\pi^\star}/d^\pi\|_\infty$ is the global worst-case distribution mismatch. Specialized to robustness radius $R = 0$, this becomes a linear-convergence result for standard (non-robust) ergodic exact PMD.

\emph{Comparison.} At $R = 0$ their result partially overlaps with the exact case ($\varepsilon_Q = 0$) of our Theorem~\ref{thm:avg-pmd-linear-inexact-limsup}. Our analysis admits any upper bound $\bar\chi$ on $\sup_t \chi_t$ (with $\sup_t \chi_t \leq M_0 \leq \underline d^{-1}$), and tighter $\bar\chi$ yield strictly sharper rates than Sun's worst-case $M_0$ in path-dependent regimes. Our analysis holds for general Legendre-Bregman divergences, not only squared Euclidean. Taking $\bar\chi = M_0$ matches Sun's rate $1-1/M_0$ at growth ratio $M_0/(M_0-1)$, a factor $M_0$ milder than Sun's $M_0^2/(M_0-1)$; tighter $\bar\chi$ trade larger growth ratios for faster path-dependent rates, a trade-off Sun's worst-case analysis cannot expose. Our Corollary~\ref{cor:exact-pmd} additionally provides an asymptotic rate $\limsup_t \Delta_t^{1/t} \leq 1 - 1/\chi$ with a superlinear regime when $\chi = 1$. Finally, we extend the analysis to inexact critics (Theorem~\ref{thm:avg-pmd-linear-inexact-limsup}), linear function approximation (Theorem~\ref{thm:non-tabular-linear-convergence}), and end-to-end sample complexity with matching information-theoretic lower bounds on the critic's sample complexity (Propositions~\ref{prop:tabular_lower_bound},~\ref{prop:nontabular_lower_bound}). Sun's framework in turn handles rectangular robust uncertainty, which ours does not.

\paragraph{Lee and Ryu \citep{lee_policy_2026}.}
\emph{Results.} Concurrent work by \citet{lee_policy_2026} extends policy gradient analysis to multichain average-reward MDPs, a strict generalization of our ergodic assumption. They introduce recurrent and transient visitation measures, establish a multichain policy gradient theorem, and analyze an $\alpha$-clipped policy mirror ascent algorithm restricted to the set $\Pi_\alpha := \{\pi : \pi(a\mid s) \geq \alpha\}$. Their Theorem~10 establishes linear convergence of the optimality gap with respect to the best $\alpha$-clipped policy. For any fixed $\alpha>0$ this leaves a clipping gap to the positive-policy optimum $J^\star_{+,\mu} := \sup_{\pi \in \Pi_+} \mathbb E_{s \sim \mu}[J^\pi(s)]$, which is nonzero in general and vanishes as $\alpha \to 0$; here $\Pi_+$ is the simplex interior. This restriction is justified by the discontinuity of $\pi \mapsto J^\pi$ on the simplex boundary in multichain MDPs. Their Theorem~13 gives sample complexity $\widetilde O(t_{\mathrm{tar}}^3\, \|Q\|_\infty^4\, R^3\, C_\alpha^6\, B_\alpha^6\, |\mathcal S||\mathcal A|/\varepsilon^6)$ under an adaptive step-size schedule (the constant step-size variant has a worse $\varepsilon^{-12}$ dependence).

\emph{Comparison.} Their axis of generalization (ergodic $\to$ multichain) and ours (sharper ergodic analysis) are complementary. Their sample complexity is $\widetilde O(\varepsilon^{-6})$ versus our $\widetilde O(\varepsilon^{-2})$ (with the critic's sample complexity tight by Proposition~\ref{prop:tabular_lower_bound}), and their optimality notion is positive-policy rather than global. Within the ergodic setting, we additionally obtain asymptotic and superlinear rates (Corollary~\ref{cor:exact-pmd}) and linear function approximation results (Section~\ref{sec:non-tabular}), neither of which the multichain analysis addresses. Concurrent work by \citet{lee_learning_2026} extends this $\alpha$-clipped scheme to a single trajectory in weakly communicating MDPs, with high-probability tabular sample complexities of $\widetilde O(\varepsilon^{-4})$ for that policy-based method and $\widetilde O(\varepsilon^{-2})$ for a value-based one. Their bounds are stated through hitting, cover, and target times; on a single trajectory, ours give a policy-based $\widetilde O(\varepsilon^{-2})$ sample complexity with high probability under linear function approximation, with the leading term of the critic's sample complexity tight by Proposition~\ref{prop:nontabular_lower_bound}.

\subsection{Broader policy-optimization lines} \label{ap:rw-broader}

The remaining works address average-reward policy optimization in parallel directions (sublinear rates, general parameterization, or infinite-state extensions). We use a uniform single-paragraph format below.

\paragraph{Kumar et al.\ \citep{kumar_global_2024}.}
\emph{Results.} \citet{kumar_global_2024} establish a global $O(1/T)$ suboptimality rate for projected policy gradient in ergodic tabular average-reward MDPs. \emph{Comparison.} Their algorithm (projected PG in Euclidean geometry) and sublinear rate differ from our linear PMD result (Theorem~\ref{thm:avg-pmd-linear-inexact-limsup} with $\varepsilon_Q=0$) under general Bregman geometry.

\paragraph{Murthy and Srikant \citep{murthy_convergence_2023}, Murthy et al.\ \citep{murthy_performance_2023}.}
\emph{Results.} \citet{murthy_convergence_2023} establish $O(1/T)$ convergence for natural policy gradient and mirror-descent-like methods in ergodic average-reward MDPs under irreducibility, with discussion of non-tabular settings. \citet{murthy_performance_2023} develop approximate-policy-iteration performance bounds with step-size-dependent tradeoffs (their Remark~4.6 notes that larger steps reduce some terms while increasing others, making the overall effect non-transparent). For a mirror-descent update with a TD critic and linear function approximation (their Corollary~4.5), the bound keeps a term proportional to $\log(1/\omega)/\beta$, where $\beta$ is the fixed step size and $\omega$ the smallest probability the iterates assign to an optimal action, so it does not vanish as the critic becomes exact. \emph{Comparison.} Our master-lemma analysis (Lemma~\ref{lem:master}) makes the step-size effect explicit and upgrades the $O(1/T)$ rate of \citet{murthy_convergence_2023} to linear via geometric step-size growth. Our bound has no such term: with growing step sizes, the error floor of Theorem~\ref{thm:non-tabular-linear-convergence} vanishes with the critic error.

\paragraph{Bai et al.\ \citep{bai_regret_2024}.}
\emph{Results.} \citet{bai_regret_2024} analyze general-parameterized policy gradient for average-reward MDPs, obtaining $\widetilde O(T^{-1/4})$ optimality error and $\widetilde O(T^{3/4})$ regret. \emph{Comparison.} Their work is broader in parameterization but weaker in rate; our PMD analysis is restricted to tabular and linear function approximation but yields linear convergence.

\paragraph{Ganesh et al.\ \citep{ganesh_order-optimal_2025, ganesh_sharper_2025}.}
\emph{Results.} \citet{ganesh_order-optimal_2025} improve regret for general-parameterized average-reward policy gradient via variance-reduction techniques, achieving near order-optimal $\widetilde O(\sqrt{T})$ rates. \citet{ganesh_sharper_2025} develop a multi-level-Monte-Carlo natural actor-critic attaining $\widetilde O(T^{-1/2})$ global convergence without requiring knowledge of mixing or hitting times. \emph{Comparison.} Both works are broader in parameterization than ours but do not establish linear rates or the $t_{\mathrm{mix}}^3$ mixing dependence we obtain.

\paragraph{Grosof et al.\ \citep{grosof_convergence_2025}.}
\emph{Results.} \citet{grosof_convergence_2025} establish $O(1/\sqrt{T})$ global convergence for natural policy gradient in a family of countably-infinite-state queueing MDPs, initialized from the MaxWeight policy. \emph{Comparison.} Extends NPG beyond finite state spaces; complementary to our finite-state PMD analysis.

\section{Preliminaries} \label{ap:preliminaries}

\subsection{Existence of the Poisson equation and differential values} \label{ap:poisson-existence}

\begin{fact}[Poisson equation and differential values]
\label{fact:ap-poisson}
For each policy $\pi$, there exist a scalar $\rho(\pi)$
and a function $h^\pi:\mathcal S\to\mathbb R$ such that
\begin{equation}
\label{eq:ap-poisson}
\rho(\pi) + h^\pi \;=\; c_\pi + P_\pi h^\pi.
\end{equation}
The pair $(\rho(\pi),h^\pi)$ is unique up to adding a constant to $h^\pi$. Throughout the paper, we fix the \emph{centered} convention $d^\pi h^\pi = 0$, under which $h^\pi(s) = \lim_{H\to\infty} \mathbb{E}[\sum_{k=0}^{H-1}(c(s_k, a_k) - \rho(\pi)) \mid s_0 = s]$ and $Q^\pi(s,a) = c(s,a) - \rho(\pi) + \mathbb{E}[h^\pi(s_1) \mid s_0 = s, a_0 = a]$.
\end{fact}

\begin{proof}
Fix a policy $\pi$, let $\mathbf 1 \in \mathbb{R}^{|\St|}$ denote the all-ones vector, and write $g \coloneqq c_\pi - \rho(\pi)\mathbf 1$. For an ergodic chain,~\eqref{eq:ap-poisson} has a solution, unique up to adding a constant to $h^\pi$; see, e.g., \citet{puterman_markov_1994} and \citet[Section~2]{li_stochastic_2024}. Applying $d^\pi$ to~\eqref{eq:ap-poisson} and using $d^\pi P_\pi = d^\pi$ gives $\rho(\pi) = d^\pi c_\pi$, so $d^\pi g = 0$, and the centered convention $d^\pi h^\pi = 0$ selects a unique $h^\pi$.

Finally, we verify the limit representation. Under Assumption~\ref{as:ergodic}, $P_\pi^k \to \mathbf 1 d^\pi$ geometrically fast (Lemma~\ref{lem:per-policy-mixing}), and $d^\pi g = 0$, so $P_\pi^k g = (P_\pi^k - \mathbf 1 d^\pi)\,g \to 0$ geometrically and $\tilde h \coloneqq \sum_{k\ge 0} P_\pi^k g$ converges. Telescoping gives $(I-P_\pi)\tilde h = g$, and $d^\pi\tilde h = \sum_{k\ge 0} d^\pi P_\pi^k g = \sum_{k\ge 0} d^\pi g = 0$, so $\tilde h = h^\pi$ by uniqueness under the centered convention. Since $(P_\pi^k g)(s) = \mathbb{E}[c(s_k,a_k)-\rho(\pi)\mid s_0=s]$, this is the stated limit for $h^\pi$; the expression for $Q^\pi$ follows by conditioning on the first transition in~\eqref{eq:diff-Q-rolling}.
\end{proof}

\subsection{Fact~\ref{fact:avg-pdl} (Performance difference lemma)}

Fact~\ref{fact:avg-pdl} is the unregularized case of \citet[Lemma~2]{li_stochastic_2024}. In our conventions it reads, for any policies $\pi,\pi'$ under Assumption~\ref{as:ergodic},
\begin{equation}
\label{eq:ap-avg-pdl-qform}
    \rho(\pi') - \rho(\pi)
    \;=\;
    \E_{s\sim d^{\pi'}}
    \Bigl[
        \bigl\langle Q^\pi_s,\,\pi'_s - \pi_s\bigr\rangle
    \Bigr].
\end{equation}

\begin{proof}
By the definition of $Q^\pi$ and~\eqref{eq:ap-poisson}, $\langle Q^\pi_s, \pi_s\rangle = h^\pi(s)$ and $\langle Q^\pi_s, \pi'_s\rangle = \langle c_s, \pi'_s\rangle - \rho(\pi) + \E[h^\pi(s_1) \mid s_0 = s, a_0 \sim \pi'_s]$. Subtracting, averaging over $s \sim d^{\pi'}$, and using the stationarity of $d^{\pi'}$ under $P_{\pi'}$ cancels the $h^\pi$ terms; $\rho(\pi') = \E_{s\sim d^{\pi'}}[\langle c_s, \pi'_s\rangle]$ then gives~\eqref{eq:ap-avg-pdl-qform}.
\end{proof}

\subsection{Fact~\ref{fact:avg-policy-gradient} (Policy gradient theorem)}

Fact~\ref{fact:avg-policy-gradient} is the average-reward policy gradient theorem \citep[Theorem~1]{sutton_policy_1999} for the direct parameterization; see also the unregularized case of \citet[Lemma~1]{li_stochastic_2024}. It also follows from Fact~\ref{fact:avg-pdl}: perturbing $\pi_s$ alone to $\pi_s + \varepsilon u \in \mathcal P(\mathcal A)$ with $\sum_a u(a) = 0$ gives a policy $\pi_\varepsilon$ with $\rho(\pi_\varepsilon) - \rho(\pi) = \varepsilon\, d^{\pi_\varepsilon}(s)\, \langle u, Q^\pi_s\rangle$, and $d^{\pi_\varepsilon}(s) \to d^\pi(s)$ as $\varepsilon \to 0$ by the continuity of $\pi \mapsto d^\pi$ shown in the proof of Lemma~\ref{lem:uniform-d-bound}.

\subsection{Fact~\ref{fact:three-point} (Three-point inequality for PMD)}
Fact~\ref{fact:three-point} is the three-point descent lemma of \citet[Lemma~6]{xiao_convergence_2022}, applied on $\mathcal P(\mathcal A)$ with the linear function $\varphi(q) = \eta_t\langle \tilde Q^t_s, q\rangle$ and reference point $\pi^t_s$: for every $p \in \mathcal P(\mathcal A)$ it gives $\varphi(\pi^{t+1}_s) + D(\pi^{t+1}_s, \pi^t_s) \le \varphi(p) + D(p, \pi^t_s) - D(p, \pi^{t+1}_s)$, which rearranges to the stated inequality. The lemma also keeps $\pi^{t+1}_s$ in the relative interior of the domain of $F$ whenever $\pi^t_s$ is, so it applies at every iteration from an interior $\pi^0$.

\subsection{Proof of Lemma~\ref{lem:uniform-d-bound} (Uniform boundedness of stationary distributions)}

\begin{proof}
Fix $\pi$. By irreducibility, $d^\pi(s) > 0$ for every $s$: if $d^\pi(s_0) = 0$, stationarity, $0 = \sum_s d^\pi(s) P_\pi(s_0 \mid s)$, forces $d^\pi(s) = 0$ wherever $P_\pi(s_0 \mid s) > 0$, and iterating along the chain gives $d^\pi \equiv 0$, a contradiction. Next, $\pi \mapsto d^\pi$ is continuous: $\pi \mapsto P_\pi$ is continuous, and $d^\pi$ is the unique solution of the linear system obtained from $d(P_\pi - I) = 0$ by replacing its last equation with $d\,\mathbf 1 = 1$; since $1$ is a simple eigenvalue of the ergodic matrix $P_\pi$, this system is nonsingular, and matrix inversion is continuous. Finally, $\Pi$ is compact, so the continuous function $(\pi, s) \mapsto d^\pi(s)$ attains its minimum on $\Pi \times \mathcal S$, which is positive; this minimum is $\underline d$. Appendix~\ref{ap:mixing-assumption-details} derives the explicit lower bound under a uniform minorization condition.
\end{proof}

\subsection{Proof of Lemma~\ref{lem:master} (Master convergence for inexact PMD)} \label{ap:master-lemma}

\begin{proof}
The proof has two parts. We first derive a one-step Lyapunov recursion that holds whenever the step-size and mismatch conditions are satisfied at time $t$ with some parameter $\bar\chi > 1$; then we iterate this recursion from $t = 0$ to obtain~(i), and iterate it on a tail with $\bar\chi = \chi + \zeta$ (for any fixed $\zeta > 0$) to obtain~(ii).

\paragraph{One-step recursion.}
Fix $t \geq 0$ and suppose $\chi_{t+1} \leq \bar\chi$ and $\eta_t/\eta_{t-1} \geq \bar\chi/(\bar\chi - 1)$ for some $\bar\chi > 1$. Applying Fact~\ref{fact:three-point} with $p = \pi^\star_s$ to the generalized PMD update~\eqref{eq:pmd-update} and dropping the nonnegative term $D(\pi^{t+1}_s, \pi^t_s) \geq 0$ gives
$\eta_t \langle \tilde Q^t_s, \pi^{t+1}_s - \pi^\star_s\rangle \leq D(\pi^\star_s, \pi^t_s) - D(\pi^\star_s, \pi^{t+1}_s)$.
Decompose $\pi^{t+1}_s - \pi^\star_s = (\pi^{t+1}_s - \pi^t_s) + (\pi^t_s - \pi^\star_s)$ and take expectation over $s \sim d^\star$:
\begin{equation} \label{eq:ap-master-step1}
\eta_t\,\mathbb{E}_{s\sim d^\star}\!\bigl[\langle \tilde Q^t_s, \pi^t_s - \pi^\star_s\rangle\bigr] + \eta_t\,\mathbb{E}_{s\sim d^\star}\!\bigl[\langle \tilde Q^t_s, \pi^{t+1}_s - \pi^t_s\rangle\bigr] \leq D_t^\star - D_{t+1}^\star.
\end{equation}
By Fact~\ref{fact:avg-pdl} and the definition of $B_t$,
$\mathbb{E}_{s\sim d^\star}[\langle \tilde Q^t_s, \pi^t_s - \pi^\star_s\rangle] \geq \Delta_t - B_t$.
By~\eqref{eq:pmd-sign} the integrand $\langle \tilde Q^t_s, \pi^{t+1}_s - \pi^t_s\rangle$ is pointwise nonpositive, so the change-of-measure ratio $d^\star(s)/d^{t+1}(s) \leq \bar\chi$ gives $\mathbb{E}_{s\sim d^\star}[\langle \tilde Q^t_s, \pi^{t+1}_s - \pi^t_s\rangle] \geq \bar\chi\,\mathbb{E}_{s\sim d^{t+1}}[\langle \tilde Q^t_s, \pi^{t+1}_s - \pi^t_s\rangle] \geq \bar\chi\bigl((\Delta_{t+1} - \Delta_t) - A_t\bigr)$, where the second inequality uses Fact~\ref{fact:avg-pdl} and the definition of $A_t$. Substituting both into~\eqref{eq:ap-master-step1} and rearranging,
\[
\Delta_{t+1} + \frac{D_{t+1}^\star}{\eta_t \bar\chi} \leq (1 - 1/\bar\chi)\Delta_t + \frac{D_t^\star}{\eta_t \bar\chi} + \frac{\bar\chi A_t + B_t}{\bar\chi}.
\]
The step-size growth condition gives $1/(\eta_t \bar\chi) \leq (1 - 1/\bar\chi) \cdot 1/(\eta_{t-1}\bar\chi)$, so defining the Lyapunov sequence $\Psi_t := \Delta_t + D_t^\star/(\eta_{t-1}\bar\chi)$ yields the one-step recursion
\begin{equation} \label{eq:ap-master-onestep}
\Psi_{t+1} \leq (1 - 1/\bar\chi)\,\Psi_t + \frac{\bar\chi A_t + B_t}{\bar\chi}.
\end{equation}

\paragraph{Part (i): non-asymptotic rate.}
Under the hypotheses of~(i), the one-step recursion~\eqref{eq:ap-master-onestep} holds at every $t \geq 0$ with $\bar\chi$ the global uniform bound. Using $A_t \leq \bar A$, $B_t \leq \bar B$ and iterating from $t = 0$,
\[
\Psi_t \leq (1 - 1/\bar\chi)^t \Psi_0 + \frac{\bar\chi \bar A + \bar B}{\bar\chi}\sum_{i=0}^{t-1}(1 - 1/\bar\chi)^i = (1 - 1/\bar\chi)^t \Psi_0 + (\bar\chi \bar A + \bar B)\bigl(1 - (1 - 1/\bar\chi)^t\bigr).
\]
Since $\Delta_t \leq \Psi_t$, subtracting $\bar\chi \bar A + \bar B$ and taking the positive part yields~\eqref{eq:master-nonasymp}.

\paragraph{Part (ii): asymptotic rate.}
Let $\chi := \limsup_t \chi_t$ and fix $\zeta > 0$, setting $\hat\chi := \chi + \zeta$. By the $\limsup$ definition, $\chi_{t+1} \leq \hat\chi$ for all $t$ beyond some $t_0(\zeta)$. By $\eta_t/\eta_{t-1} \to \infty$, there exists $t_1(\zeta)$ such that $\eta_t/\eta_{t-1} \geq \hat\chi/(\hat\chi - 1)$ for $t \geq t_1(\zeta)$. Set $t_\zeta := \max\{t_0(\zeta), t_1(\zeta)\}$; both hypotheses of the one-step recursion hold with $\bar\chi = \hat\chi$ for all $t \geq t_\zeta$. Iterating~\eqref{eq:ap-master-onestep} from $t_\zeta$ gives
\[
\Psi_t \leq (1 - 1/\hat\chi)^{t - t_\zeta}\, \Psi_{t_\zeta} + (\hat\chi \bar A + \bar B)\bigl(1 - (1 - 1/\hat\chi)^{t - t_\zeta}\bigr) \qquad \forall\, t \geq t_\zeta.
\]
Hence $(\Delta_t - (\hat\chi \bar A + \bar B))_+ \leq (1 - 1/\hat\chi)^{t - t_\zeta}\, \Psi_{t_\zeta}$. Here $\Psi_{t_\zeta} < \infty$: $D_0^\star < \infty$, and the one-step inequality, which holds before the growth condition is invoked, keeps every $D_t^\star$ finite. Taking $t$-th roots, $\limsup$, and noting $\Psi_{t_\zeta}^{1/t} \to 1$ yields~\eqref{eq:master-asymp}: $\limsup_{t\to\infty}(\Delta_t - ((\chi+\zeta) \bar A + \bar B))_+^{1/t} \leq 1 - 1/(\chi+\zeta)$. The same bound also gives $\limsup_t \Delta_t \leq \hat\chi \bar A + \bar B$; letting $\zeta \downarrow 0$ yields $\limsup_t \Delta_t \leq \chi \bar A + \bar B$. When $\bar A = 0$ the floor $\hat\chi \bar A + \bar B$ reduces to $\bar B$ independently of $\zeta$, so~\eqref{eq:master-asymp} becomes $\limsup_{t\to\infty}(\Delta_t - \bar B)_+^{1/t} \leq 1 - 1/(\chi+\zeta)$; letting $\zeta \downarrow 0$ gives the sharper rate $1 - 1/\chi$.
\end{proof}

\subsection{Constant step sizes} \label{ap:constant-step}
The one-step inequality in the proof of Lemma~\ref{lem:master} holds before the growth condition is invoked: if $\sup_t \chi_t \leq \bar\chi$, then for every $t \geq 0$,
\[
\Delta_{t+1} + \frac{D_{t+1}^\star}{\eta_t \bar\chi} \;\leq\; \Bigl(1 - \frac{1}{\bar\chi}\Bigr)\Delta_t + \frac{D_t^\star}{\eta_t \bar\chi} + \frac{\bar\chi A_t + B_t}{\bar\chi}.
\]
With a constant step size $\eta_t \equiv \eta$, multiplying by $\bar\chi$ and summing over $t = 0, \dots, T-1$ telescopes the divergence terms. Dropping $D_T^\star \geq 0$ and using $A_t \leq \bar A$, $B_t \leq \bar B$ gives
\[
\bar\chi \sum_{t=1}^{T} \Delta_t - (\bar\chi - 1) \sum_{t=0}^{T-1} \Delta_t \;\leq\; \frac{D_0^\star}{\eta} + T\,(\bar\chi \bar A + \bar B).
\]
The left side equals $\sum_{t=1}^{T} \Delta_t + (\bar\chi - 1)(\Delta_T - \Delta_0) \geq \sum_{t=1}^{T} \Delta_t - (\bar\chi - 1)\Delta_0$, so
\[
\frac{1}{T}\sum_{t=1}^{T} \Delta_t \;\leq\; \frac{D_0^\star}{\eta T} + \frac{(\bar\chi - 1)\,\Delta_0}{T} + \bar\chi \bar A + \bar B,
\]
an $O(1/T)$ rate for the average suboptimality of the iterates, and hence for the best iterate, to the same error floor as Lemma~\ref{lem:master}(i). In the exact case ($\bar A = \bar B = 0$), the descent property $\rho(\pi^{t+1}) \leq \rho(\pi^t)$ of exact PMD (Section~\ref{sec:preliminaries}) makes $\Delta_t$ nonincreasing, so the same bound holds for the last iterate $\Delta_T$.

\paragraph{What the growth condition buys.} The growth condition is what makes the contraction of Lemma~\ref{lem:master}(i) uniform: rate $1 - 1/\bar\chi$ from $t = 0$ on every instance, with no dependence on advantage gaps. In the proof of Lemma~\ref{lem:master}, this requires the term $D_t^\star/(\eta_{t-1}\bar\chi)$ to contract along with $\Delta_t$, which the growth condition provides. Constant steps can still converge linearly, but at an instance-dependent rate: in a two-armed bandit, constant-step KL-PMD shrinks the odds of the suboptimal action by the same factor at every step, a factor set by $\eta$ and the gap between the two actions' costs. On every instance, the cumulative bound of Proposition~\ref{prop:ap-exact-geometries}(i) below gives a linear rate for constant step sizes, with an instance-dependent constant. The $\chi$-adaptive rate of Lemma~\ref{lem:master}(ii) and the superlinear regime of Corollary~\ref{cor:exact-pmd} also rest on growing step sizes.

\subsection{Exact PMD without a convergence assumption} \label{ap:exact-geometries}
Corollary~\ref{cor:exact-pmd} obtains a superlinear rate when $\{\pi^t\}$ converges to an optimal policy. This subsection removes that assumption for the KL divergence and the squared Euclidean distance. Both arguments rest on the optimal advantage $g^\star(s,a) := Q^{\pi^\star}(s,a) - h^{\pi^\star}(s)$, where $\pi^\star$ is any optimal policy.

\begin{lemma}[Optimal advantage] \label{lem:ap-optimal-advantage}
The optimal advantage $g^\star$ is nonnegative and the same for every optimal $\pi^\star$, and for every policy $\pi$,
\begin{equation} \label{eq:ap-gap-identity}
\rho(\pi) - \rho^\star \;=\; \mathbb{E}_{s\sim d^\pi}\bigl[\langle g^\star_s,\, \pi_s\rangle\bigr].
\end{equation}
In particular, $\pi$ is optimal if and only if each $\pi_s$ is supported on $\mathcal A^\star_s := \{a : g^\star(s,a) = 0\}$.
\end{lemma}

\begin{proof}
By~\eqref{eq:ap-poisson}, $\langle Q^{\pi^\star}_s, \pi^\star_s\rangle = h^{\pi^\star}(s)$, so $\langle g^\star_s, \pi^\star_s\rangle = 0$, and Fact~\ref{fact:avg-pdl} with $(\pi, \pi')$ replaced by $(\pi^\star, \pi)$ gives~\eqref{eq:ap-gap-identity}. If $g^\star(s_0, a_0) < 0$, the policy $\pi$ that plays $a_0$ at $s_0$ and agrees with $\pi^\star$ elsewhere has $\rho(\pi) - \rho^\star = d^\pi(s_0)\, g^\star(s_0, a_0) < 0$ by~\eqref{eq:ap-gap-identity} and Lemma~\ref{lem:uniform-d-bound}, contradicting the optimality of $\pi^\star$. Hence $g^\star \geq 0$, and since $d^\pi \geq \underline d > 0$, the right side of~\eqref{eq:ap-gap-identity} vanishes exactly when each $\pi_s$ is supported on $\mathcal A^\star_s$. Finally, if $\pi'$ is another optimal policy, then each $\pi'_s$ is supported on $\mathcal A^\star_s$, so $\langle Q^{\pi^\star}_s, \pi'_s\rangle = h^{\pi^\star}(s)$ for all $s$. As $\rho(\pi') = \rho^\star$, this says that $h^{\pi^\star}$ solves the Poisson equation~\eqref{eq:ap-poisson} of $\pi'$, so $h^{\pi^\star} - h^{\pi'}$ is a constant by Fact~\ref{fact:ap-poisson}. Then $Q^{\pi^\star} - Q^{\pi'}$ is the same constant, and $\pi'$ defines the same $g^\star$.
\end{proof}

\begin{proposition}[Exact PMD in the KL and squared Euclidean geometries] \label{prop:ap-exact-geometries}
Suppose not every policy is optimal. Let $\{\pi^t\}$ be generated by~\eqref{eq:pmd-update} with $\tilde Q^t = Q^t$, $\pi^0_s \in \operatorname{rint}(\mathcal P(\mathcal A))$ for all $s$, and nondecreasing step sizes $\eta_t > 0$. Let $g_{\min} := \min\{g^\star(s,a) : g^\star(s,a) > 0\}$ and $g_{\max} := \max_{s,a} g^\star(s,a)$.

\emph{(i) KL divergence.} For all sufficiently large $t$,
\[
\Delta_{t+1} \;\leq\; \frac{2 g_{\max}}{\underline d\, g_{\min}}\, e^{-\eta_t g_{\min}/2}\, \Delta_t,
\]
and there is a constant $C < \infty$ such that $\Delta_t \leq C \exp\bigl(-\tfrac{g_{\min}}{2}\sum_{j=0}^{t-1}\eta_j\bigr)$ for all $t \geq 0$. In particular, if $\eta_t \to \infty$, then $\Delta_{t+1}/\Delta_t \to 0$.

\emph{(ii) Squared Euclidean distance.} There exists $t_0$ such that $\pi^t = \pi^{t_0}$ for all $t \geq t_0$, and $\pi^{t_0}$ is optimal.
\end{proposition}

\begin{proof}
The proof has four steps. Once $\Delta_t$ is small, $Q^t$ separates optimal from suboptimal actions by at least $g_{\min}/2$, so every update moves probability away from suboptimal actions; Steps~3 and~4 quantify this for each geometry.

\paragraph{Step 1: $\Delta_t \to 0$.} By Lemma~\ref{lem:uniform-d-bound}, $\sup_t \chi_t \leq \bar\chi := 1/\underline d$, and $D_0^\star < \infty$ for both geometries (Remark~\ref{rmk:initial-divergence}). For nondecreasing $\eta_t$, summation by parts gives $\sum_{t=0}^{T-1}(D_t^\star - D_{t+1}^\star)/\eta_t \leq D_0^\star/\eta_0$, so the argument of Appendix~\ref{ap:constant-step} goes through with $\eta = \eta_0$ and gives $\Delta_T \leq (D_0^\star/\eta_0 + (\bar\chi - 1)\Delta_0)/T \to 0$.

\paragraph{Step 2: Suboptimal mass and action contrasts.} Let $\mathcal B_s := \mathcal A \setminus \mathcal A^\star_s$ and $m_t := \max_s \pi^t(\mathcal B_s \mid s)$. Since $g^\star = 0$ on $\mathcal A^\star_s$ and $g_{\min} \leq g^\star \leq g_{\max}$ on $\mathcal B_s$, \eqref{eq:ap-gap-identity} and $d^t \geq \underline d$ give
\begin{equation} \label{eq:ap-mass-sandwich}
\underline d\, g_{\min}\, m_t \;\leq\; \Delta_t \;\leq\; g_{\max}\, m_t,
\end{equation}
so $m_t \to 0$ by Step~1. Moreover, every limit point of $\{\pi^t\}$ is optimal, as in Corollary~\ref{cor:exact-pmd}, and $\pi \mapsto Q^\pi$ is continuous on the compact set $\Pi$: by Fact~\ref{fact:ap-poisson} it suffices that $h^\pi = (I - P_\pi + \mathbf 1 d^\pi)^{-1}(c_\pi - \rho(\pi)\mathbf 1)$ is continuous, which holds because the matrix is invertible under Assumption~\ref{as:ergodic} and continuous in $\pi$ by the proof of Lemma~\ref{lem:uniform-d-bound}. By Lemma~\ref{lem:ap-optimal-advantage}, every optimal policy has the action contrasts $Q^{\pi^\star}(s,b) - Q^{\pi^\star}(s,a) = g^\star(s,b) \geq g_{\min}$ for $a \in \mathcal A^\star_s$ and $b \in \mathcal B_s$. If some contrast of $Q^t$ stayed below $g_{\min}/2$ along a subsequence, the corresponding policies would have an optimal limit point, which continuity rules out. Since also $m_t \to 0$, there is $t_1$ such that for all $t \geq t_1$, all $s$, $a \in \mathcal A^\star_s$ and $b \in \mathcal B_s$,
\begin{equation} \label{eq:ap-contrast}
Q^t(s,b) - Q^t(s,a) \;\geq\; g_{\min}/2 \qquad\text{and}\qquad m_t \;\leq\; 1/2.
\end{equation}

\paragraph{Step 3: Contraction under the KL divergence.} Here~\eqref{eq:pmd-update} gives $\pi^{t+1}(a\mid s) = \pi^t(a\mid s)\, e^{-\eta_t Q^t(s,a)}/Z_t(s)$ with $Z_t(s) := \sum_{a'} \pi^t(a'\mid s)\, e^{-\eta_t Q^t(s,a')}$, so all iterates are interior and $\Delta_t > 0$ by~\eqref{eq:ap-mass-sandwich}. Fix $t \geq t_1$ and $s$, and let $q := \max_{a \in \mathcal A^\star_s} Q^t(s,a)$. Then $Z_t(s) \geq (1 - \pi^t(\mathcal B_s\mid s))\, e^{-\eta_t q}$, and $Q^t(s,b) \geq q + g_{\min}/2$ for $b \in \mathcal B_s$ by~\eqref{eq:ap-contrast}, so $\pi^{t+1}(b\mid s) \leq \pi^t(b\mid s)\, e^{-\eta_t g_{\min}/2}/(1 - \pi^t(\mathcal B_s\mid s))$. Summing over $b \in \mathcal B_s$ and maximizing over $s$,
\begin{equation} \label{eq:ap-kl-mass}
m_{t+1} \;\leq\; \frac{e^{-\eta_t g_{\min}/2}}{1 - m_t}\, m_t \;\leq\; 2 e^{-\eta_t g_{\min}/2}\, m_t,
\end{equation}
and~\eqref{eq:ap-mass-sandwich} gives $\Delta_{t+1} \leq g_{\max} m_{t+1} \leq 2 g_{\max} e^{-\eta_t g_{\min}/2} m_t \leq \frac{2 g_{\max}}{\underline d\, g_{\min}} e^{-\eta_t g_{\min}/2} \Delta_t$. For the cumulative bound, $m_t \to 0$ makes $1/(1 - m_t) \leq e^{\eta_0 g_{\min}/4}$ eventually, so~\eqref{eq:ap-kl-mass} and $\eta_t \geq \eta_0$ make $m_t$ decay geometrically. Hence $\sum_t m_t < \infty$ and $\prod_{t \geq t_1}(1 - m_t)^{-1} < \infty$, and iterating~\eqref{eq:ap-kl-mass} from $t_1$ and applying~\eqref{eq:ap-mass-sandwich} gives the bound for $t \geq t_1$; enlarging $C$ covers $t < t_1$.

\paragraph{Step 4: Finite termination under the squared Euclidean distance.} With $D(p, p') = \tfrac12\|p - p'\|_2^2$ (other scalings rescale $\eta_t$), $\pi^{t+1}_s$ is the Euclidean projection of $y := \pi^t_s - \eta_t Q^t_s$ onto $\mathcal P(\mathcal A)$, so $\pi^{t+1}(a\mid s) = (y_a - \lambda)_+$ with $\lambda$ chosen so that these sum to $1$. Since $\sum_{a \in \mathcal A^\star_s}(y_a - \lambda) \leq 1$ and $\sum_{a \in \mathcal A^\star_s}\pi^t(a\mid s) = 1 - \pi^t(\mathcal B_s\mid s)$, we have $\lambda \geq -\pi^t(\mathcal B_s\mid s)/|\mathcal A^\star_s| - \eta_t \bar q$, where $\bar q$ is the average of $Q^t(s, \cdot)$ over $\mathcal A^\star_s$. For $t \geq t_1$ and $b \in \mathcal B_s$, \eqref{eq:ap-contrast} and $\eta_t \geq \eta_0$ then give
\[
y_b - \lambda \;\leq\; \pi^t(b\mid s) + \pi^t(\mathcal B_s\mid s) - \eta_t\bigl(Q^t(s,b) - \bar q\bigr) \;\leq\; 2 m_t - \eta_0 g_{\min}/2.
\]
Take $t \geq t_1$ with $m_t < \eta_0 g_{\min}/4$. Then $\pi^{t+1}(b\mid s) = 0$ for all $s$ and $b \in \mathcal B_s$, so $\pi^{t_0}$ with $t_0 := t+1$ is optimal by Lemma~\ref{lem:ap-optimal-advantage}. An optimal $\pi$ is a fixed point of the update: by Lemma~\ref{lem:ap-optimal-advantage}, $Q^\pi_s = h^\pi(s)\mathbf 1 + g^\star_s$, so for $y := \pi_s - \eta_t Q^\pi_s$ the threshold $\lambda = -\eta_t h^\pi(s)$ gives $(y_a - \lambda)_+ = (\pi(a\mid s) - \eta_t g^\star(s,a))_+$, which equals $\pi(a\mid s)$ on $\mathcal A^\star_s$ and $0$ on $\mathcal B_s$ and thus sums to $1$. Hence $\pi^t = \pi^{t_0}$ for all $t \geq t_0$.
\end{proof}

With the doubling schedule of Section~\ref{sec:master}, $\sum_{j<t}\eta_j = \eta_0(2^t - 1)$, so the bound of part~(i) decays doubly exponentially in $t$. Proposition~\ref{prop:ap-exact-geometries} is the average-reward counterpart of \citet[Theorems~4.3 and~4.5]{lin_convergence_2024}, who prove both parts for discounted MDPs, whereas the earlier superlinear rates of \citet{khodadadian_linear_2021} and \citet[Section~4.3]{xiao_convergence_2022} assume additional conditions on the iterates' transition matrices. Its proof uses the same two mechanisms: the contraction of the suboptimal action mass, by which \citet{khodadadian_linear_2021} obtain an asymptotic linear rate for NPG with constant step sizes, and the simplex projection argument behind the finite termination of projected $Q$-descent \citep{liu_convergence_2025}. In average reward, the new ingredients are Step~1, which gives $\Delta_t \to 0$ for nondecreasing step sizes, and Lemma~\ref{lem:ap-optimal-advantage}, which defines the optimal advantage through the Poisson equation.

\subsection{Tightness of the non-asymptotic rate} \label{ap:rate-lower-bound}
Under its growth condition on the step sizes, Lemma~\ref{lem:master}(i) gives the rate $1 - 1/\bar\chi$ whenever $\sup_t \chi_t \le \bar\chi$. The following proposition matches this rate up to a constant factor over any fixed number of iterations, for any divergence and any positive step sizes.

\begin{proposition}[Lower bound for exact PMD] \label{prop:ap-rate-lower-bound}
For every $\bar\chi > 1$ and every integer $T \ge 1$, there are an MDP with costs in $[0,1]$ satisfying Assumption~\ref{as:ergodic} and a policy $\pi^0$ with $\pi^0_s \in \operatorname{rint}(\mathcal P(\mathcal A))$ for all $s$, such that $\Delta_0 > 0$, $\chi_0 \ge \bar\chi/2$, $d^\star(s) \le \bar\chi\, d^\pi(s)$ for every policy $\pi$ and every state $s$ (so $\chi_t \le \bar\chi$ for all $t$), and the iterates of~\eqref{eq:pmd-update} with $\tilde Q^t = Q^t$, any Bregman divergence $D$ and any step sizes $\eta_t > 0$ satisfy
\[
\Delta_t \;\ge\; \frac12\Bigl(1 - \frac{1}{\bar\chi}\Bigr)^t\,\Delta_0 \qquad \text{for all } 0 \le t \le T.
\]
\end{proposition}

\begin{proof}
The instance adds a uniform reset to the hard chain of \citet[Appendix~E]{johnson_optimal_2023}. The reset makes every policy ergodic and represents the average cost as a scaled discounted cost with discount factor $\gamma := 1 - 1/\bar\chi$, so exact PMD runs as on their discounted chain, where the optimal action can become greedy in at most one new state per iteration.

\paragraph{Step 1: The instance.} Let $\gamma := 1 - 1/\bar\chi \in (0, 1)$, $n := T + \lceil \log 4/\log(1/\gamma)\rceil$ and $\delta := \min\{(1-\gamma)\gamma^n/2,\ \gamma^{n+1}/(8n)\}$. The states are $s_0, s_1, \dots, s_n, s_1', \dots, s_n'$, so $|\mathcal S| = 2n+1$, and the actions are $a_1, a_2$. Let $P^{\mathrm{ch}}$ be the following deterministic chain: from $s_i$ with $1 \le i \le n$, action $a_1$ moves to $s_{i-1}$ at cost $1$ and action $a_2$ moves to $s_i'$ at cost $1 - r_i$, where $r_i := \gamma^{i+1} + \delta$; the state $s_0$ is absorbing at cost $0$, and $s_i'$ is absorbing at cost $1 - r_i$. The MDP has the transition kernel $P(\cdot\mid s,a) := \gamma P^{\mathrm{ch}}(\cdot\mid s,a) + (1-\gamma)u$, where $u$ is the uniform distribution on $\mathcal S$. Then $P_\pi \ge (1-\gamma)/|\mathcal S|$ entrywise for every policy $\pi$, so Assumption~\ref{as:ergodic} holds. The initial policy $\pi^0$ plays $a_1$ with probability $(1-\gamma)\delta$ in every state.

\paragraph{Step 2: Reduction to discounted costs.} For a policy $\pi$, let $W^\pi := (I - \gamma P^{\mathrm{ch}}_\pi)^{-1} c_\pi$ be its discounted cost on the chain and $Q^\pi_\gamma(s,a) := c(s,a) + \gamma P^{\mathrm{ch}}(\cdot\mid s,a)\,W^\pi$ its action values. Since $P_\pi = \gamma P^{\mathrm{ch}}_\pi + (1-\gamma)\mathbf 1 u$, we have $c_\pi + P_\pi W^\pi = W^\pi + (1-\gamma)(uW^\pi)\mathbf 1$, so $W^\pi$ solves the Poisson equation~\eqref{eq:ap-poisson} with
\begin{equation} \label{eq:ap-reset-rho}
\rho(\pi) \;=\; (1-\gamma)\, u W^\pi.
\end{equation}
By Fact~\ref{fact:ap-poisson}, $h^\pi - W^\pi$ is a constant, so $Q^\pi - Q^\pi_\gamma$ is the same constant at every $(s,a)$. Adding a constant to $\tilde Q^t_s$ leaves the argmin set in~\eqref{eq:pmd-update} unchanged, so exact PMD has the same possible iterates as with $Q^t_\gamma := Q^{\pi^t}_\gamma$ in place of $Q^t$, which is the discounted exact PMD of \citet{johnson_optimal_2023}.

Let $\pi^\star$ play $a_1$ at every $s_i$, and let $W^\star := W^{\pi^\star}$. Then $W^\star(s_0) = 0$, $W^\star(s_i) = (1-\gamma^i)/(1-\gamma)$ and $W^\star(s_i') = (1-r_i)/(1-\gamma)$, so $Q^{\pi^\star}_\gamma(s_i,a_1) = W^\star(s_i)$ and
\[
Q^{\pi^\star}_\gamma(s_i,a_2) - W^\star(s_i) \;=\; \gamma^i - \frac{\delta}{1-\gamma} \;=:\; g_i \;>\; 0,
\]
because $\delta < (1-\gamma)\gamma^n$. Hence $W^\star$ satisfies the Bellman optimality equation of the chain, so $W^\star \le W^\pi$ for every policy $\pi$, and $\pi^\star$ is optimal by~\eqref{eq:ap-reset-rho}. The optimal advantage of Lemma~\ref{lem:ap-optimal-advantage} is $g^\star(s,a) = Q^{\pi^\star}_\gamma(s,a) - W^\star(s)$, which equals $g_i$ at $(s_i, a_2)$ and vanishes elsewhere. Since $W^\pi = W^\star$ at the states that are absorbing under $P^{\mathrm{ch}}$, every policy $\pi$ satisfies, for $1 \le i \le n$,
\begin{equation} \label{eq:ap-reset-gap}
W^\pi(s_i) - W^\star(s_i) \;=\; \pi(a_1\mid s_i)\,\gamma\bigl(W^\pi(s_{i-1}) - W^\star(s_{i-1})\bigr) + \pi(a_2\mid s_i)\, g_i,
\end{equation}
and~\eqref{eq:ap-reset-rho} gives $\Delta_t = \frac{1-\gamma}{|\mathcal S|}\sum_{i=1}^n \bigl(W^t(s_i) - W^\star(s_i)\bigr)$ with $W^t := W^{\pi^t}$. In particular, $\Delta_0 > 0$ by~\eqref{eq:ap-reset-gap}, since $g_i > 0$ and $\pi^0(a_2\mid s_i) > 0$.

\paragraph{Step 3: The optimal action propagates at most one state per iteration.} We show by induction that $\pi^t(a_1\mid s_i) \le (1-\gamma)\delta$ for all $i > t$. This holds at $t = 0$. Suppose it holds at $t$, and let $i > t+1$. By~\eqref{eq:ap-reset-gap} at $s_{i-1}$, the induction hypothesis and $W^t \ge W^\star$, we have $W^t(s_{i-1}) - W^\star(s_{i-1}) \ge (1-(1-\gamma)\delta)\, g_{i-1}$. Since $Q^t_\gamma(s_i,a_1) = 1 + \gamma W^t(s_{i-1})$ and $Q^t_\gamma(s_i,a_2) = W^\star(s_i) + g_i$, and since $\gamma g_{i-1} - g_i = \delta$,
\[
Q^t_\gamma(s_i,a_1) - Q^t_\gamma(s_i,a_2) \;=\; \gamma\bigl(W^t(s_{i-1}) - W^\star(s_{i-1})\bigr) - g_i \;\ge\; \delta - (1-\gamma)\delta\,\gamma\, g_{i-1} \;>\; 0.
\]
The last inequality holds because $(1-\gamma)\gamma\, g_{i-1} < 1$. With two actions, the sign property~\eqref{eq:pmd-sign} at $s_i$ reads $\bigl(\pi^{t+1}(a_1\mid s_i) - \pi^t(a_1\mid s_i)\bigr)\bigl(Q^t_\gamma(s_i,a_1) - Q^t_\gamma(s_i,a_2)\bigr) \le 0$, so $\pi^{t+1}(a_1\mid s_i) \le \pi^t(a_1\mid s_i) \le (1-\gamma)\delta$.

\paragraph{Step 4: Value gaps.} For $i > t$, \eqref{eq:ap-reset-gap}, $W^t \ge W^\star$ and Step~3 give $W^t(s_i) - W^\star(s_i) \ge (1-(1-\gamma)\delta)\, g_i \ge \gamma^i - 2\delta/(1-\gamma)$. At $t = 0$, \eqref{eq:ap-reset-gap}, $g_i \le \gamma^i$ and $W^0(s_0) = W^\star(s_0)$ give $W^0(s_i) - W^\star(s_i) \le \gamma^i$ by induction on $i$.

\paragraph{Step 5: The lower bound.} Let $0 \le t \le T$. By Step~2, dropping the states $s_i$ with $i \le t$, and by Step~4,
\[
\Delta_t \;\ge\; \frac{1-\gamma}{|\mathcal S|}\sum_{i=t+1}^{n}\Bigl(\gamma^i - \frac{2\delta}{1-\gamma}\Bigr) \;\ge\; \frac{\gamma^{t+1} - \gamma^{n+1} - 2n\delta}{|\mathcal S|},
\qquad
\Delta_0 \;\le\; \frac{1-\gamma}{|\mathcal S|}\sum_{i=1}^{n}\gamma^i \;\le\; \frac{\gamma}{|\mathcal S|}.
\]
By the choice of $n$, $\gamma^{n+1} \le \gamma^{t+1}/4$, and $2n\delta \le \gamma^{n+1}/4$, so $\Delta_t \ge \gamma^{t+1}/(2|\mathcal S|) \ge \frac12\gamma^t\Delta_0$.

\paragraph{Step 6: Mismatch.} For every policy $\pi$, stationarity and the reset give $d^\pi(s) \ge (1-\gamma)/|\mathcal S|$ for all $s$, and $d^\pi(s) \ge \gamma d^\pi(s) + (1-\gamma)/|\mathcal S|$, hence $d^\pi(s) \ge 1/|\mathcal S|$, at $s_0$ and $s_i'$, which are absorbing under $P^{\mathrm{ch}}$. Under $\pi^\star$, the states $s_i'$ receive mass from other states only through the reset, so $d^\star(s_i') = 1/|\mathcal S|$; the chain states satisfy $d^\star(s_i) = \gamma d^\star(s_{i+1}) + (1-\gamma)/|\mathcal S|$ with $d^\star(s_{n+1}) := 0$, so $d^\star(s_i) = (1-\gamma^{n-i+1})/|\mathcal S|$; and the remaining mass is $d^\star(s_0) = (1-\gamma^{n+1})/((1-\gamma)|\mathcal S|)$. Hence $d^\star(s) \le d^\pi(s)/(1-\gamma) = \bar\chi\, d^\pi(s)$ for every $s$. Under $\pi^0$, stationarity at $s_0$ gives $(1-\gamma)\,d^0(s_0) = \gamma(1-\gamma)\delta\, d^0(s_1) + (1-\gamma)/|\mathcal S| \le \gamma(1-\gamma)\delta + (1-\gamma)/|\mathcal S|$, that is, $d^0(s_0) \le \delta + 1/|\mathcal S|$. As $\gamma^{n+1} \le 1/4$ and $|\mathcal S|\,\delta \le 3\gamma^{n+1}/8 \le 3/32$,
\[
\chi_0 \;\ge\; \frac{d^\star(s_0)}{d^0(s_0)} \;\ge\; \frac{1-\gamma^{n+1}}{(1-\gamma)\,(1+|\mathcal S|\,\delta)} \;\ge\; \frac{3/4}{1+3/32}\,\bar\chi \;\ge\; \frac{\bar\chi}{2}. \qedhere
\]
\end{proof}

Proposition~\ref{prop:ap-rate-lower-bound} adapts \citet[Theorem~4.2]{johnson_optimal_2023} to average reward, using their discounted chain and the propagation argument of their Appendix~E in our cost convention. The reset makes the instance ergodic and gives~\eqref{eq:ap-reset-rho}; Step~5 aggregates the statewise gaps, and Step~6 calibrates $\gamma = 1 - 1/\bar\chi$ through the mismatch bound. Since the instance can be built for every $T$, no bound $\Delta_t \le C r^t \Delta_0$ with $r < 1 - 1/\bar\chi$ and one constant $C$ for all MDPs with $\sup_t \chi_t \le \bar\chi$ holds for all $t \ge 0$, whatever the divergence and the step sizes. The comparison with the prefactor of Lemma~\ref{lem:master}(i), the initial Lyapunov value $\Psi_0 = \Delta_0 + D_0^\star/(\eta_{-1}\bar\chi)$, is direct: if $D_0^\star < \infty$, the step sizes $\eta_t := \eta_{-1}\gamma^{-(t+1)}$ with $\eta_{-1} \ge D_0^\star/(\bar\chi\Delta_0)$ satisfy its growth condition and give $\Psi_0 \le 2\Delta_0$, so $\frac14\gamma^t\Psi_0 \le \Delta_t \le \gamma^t\Psi_0$ for $0 \le t \le T$. As in their result, the instance must depend on $T$: by Proposition~\ref{prop:ap-exact-geometries}(ii), exact PMD with the squared Euclidean distance and nondecreasing step sizes reaches an optimal policy in finitely many iterations.

\section{Proofs in the tabular setting} \label{ap:proofs-tabular}

\subsection{Proof of Theorem~\ref{thm:avg-pmd-linear-inexact-limsup} (Linear convergence of tabular inexact PMD)}
\begin{proof}
Here, $\tilde Q^t = \widehat Q^t$. Under Assumption~\ref{as:uniform-critic-error}, $\|\widehat Q^t - Q^t\|_\infty \leq \varepsilon_Q$; combined with $\|\pi_s - \pi'_s\|_1 \leq 2$ for any two policies, $|\langle \widehat Q^t_s - Q^t_s, \pi_s - \pi'_s\rangle| \leq 2\varepsilon_Q$. Taking expectations preserves this bound, so $A_t, B_t \leq 2\varepsilon_Q$ for all $t$.

Applying Lemma~\ref{lem:master}(i) with $\bar A = \bar B = 2\varepsilon_Q$ yields directly the non-asymptotic bound with floor $\bar\chi \bar A + \bar B = 2(\bar\chi + 1)\varepsilon_Q$; part~(ii) yields, for every $\zeta > 0$, $\limsup_t (\Delta_t - ((\chi+\zeta)\bar A + \bar B))_+^{1/t} \le 1 - 1/(\chi+\zeta)$ with floor $(\chi+\zeta)\bar A + \bar B = 2(\chi+\zeta+1)\varepsilon_Q$.
\end{proof}

\subsection{Proof of Lemma~\ref{lem:Q_estimation_burnin} (Tabular critic upper bound)}

\begin{lemma}[Tabular critic upper bound, detailed form]
\label{lem:ap-tabular-critic}
Fix a policy $\pi$, a target critic accuracy $\varepsilon_Q \in (0, 1)$, and a confidence $\delta \in (0, 1)$, and set
\begin{align*}
H &:= \Bigl\lceil 4\,t_{\mathrm{mix}} \log\frac{72\,t_{\mathrm{mix}}}{\varepsilon_Q} \Bigr\rceil, &
\tau &:= \lceil 2H/3 \rceil, \\
B &:= \Bigl\lceil \frac{3600\,H\,t_{\mathrm{mix}}}{\varepsilon_Q^2} \log\frac{8\,|\mathcal S|}{\delta} \Bigr\rceil, &
n &:= \Bigl\lceil \frac{72\,t_{\mathrm{mix}}^2}{\varepsilon_Q^2} \log\frac{4\,|\mathcal S||\mathcal A|}{\delta} \Bigr\rceil.
\end{align*}
Let $\widehat Q^\pi$ be computed as in Section~\ref{sec:sample-complexity-tabular} with these parameters: for each state $s$, $\hat h^\pi(s)$ and $\hat\rho_s(\pi)$ are given by~\eqref{eq:q-estimation} and~\eqref{eq:rho-estimation} from $B$ independent rollouts that start at $s$, and for each pair $(s, a)$,
\begin{gather*}
\widehat Q^\pi(s, a) \;:=\; c(s, a) - \hat\rho_s(\pi) + \frac{1}{n}\sum_{j=1}^n \mathrm{clip}\bigl(\hat h^\pi(s'_j)\bigr),
\\
\mathrm{clip}(x) \;:=\; \max\bigl\{-2t_{\mathrm{mix}},\, \min\{2t_{\mathrm{mix}},\, x\}\bigr\},
\end{gather*}
where $s'_1, \dots, s'_n$ are drawn independently from $P(\cdot \mid s, a)$, independently of the rollouts. Then $\|\widehat Q^\pi - Q^\pi\|_\infty \leq \varepsilon_Q$ with probability at least $1 - \delta$, and the critic uses
\[
N \;=\; |\mathcal S|\,B\,H + |\mathcal S||\mathcal A|\,n
\]
generative-model transitions, counting $H$ transitions per rollout.
\end{lemma}

The proof bounds the error of each term of the identity $Q^\pi(s, a) = c(s, a) - \rho(\pi) + P(\cdot \mid s, a)\, h^\pi$ (Fact~\ref{fact:ap-poisson}): the rollouts make $\hat h^\pi$ and every $\hat\rho_s(\pi)$ accurate to $\varepsilon_Q/3$, and the one-step averages add at most $\varepsilon_Q/3$. The rollout bound is the following lemma, which also supplies the labels of Section~\ref{sec:sample-complexity-non-tabular}:

\begin{lemma}[Truncated centered-return estimator]
\label{lem:ap-Q_estimation_burnin}
Fix a policy $\pi$, a state--action pair $(s, a) \in \mathcal S \times \mathcal A$, a target critic accuracy $\varepsilon_Q \in (0, 1)$, and a confidence $\delta \in (0, 1)$. For $H \geq 4$, set $\tau := \lceil 2H/3 \rceil$ and $H_\tau := H - \tau$. Generate $B$ independent length-$H$ trajectories $\{x_k^{(i)}\}_{k=0}^{H-1}$ by setting $x_0^{(i)} = (s, a)$ and following $\pi$, and define
\[
\hat\rho_{\tau, H} \;:=\; \frac{1}{B\,H_\tau}\sum_{i=1}^B \sum_{k=\tau}^{H-1} c(x_k^{(i)}),
\qquad
\widehat Q_{H, \tau}(s, a) \;:=\; \frac{1}{B}\sum_{i=1}^B \sum_{k=0}^{H-1}\bigl(c(x_k^{(i)}) - \hat\rho_{\tau, H}\bigr).
\]
There exist universal constants $C_1, C_2, C_3 > 0$ such that whenever
\[
H \;\geq\; C_2\, t_{\mathrm{mix}}\, \log(C_3\, t_{\mathrm{mix}}/\varepsilon_Q)
\qquad\text{and}\qquad
B \;\geq\; C_1\,\frac{H\, t_{\mathrm{mix}}}{\varepsilon_Q^2}\,\log\frac{2}{\delta},
\]
we have $|\widehat Q_{H, \tau}(s, a) - Q^\pi(s, a)| \leq \varepsilon_Q$ with probability at least $1 - \delta$. The same holds for $\hat h^\pi(s)$ and $h^\pi(s)$ in place of $\widehat Q_{H, \tau}(s, a)$ and $Q^\pi(s, a)$ when each trajectory instead starts at $s$ with an independent first action $a_0^{(i)} \sim \pi_s$, so that the same two formulas give~\eqref{eq:rho-estimation} and~\eqref{eq:q-estimation}.
\end{lemma}

The proof of Lemma~\ref{lem:ap-Q_estimation_burnin} rests on a sharp variance bound for partial sums of a uniformly mixing Markov chain applied to the independent trajectories, followed by a Bernstein concentration argument. The key input is the following covariance summability bound:

\begin{lemma}[Covariance summability under uniform mixing]
\label{lem:ap-covariance-summability}
Fix a policy $\pi$ and a function $g : \mathcal S \times \mathcal A \to \mathbb{R}$ with $\mathbb{E}_{\nu^\pi}[g] = 0$, where $\nu^\pi(s, a) := d^\pi(s) \pi(a \mid s)$. Let $M := \|g\|_\infty$ and $\mathrm{osc}(g) := \sup g - \inf g$. Then for any initial joint distribution $\xi_0 \in \mathcal{P}(\mathcal S \times \mathcal A)$ and every $H \geq 1$,
\[
\mathbb{E}_{\xi_0}\!\left[\left(\sum_{k=0}^{H-1} g(x_k)\right)^{\!2}\right] \;\leq\; M^2\, H \;+\; 4\, M\,\mathrm{osc}(g)\, H\, t_{\mathrm{mix}}.
\]
\end{lemma}

\begin{proof}
Write $S_H := \sum_{k=0}^{H-1} g(x_k)$. Expanding the square,
\[
\mathbb{E}_{\xi_0}[S_H^2] \;=\; \sum_{k=0}^{H-1} \mathbb{E}_{\xi_0}\!\left[g(x_k)^2\right] + 2 \sum_{0 \leq i < j \leq H-1} \mathbb{E}_{\xi_0}[g(x_i)\, g(x_j)].
\]
Since $|g| \leq M$, the diagonal satisfies $\sum_{k=0}^{H-1} \mathbb{E}_{\xi_0}[g(x_k)^2] \leq M^2 H$.

For the cross terms, fix $i < j$ and let $\ell := j - i \geq 1$. We bound $\mathbb{E}[g(x_\ell) \mid x_0 = (s, a)]$ uniformly over $(s, a)$. Conditional on $x_0 = (s, a)$, the law of $s_\ell$ is $P(\cdot \mid s, a)\, P_\pi^{\ell-1}$ (the first transition uses the fixed action $a$, after which $\pi$ acts). By the submultiplicative decay~\eqref{eq:tmix-decay},
\[
\bigl\|P(\cdot \mid s, a)\, P_\pi^{\ell-1} - d^\pi\bigr\|_{\mathrm{TV}} \;\leq\; 2^{-\lfloor (\ell-1)/t_{\mathrm{mix}} \rfloor}.
\]
Since $a_\ell \mid s_\ell \sim \pi(\cdot \mid s_\ell)$, the law of $x_\ell = (s_\ell, a_\ell)$ satisfies the same total variation (TV) bound relative to $\nu^\pi$. Using $\mathbb{E}_{\nu^\pi}[g] = 0$ and the standard total-variation duality $|\mathbb{E}_{\nu'} f - \mathbb{E}_{\nu^\pi} f| \leq (\sup f - \inf f)\, \|\nu' - \nu^\pi\|_{\mathrm{TV}}$ with $\sup g - \inf g = \mathrm{osc}(g)$,
\[
\bigl|\mathbb{E}[g(x_\ell) \mid x_0 = (s, a)]\bigr| \;\leq\; \mathrm{osc}(g)\, 2^{-\lfloor (\ell-1)/t_{\mathrm{mix}} \rfloor} \qquad \forall\, (s, a), \; \ell \geq 1.
\]
By the Markov property (tower rule over $x_i$), for $i < j$ and $\ell = j - i$,
\[
\bigl|\mathbb{E}_{\xi_0}[g(x_i)\, g(x_j)]\bigr| \;=\; \bigl|\mathbb{E}_{\xi_0}\!\bigl[g(x_i)\, \mathbb{E}[g(x_j) \mid x_i]\bigr]\bigr| \;\leq\; M\,\mathrm{osc}(g)\, 2^{-\lfloor (j-i-1)/t_{\mathrm{mix}} \rfloor}.
\]
Summing over cross terms and using the block decomposition $\sum_{\ell=0}^{\infty} 2^{-\lfloor \ell/t_{\mathrm{mix}} \rfloor} = t_{\mathrm{mix}} \sum_{q \geq 0} 2^{-q} = 2\, t_{\mathrm{mix}}$,
\[
\left|2 \sum_{0 \leq i < j \leq H-1} \mathbb{E}_{\xi_0}[g(x_i)\, g(x_j)]\right|
\;\leq\; 2 M\,\mathrm{osc}(g) \sum_{\ell = 1}^{H-1} (H - \ell)\, 2^{-\lfloor (\ell-1)/t_{\mathrm{mix}} \rfloor}
\;\leq\; 4 M\,\mathrm{osc}(g)\, H\, t_{\mathrm{mix}}.
\]
Combining,
\[
\mathbb{E}_{\xi_0}[S_H^2] \;\leq\; M^2\, H \;+\; 4\, M\,\mathrm{osc}(g)\, H\, t_{\mathrm{mix}}. \qedhere
\]
\end{proof}

\begin{proof}[Proof of Lemma~\ref{lem:ap-Q_estimation_burnin}]
Throughout, write $\rho := \rho(\pi)$, $Q := Q^\pi$, and $g(x) := c(x) - \rho$. Since $c(x) \in [0,1]$ and $\rho \in [0,1]$, we have $M = \|g\|_\infty \leq 1$ and $\mathrm{osc}(g) = \mathrm{osc}(c) \leq 1$; in particular $\mathbb{E}_{\nu^\pi}[g] = 0$, so Lemma~\ref{lem:ap-covariance-summability} applies and its conclusion $M^2 J + 4M\,\mathrm{osc}(g)\, J\, t_{\mathrm{mix}} \leq J\,(1 + 4\, t_{\mathrm{mix}})$ yields, for any initial distribution $\xi_0$ and any $J \geq 1$,
\begin{equation}
\label{eq:ap-var-bound}
\mathbb{E}_{\xi_0}\!\left[\left(\sum_{k=0}^{J-1} g(x_k)\right)^{\!2}\right] \;\leq\; J\,(1 + 4\, t_{\mathrm{mix}}).
\end{equation}

\paragraph{Trajectory-level decomposition.}
Define, per trajectory $i$,
\[
U_i \;:=\; \sum_{k=0}^{\tau-1} g\bigl(x_k^{(i)}\bigr), \qquad T_i \;:=\; \sum_{k=\tau}^{H-1} g\bigl(x_k^{(i)}\bigr),
\]
and $Y_i := U_i - (\tau/H_\tau)\,T_i$. A short calculation gives
\[
\widehat Q_{H,\tau}(s,a) \;=\; \frac{1}{B}\sum_{i=1}^B\sum_{k=0}^{H-1}\bigl(c(x_k^{(i)}) - \rho\bigr) \;-\; H\bigl(\hat\rho_{\tau, H} - \rho\bigr) \;=\; \frac{1}{B}\sum_{i=1}^B Y_i,
\]
since $H\,(\hat\rho_{\tau, H} - \rho) = (H/H_\tau)\sum_{i}T_i/B$ and $A_i := U_i + T_i = \sum_{k=0}^{H-1} g(x_k^{(i)})$ implies $A_i - (H/H_\tau)\,T_i = U_i - (\tau/H_\tau)\,T_i = Y_i$. The $B$ trajectories are independent and identically distributed, so $\{Y_i\}$ are i.i.d.

\paragraph{Variance of $Y_i$.}
By \eqref{eq:ap-var-bound} with $\xi_0 = \delta_{(s,a)}$ and horizon $\tau$,
\[
\mathbb{E}[U_i^2] \;\leq\; \tau\,(1 + 4\, t_{\mathrm{mix}}).
\]
Conditional on $x_\tau^{(i)} = (s', a')$, the chain $\{x_{\tau + j}^{(i)}\}_{j \geq 0}$ is a state--action chain under $\pi$ with initial distribution $\delta_{(s', a')}$. By Lemma~\ref{lem:ap-covariance-summability} (which handles any initial joint distribution) with horizon $H_\tau$,
\[
\mathbb{E}[T_i^2 \mid x_\tau^{(i)}] \;\leq\; H_\tau\,(1 + 4\, t_{\mathrm{mix}}) \quad\text{a.s.,}\qquad \text{hence} \qquad \mathbb{E}[T_i^2] \;\leq\; H_\tau\,(1 + 4\, t_{\mathrm{mix}}).
\]
By $(X - Y)^2 \leq 2 X^2 + 2 Y^2$,
\[
\mathrm{Var}(Y_i) \;\leq\; \mathbb{E}[Y_i^2] \;\leq\; 2\,\mathbb{E}[U_i^2] + 2\,\frac{\tau^2}{H_\tau^2}\,\mathbb{E}[T_i^2] \;\leq\; 2\,(1 + 4\, t_{\mathrm{mix}})\,\left(\tau + \frac{\tau^2}{H_\tau}\right).
\]
For $\tau = \lceil 2H/3 \rceil$ and $H \geq 4$, a direct check gives $\tau + \tau^2/H_\tau \leq 4H$ (e.g., $H = 4$: $\tau = 3, H_\tau = 1, \tau + \tau^2/H_\tau = 12 = 3H$; $H = 5$: $\tau = 4, H_\tau = 1, \tau + \tau^2/H_\tau = 20 = 4H$; asymptotically $\tau \approx 2H/3, \tau^2/H_\tau \approx 4H/3$, total $\approx 2H$). Hence, using $1 + 4\,t_{\mathrm{mix}} \leq 5\,t_{\mathrm{mix}}$ for $t_{\mathrm{mix}} \geq 1$,
\[
\mathrm{Var}(Y_i) \;\leq\; 8\,H\,(1 + 4\, t_{\mathrm{mix}}) \;\leq\; 40\,H\, t_{\mathrm{mix}}.
\]

\paragraph{Range of $Y_i$.}
Since $|g| \leq 1$: $|U_i| \leq \tau$, $|T_i| \leq H_\tau$, and $|(\tau/H_\tau)\,T_i| \leq \tau$, so $|Y_i| \leq 2\tau \leq 2H$. Hence the centered range is $|Y_i - \mathbb{E}[Y_i]| \leq 4H$.

\paragraph{Bernstein concentration.}
By i.i.d.\ Bernstein's inequality, with $\sigma^2 := \mathrm{Var}(Y_i) \leq 40\,H\, t_{\mathrm{mix}}$ and centered range $4H$, for any $u > 0$,
\[
\Pr\!\left(\left|\frac{1}{B}\sum_{i=1}^B \bigl(Y_i - \mathbb{E}[Y_i]\bigr)\right| \geq u\right) \;\leq\; 2\exp\!\left(-\frac{B\,u^2}{2\,\sigma^2 + \tfrac{8}{3}\,H\,u}\right).
\]
Setting $u = \varepsilon_Q / 2$, the right-hand side is $\leq \delta$ when
\[
B \;\geq\; \frac{4\,\log(2/\delta)}{\varepsilon_Q^2}\,\left(80\,H\, t_{\mathrm{mix}} + \tfrac{4\,H\,\varepsilon_Q}{3}\right).
\]
For $\varepsilon_Q \leq 1$ and $t_{\mathrm{mix}} \geq 1$, the variance term dominates ($80\,H\,t_{\mathrm{mix}} + \tfrac{4}{3}H\,\varepsilon_Q \leq \tfrac{244}{3}\,H\,t_{\mathrm{mix}}$), so an absolute constant $C_1 > 0$ (e.g., $C_1 = 400$) suffices:
\begin{equation}
\label{eq:ap-bernstein-bound}
B \;\geq\; C_1\,\frac{H\, t_{\mathrm{mix}}}{\varepsilon_Q^2}\,\log\!\frac{2}{\delta}
\qquad \Longrightarrow \qquad
\Pr\!\left(\bigl|\widehat Q_{H,\tau}(s,a) - \mathbb{E}[\widehat Q_{H,\tau}(s,a)]\bigr| \leq \varepsilon_Q/2\right) \geq 1 - \delta.
\end{equation}

\paragraph{Bias bound.}
By Lemma~\ref{lem:uniform-mix}, starting from $x_0 = (s, a)$, the state marginal at time $k \geq 1$ is $P(\cdot\mid s,a)\,P_\pi^{k-1}$, and since $a_k \mid s_k \sim \pi$, the total-variation duality together with~\eqref{eq:tmix-decay} gives
\begin{equation}
\label{eq:ap-reward-dev}
\bigl|\mathbb{E}[c(x_k) - \rho \mid x_0 = (s,a)]\bigr| \;\leq\; 2^{-\lfloor (k-1)/t_{\mathrm{mix}}\rfloor}, \qquad k \geq 1,
\end{equation}
using $\mathrm{osc}(c) \leq 1$. In particular, the series $Q(s,a) = \sum_{k=0}^\infty \mathbb{E}[g(x_k) \mid x_0 = (s,a)]$ converges absolutely.

Define $q_H := \mathbb{E}[A_1 \mid x_0 = (s,a)] = \sum_{k=0}^{H-1} \mathbb{E}[g(x_k) \mid x_0]$. By~\eqref{eq:ap-reward-dev} and the block-decomposition bound $\sum_{\ell = \ell_0}^\infty 2^{-\lfloor \ell/t_{\mathrm{mix}}\rfloor} \leq 2\, t_{\mathrm{mix}} \cdot 2^{-\lfloor \ell_0/t_{\mathrm{mix}}\rfloor}$,
\[
|q_H - Q(s,a)| \;=\; \left|\sum_{k=H}^\infty \mathbb{E}[g(x_k) \mid x_0]\right| \;\leq\; \sum_{k=H}^\infty 2^{-\lfloor (k-1)/t_{\mathrm{mix}}\rfloor} \;\leq\; 2\, t_{\mathrm{mix}}\, 2^{-\lfloor (H-1)/t_{\mathrm{mix}}\rfloor}.
\]
Similarly, $|\mathbb{E}[T_1]| \leq 2\, t_{\mathrm{mix}}\, 2^{-\lfloor (\tau-1)/t_{\mathrm{mix}}\rfloor}$. Using $\mathbb{E}[\widehat Q_{H,\tau}(s,a)] = \mathbb{E}[Y_1] = \mathbb{E}[A_1] - (H/H_\tau)\,\mathbb{E}[T_1]$,
\[
\bigl|\mathbb{E}[\widehat Q_{H,\tau}(s,a)] - Q(s,a)\bigr| \;\leq\; |q_H - Q(s,a)| + \frac{H}{H_\tau}\,|\mathbb{E}[T_1]|.
\]
For any $\tau$ with $H_\tau = H - \tau \geq H/5$ (e.g., $\tau = \lceil 2H/3 \rceil$ or $\tau = \lceil H/2 \rceil$) and $H \geq 4$: $H/H_\tau \leq 5$. Using $2^{-\lfloor (H-1)/t_{\mathrm{mix}}\rfloor} \leq 2^{-\lfloor (\tau-1)/t_{\mathrm{mix}}\rfloor}$ (since $H \geq \tau$),
\begin{equation}
\label{eq:ap-bias-bound}
\bigl|\mathbb{E}[\widehat Q_{H,\tau}(s,a)] - Q(s,a)\bigr| \;\leq\; 12\, t_{\mathrm{mix}}\, 2^{-\lfloor (\tau-1)/t_{\mathrm{mix}}\rfloor}.
\end{equation}

\paragraph{Combining.}
Choose $H$ so that $12\, t_{\mathrm{mix}}\, 2^{-\lfloor (\tau-1)/t_{\mathrm{mix}}\rfloor} \leq \varepsilon_Q/2$, i.e., $\lfloor (\tau-1)/t_{\mathrm{mix}}\rfloor \geq \log_2(24\, t_{\mathrm{mix}}/\varepsilon_Q)$. With $\tau \geq 2H/3$ and $H \geq 4$, this is implied by
\[
H \;\geq\; C_2\, t_{\mathrm{mix}}\, \log(C_3\, t_{\mathrm{mix}}/\varepsilon_Q)
\]
for suitable universal constants (e.g., $C_2 = 4, C_3 = 24$, which absorbs the integer-floor correction $\lfloor(\tau-1)/t_{\mathrm{mix}}\rfloor \geq (\tau-1)/t_{\mathrm{mix}} - 1$). On the high-probability event of~\eqref{eq:ap-bernstein-bound},
\begin{align*}
\bigl|\widehat Q_{H,\tau}(s,a) - Q(s,a)\bigr|
&\;\leq\; \bigl|\widehat Q_{H,\tau}(s,a) - \mathbb{E}[\widehat Q_{H,\tau}(s,a)]\bigr| + \bigl|\mathbb{E}[\widehat Q_{H,\tau}(s,a)] - Q(s,a)\bigr| \\
&\;\leq\; \varepsilon_Q/2 + \varepsilon_Q/2 \;=\; \varepsilon_Q,
\end{align*}
which holds with probability at least $1 - \delta$.

\paragraph{Initial action drawn from $\pi_s$.}
If each trajectory instead starts at $s$ with an independent first action $a_0^{(i)} \sim \pi_s$, the trajectories remain independent and identically distributed, and Lemma~\ref{lem:ap-covariance-summability} holds for every initial joint distribution, so the variance, range, and Bernstein steps apply verbatim with the initial distribution $\delta_s \otimes \pi_s$ in place of $\delta_{(s,a)}$. The mean of the estimate is the average over $a \sim \pi_s$ of $\mathbb{E}[\widehat Q_{H,\tau}(s, a)]$, and $h^\pi(s) = \sum_a \pi(a \mid s)\, Q^\pi(s, a)$ by definition, so averaging~\eqref{eq:ap-bias-bound} over $a$ bounds the bias by the same quantity. The same $H$ and $B$ therefore give $|\hat h^\pi(s) - h^\pi(s)| \leq \varepsilon_Q$ with probability at least $1 - \delta$.
\end{proof}

\begin{proof}[Proof of Lemma~\ref{lem:ap-tabular-critic}]
Write $\rho := \rho(\pi)$ and $h := h^\pi$. We allot failure probability $\delta/4$ to the state values, $\delta/4$ to the average costs, and $\delta/2$ to the one-step averages.

\paragraph{Horizon.}
With $x := 72\,t_{\mathrm{mix}}/\varepsilon_Q \geq 72$, the choices $\tau \geq 2H/3$ and $H \geq 4\,t_{\mathrm{mix}}\log x$ give
\[
\Bigl\lfloor \frac{\tau - 1}{t_{\mathrm{mix}}} \Bigr\rfloor \;\geq\; \frac{2H}{3\,t_{\mathrm{mix}}} - 2 \;\geq\; \frac{8}{3}\log x - 2 \;\geq\; \log_2 x,
\]
where the last step holds because $(8/3 - 1/\log 2)\log 72 > 2$. Hence
\begin{equation} \label{eq:ap-horizon-check}
12\, t_{\mathrm{mix}}\, 2^{-\lfloor (\tau-1)/t_{\mathrm{mix}}\rfloor} \;\leq\; \frac{12\,t_{\mathrm{mix}}}{x} \;=\; \frac{\varepsilon_Q}{6},
\end{equation}
which is the horizon condition in the proof of Lemma~\ref{lem:ap-Q_estimation_burnin} at accuracy $\varepsilon_Q/3$. Also $H \geq 4$.

\paragraph{State values.}
Fix a state $s$. Apply Lemma~\ref{lem:ap-Q_estimation_burnin} to the rollouts that start at $s$, at accuracy $\varepsilon_Q/3$ and confidence $\delta/(4|\mathcal S|)$: its horizon condition is~\eqref{eq:ap-horizon-check}, and its Bernstein condition~\eqref{eq:ap-bernstein-bound} with $C_1 = 400$ reads $B \geq 3600\,H\,t_{\mathrm{mix}}\,\varepsilon_Q^{-2}\log(8|\mathcal S|/\delta)$. Hence $|\hat h^\pi(s) - h(s)| \leq \varepsilon_Q/3$ except on an event of probability at most $\delta/(4|\mathcal S|)$. Clipping cannot increase this distance, because $|h(s)| \leq 2\,t_{\mathrm{mix}}$. Indeed, by Fact~\ref{fact:ap-poisson}, $h(s) = \sum_{k \geq 0} (\delta_s P_\pi^k - d^\pi)\, c_\pi$ with $c_\pi(s') := \sum_{a} \pi(a \mid s')\, c(s', a) \in [0, 1]$, and the total-variation duality in the proof of Lemma~\ref{lem:ap-covariance-summability} together with~\eqref{eq:tmix-decay} bounds the $k$-th term by $2^{-\lfloor k/t_{\mathrm{mix}}\rfloor}$, whose sum over $k \geq 0$ is $2\,t_{\mathrm{mix}}$.

\paragraph{Average costs.}
Fix a state $s$. The estimate $\hat\rho_s(\pi)$ averages $B$ independent tail averages with values in $[0, 1]$. As for $h(s)$, $|\mathbb{E}[c(s_k^{(i)}, a_k^{(i)})] - \rho| \leq 2^{-\lfloor k/t_{\mathrm{mix}}\rfloor}$, so by the block-decomposition bound in the proof of Lemma~\ref{lem:ap-Q_estimation_burnin} and~\eqref{eq:ap-horizon-check}, each tail average has mean within $2\,t_{\mathrm{mix}}\,2^{-\lfloor \tau/t_{\mathrm{mix}}\rfloor} \leq \varepsilon_Q/36$ of $\rho$. Hoeffding's inequality and $B \geq 18\,\varepsilon_Q^{-2}\log(8|\mathcal S|/\delta)$ then give $|\hat\rho_s(\pi) - \rho| \leq \varepsilon_Q/6 + \varepsilon_Q/36 \leq \varepsilon_Q/3$ except on an event of probability at most $2\exp(-B\varepsilon_Q^2/18) \leq \delta/(4|\mathcal S|)$.

\paragraph{One-step averages.}
Condition on the rollouts. For each pair $(s, a)$, the values $\mathrm{clip}(\hat h^\pi(s'_j))$, $j = 1, \dots, n$, are independent, lie in $[-2t_{\mathrm{mix}}, 2t_{\mathrm{mix}}]$, and have mean $P(\cdot \mid s, a)\,\mathrm{clip}(\hat h^\pi)$. Hoeffding's inequality bounds the conditional probability that their average deviates from this mean by more than $\varepsilon_Q/3$ by $2\exp(-n\varepsilon_Q^2/(72\,t_{\mathrm{mix}}^2)) \leq \delta/(2|\mathcal S||\mathcal A|)$, by the choice of $n$. Since this bound holds for every outcome of the rollouts, a union bound over the pairs costs at most $\delta/2$.

\paragraph{Combining.}
Outside the exceptional events, whose probabilities sum to at most $|\mathcal S| \cdot \delta/(4|\mathcal S|) + |\mathcal S| \cdot \delta/(4|\mathcal S|) + \delta/2 = \delta$, the identity of Fact~\ref{fact:ap-poisson} gives, for every pair $(s, a)$,
\begin{align*}
\bigl|\widehat Q^\pi(s, a) - Q^\pi(s, a)\bigr|
&\;\leq\; |\hat\rho_s(\pi) - \rho| + \Bigl|\frac{1}{n}\sum_{j=1}^n \mathrm{clip}\bigl(\hat h^\pi(s'_j)\bigr) - P(\cdot \mid s, a)\,\mathrm{clip}(\hat h^\pi)\Bigr| \\
&\qquad + \bigl|P(\cdot \mid s, a)\bigl(\mathrm{clip}(\hat h^\pi) - h\bigr)\bigr| \;\leq\; \frac{\varepsilon_Q}{3} + \frac{\varepsilon_Q}{3} + \frac{\varepsilon_Q}{3} \;=\; \varepsilon_Q,
\end{align*}
where the last term uses $\|\mathrm{clip}(\hat h^\pi) - h\|_\infty \leq \varepsilon_Q/3$ and that $P(\cdot \mid s, a)$ is a probability distribution. The critic draws $B$ rollouts of length $H$ at each of the $|\mathcal S|$ states and $n$ next states at each of the $|\mathcal S||\mathcal A|$ pairs, so $N = |\mathcal S|BH + |\mathcal S||\mathcal A|n$. Since $H = O(t_{\mathrm{mix}}\log(t_{\mathrm{mix}}/\varepsilon_Q))$, $B = O(H\,t_{\mathrm{mix}}\,\varepsilon_Q^{-2}\log(|\mathcal S|/\delta))$, and $n = O(t_{\mathrm{mix}}^2\,\varepsilon_Q^{-2}\log(|\mathcal S||\mathcal A|/\delta))$, this is $N = \widetilde O((|\mathcal S|\,t_{\mathrm{mix}}^3 + |\mathcal S||\mathcal A|\,t_{\mathrm{mix}}^2)/\varepsilon_Q^2)$, as stated in Lemma~\ref{lem:Q_estimation_burnin}.
\end{proof}

\subsection{Proof of Theorem~\ref{thm:pmd_sample_complexity} (Tabular sample complexity)}

\begin{theorem}[Tabular sample complexity, detailed form]
\label{thm:ap-pmd_sample_complexity}
Suppose $D_0^\star < \infty$, and recall the initial Lyapunov value $\Psi_0 = \Delta_0 + D_0^\star/(\eta_{-1}\bar\chi)$. Fix a target suboptimality $\varepsilon \in (0, 1)$ and a confidence $\delta \in (0, 1)$, and set $\varepsilon_Q := \varepsilon/(4(\bar\chi + 1))$ and $T := \lceil \bar\chi\,\log\max\{e, 2\,\Psi_0/\varepsilon\} \rceil$. Run inexact PMD (Theorem~\ref{thm:avg-pmd-linear-inexact-limsup}) for $T$ iterations under the hypotheses of Lemma~\ref{lem:master}(i), computing each $\widehat Q^t$ by the critic of Lemma~\ref{lem:ap-tabular-critic} for $\pi^t$ at accuracy $\varepsilon_Q$ and confidence $\delta/T$, from fresh generative-model transitions. Then $\Delta_T \leq \varepsilon$ with probability at least $1 - \delta$, using
\[
N \;=\; T\,\bigl(|\mathcal S|\,B\,H + |\mathcal S||\mathcal A|\,n\bigr)
\]
generative-model transitions, where $H$, $B$, and $n$ are the parameters of Lemma~\ref{lem:ap-tabular-critic} at accuracy $\varepsilon_Q$ and confidence $\delta/T$.
\end{theorem}

\begin{proof}
The proof has two ingredients: a union bound converting the per-call guarantee of Lemma~\ref{lem:ap-tabular-critic} into critic accuracy along the PMD trajectory, and a deterministic composition with Theorem~\ref{thm:avg-pmd-linear-inexact-limsup}.

\paragraph{Critic accuracy along the trajectory.}
Let $\mathcal F_{t-1}$ denote the $\sigma$-algebra generated by the transitions used in iterations $0, \dots, t-1$; the policy $\pi^t$ is $\mathcal F_{t-1}$-measurable, and the critic call uses fresh generative-model randomness, independent of $\mathcal F_{t-1}$. Applying Lemma~\ref{lem:ap-tabular-critic} conditionally on $\mathcal F_{t-1}$ gives
\[
\Pr\!\left(\bigl\|\widehat Q^t - Q^{\pi^t}\bigr\|_\infty > \varepsilon_Q \;\middle|\; \mathcal F_{t-1}\right) \;\leq\; \frac{\delta}{T},
\]
and a union bound over $t \in \{0, \dots, T-1\}$ bounds the probability that some iteration fails by $\delta$. On the complementary good event, the bound of Assumption~\ref{as:uniform-critic-error} holds at level $\varepsilon_Q$ for every $t < T$, which is all that the bound of Theorem~\ref{thm:avg-pmd-linear-inexact-limsup} at time $T$ uses.

\paragraph{Composition with Theorem~\ref{thm:avg-pmd-linear-inexact-limsup}.}
Condition on the good event. Since $\sup_t \chi_t \leq \bar\chi$, Theorem~\ref{thm:avg-pmd-linear-inexact-limsup} gives
\[
\bigl(\Delta_T - 2(\bar\chi + 1)\varepsilon_Q\bigr)_+ \;\leq\; \Bigl(1 - \tfrac{1}{\bar\chi}\Bigr)^T\, \Psi_0.
\]
The choice $\varepsilon_Q = \varepsilon/(4(\bar\chi + 1))$ makes the error floor equal to $\varepsilon/2$. Using $(1 - 1/\bar\chi)^T \leq \exp(-T/\bar\chi)$, the choice $T = \lceil \bar\chi\,\log\max\{e, 2\,\Psi_0/\varepsilon\} \rceil$ ensures $(1 - 1/\bar\chi)^T\, \Psi_0 \leq \varepsilon/2$. Therefore
\[
\Delta_T \;\leq\; \tfrac{\varepsilon}{2} + \tfrac{\varepsilon}{2} \;=\; \varepsilon,
\]
with probability at least $1 - \delta$ over the sampling randomness.
\end{proof}

The total number of generative-model transitions is $N = T\,(|\mathcal S|BH + |\mathcal S||\mathcal A|n)$. At accuracy $\varepsilon_Q = \varepsilon/(4(\bar\chi + 1))$ and confidence $\delta/T$, the parameters of Lemma~\ref{lem:ap-tabular-critic} satisfy
\[
B\,H \;=\; \widetilde O\!\left(\frac{H^2\,t_{\mathrm{mix}}\,(\bar\chi + 1)^2}{\varepsilon^2}\right),
\qquad
n \;=\; \widetilde O\!\left(\frac{t_{\mathrm{mix}}^2\,(\bar\chi + 1)^2}{\varepsilon^2}\right),
\qquad
T = \widetilde\Theta(\bar\chi),
\]
with $H = \Theta(t_{\mathrm{mix}}\,\log(t_{\mathrm{mix}}/\varepsilon_Q)) = \widetilde O(t_{\mathrm{mix}})$, so $H^2\,t_{\mathrm{mix}} = \widetilde O(t_{\mathrm{mix}}^3)$. Combining, and using $\bar\chi \geq 1$ (so $(\bar\chi + 1)^2 \leq 4\bar\chi^2$),
\[
N \;=\; \widetilde O\!\left(\frac{(|\mathcal S|\,t_{\mathrm{mix}}^3 + |\mathcal S||\mathcal A|\,t_{\mathrm{mix}}^2)\,\bar\chi^3}{\varepsilon^2}\right),
\]
recovering Theorem~\ref{thm:pmd_sample_complexity}.

\subsection{Proof of Proposition~\ref{prop:tabular_lower_bound} (Tabular critic lower bound)}

We first restate the proposition with the explicit construction used in the proof.

\begin{proposition}[Tabular critic lower bound, detailed form]
\label{prop:ap-tabular_lower_bound}
For every $|\mathcal S| \geq 2$, $p \in (0, 1/4)$, and $\varepsilon_Q \in (0, 1/(128 p)]$, there are a finite family of MDPs with state space $\mathcal S$ and action space $\mathcal A$ and a deterministic policy $\pi$ such that every member satisfies Assumption~\ref{as:ergodic} and the bound of Lemma~\ref{lem:uniform-mix} with $t_{\mathrm{mix}} = \lceil \log 4/p \rceil < 2/p$, and the $1/4$-mixing time of $\pi$ in every member is at least $1/(16p)$. Any (possibly adaptive) generative-model algorithm that outputs $\widehat Q \in \mathbb{R}^{\mathcal S \times \mathcal A}$ with $\Pr(\|\widehat Q - Q^\pi\|_\infty \leq \varepsilon_Q) \geq 2/3$ on every member, where $Q^\pi$ is the differential action-value function of $\pi$ in that member, must use
\[
N \;\geq\; \frac{d}{2^{12}\cdot 9\,p^3\,\varepsilon_Q^2} + \frac{d\,(|\mathcal A| - 1)}{2^{8}\cdot 9\,p^2\,\varepsilon_Q^2}
\;\geq\; \frac{|\mathcal S|\,\bigl(t_{\mathrm{mix}}^3 + |\mathcal A|\,t_{\mathrm{mix}}^2\bigr)}{2^{13}\cdot 135\,\varepsilon_Q^2}
\]
generative-model transitions, where $d := \lfloor |\mathcal S|/2 \rfloor$.
\end{proposition}
Here, an algorithm is adaptive if it chooses each query based on the samples seen so far.

\begin{proof}
The family hides one bit in the transitions of the action that $\pi$ takes at each pair of states and one bit in the transitions of every other action at each paired state. A critic that is accurate to $\varepsilon_Q$ in sup norm reveals every bit, the bits change disjoint groups of transition rows, and a Le Cam two-point test for each bit then forces enough queries to each group. This is Assouad's hypercube reduction in coordinatewise form. With one action and $2d$ states, the construction is that of Proposition~\ref{prop:ap-nontabular_lower_bound}.

\paragraph{Construction.}
Write $\mathcal A = \{a_1, \dots, a_{|\mathcal A|}\}$, and identify $\mathcal S$ with $[d] \times \{-1, +1\}$ if $|\mathcal S| = 2d$, and with $([d] \times \{-1, +1\}) \cup \{o\}$, for one more state $o$, if $|\mathcal S| = 2d + 1$. Let $u$ be the uniform distribution on $\mathcal S$, set $c((j, \sigma), a) := (1 + \sigma)/2$ and $c(o, a) := 1/2$ for every action $a$, and let $\pi$ always take $a_1$. The members are indexed by state bits $\theta \in \{0, 1\}^d$ and action bits $\vartheta_{j,\sigma,a} \in \{-1, +1\}$ for $(j, \sigma) \in [d] \times \{-1, +1\}$ and $a \neq a_1$, of which there are none if $|\mathcal A| = 1$. With
\[
\Delta \;:=\; 64\,p^2\,\varepsilon_Q, \qquad \Delta' \;:=\; 8\,p\,\varepsilon_Q, \qquad q_j \;:=\; p + \theta_j\,\Delta,
\]
the local kernel $R_{\theta,\vartheta}(\cdot \mid (j, \sigma), a)$ moves $(j, \sigma)$ to $(j, -\sigma)$ with probability $q_j$ if $a = a_1$ and with probability $1/2 + \vartheta_{j,\sigma,a}\Delta'$ if $a \neq a_1$, and otherwise leaves it in place. Set $R_{\theta,\vartheta}(\cdot \mid o, a) := \delta_o$. The transition kernel is
\[
P_{\theta,\vartheta}(\cdot \mid s, a) \;:=\; (1 - p)\, R_{\theta,\vartheta}(\cdot \mid s, a) + p\, u(\cdot).
\]
The restriction $\varepsilon_Q \leq 1/(128 p)$ gives $\Delta \leq p/2$ and $\Delta' \leq 1/16$.

\paragraph{Mixing for every policy.}
Every row satisfies $P_{\theta,\vartheta}(\cdot \mid s, a) \geq p\, u(\cdot)$, which is the minorization condition~\eqref{eq:doeblin-controlled} with $\alpha = p$ and $\nu = u$. Every policy therefore induces an irreducible and aperiodic chain, so Assumption~\ref{as:ergodic} holds, and Appendix~\ref{ap:mixing-assumption-details} gives $\sup_{\mu_0}\|\mu_0 P_{\pi'}^k - d^{\pi'}\|_{\mathrm{TV}} \leq (1 - p)^k \leq e^{-pk}$ for every policy $\pi'$. At $k = \lceil \log 4/p \rceil$ this is at most $1/4$, which gives the bound of Lemma~\ref{lem:uniform-mix} with $t_{\mathrm{mix}} = \lceil \log 4/p \rceil \leq \log 4/p + 1 < 2/p$, where the last step uses $1 < 1/(4p)$.

\paragraph{The evaluated policy.}
Under $\pi$, each pair $\{(j, +1), (j, -1)\}$ moves as in the construction of Proposition~\ref{prop:ap-nontabular_lower_bound}, and $o$ moves only by the reset. Steps~1--4 of its proof therefore apply: the uniform distribution is stationary, so $d^\pi = u$ and $\rho(\pi) = 1/2$; the eigenfunction argument shows that the $1/4$-mixing time of $\pi$ is at least $1/(16p)$; the centered Poisson solution is $h^\pi(j, \sigma) = \beta_j\,\sigma$ with $\beta_j := \beta(q_j) \in [1/(8p), 1/(5p)]$, where $\beta(q) := 1/(2p + 4q(1 - p))$; and $\beta(p) - \beta(p + \Delta) \geq \Delta/(16p^2) = 4\,\varepsilon_Q$. The state $o$ does not change these steps: $u$ remains stationary because the local kernel of $\pi$ is symmetric on each pair and fixes $o$, the eigenfunctions vanish at $o$, and $h^\pi(o) := 0$ solves the Poisson equation at $o$ because $c(o, a_1) = \rho(\pi)$ and $u\, h^\pi = 0$.

\paragraph{Decoding.}
By Fact~\ref{fact:ap-poisson}, $Q^\pi(s, a) = c(s, a) - \rho(\pi) + P_{\theta,\vartheta}(\cdot \mid s, a)\, h^\pi$, in which the reset contributes $p\, u\, h^\pi = 0$. For a row at $(j, \sigma)$ that flips the sign with probability $f$, the local part of $P_{\theta,\vartheta}(\cdot \mid s, a)\, h^\pi$ is $(1 - p)(1 - 2f)\,\beta_j\sigma$, so
\begin{align*}
Q^\pi((j, \sigma), a_1) &= \beta_j\,\sigma, \qquad Q^\pi(o, a) = 0, \\
\sigma\, Q^\pi((j, \sigma), a) &= \frac12 - 2(1 - p)\,\beta_j\,\Delta'\,\vartheta_{j,\sigma,a} \qquad (a \neq a_1),
\end{align*}
where the first identity is the Poisson equation at $(j, \sigma)$. Since $(1 - p)\,\beta_j\,\Delta' \geq \frac34 \cdot \frac{1}{8p} \cdot 8p\,\varepsilon_Q$, an action bit moves $\sigma Q^\pi((j, \sigma), a)$ away from $1/2$ by at least $\frac32\varepsilon_Q$, and a state bit moves $Q^\pi((j, +1), a_1)$ between $\beta(p)$ and $\beta(p + \Delta)$, which differ by at least $4\varepsilon_Q$. Hence, whenever $\|\widehat Q - Q^\pi\|_\infty \leq \varepsilon_Q$, the decoders
\[
\widehat\theta_j \;:=\; \mathbf{1}\Bigl\{\widehat Q((j, +1), a_1) < \frac{\beta(p) + \beta(p + \Delta)}{2}\Bigr\},
\qquad
\hat\vartheta_{j,\sigma,a} \;:=\; \operatorname{sign}\Bigl(\frac12 - \sigma\, \widehat Q((j, \sigma), a)\Bigr)
\]
return every bit of the member.

\paragraph{Per-query KL divergence.}
Changing $\theta_j$ changes only the rows $((j, \pm 1), a_1)$, and changing $\vartheta_{j,\sigma,a}$ only the row $((j, \sigma), a)$, so different bits change disjoint groups of rows. For a row of $\theta_j$, Step~5 in the proof of Proposition~\ref{prop:ap-nontabular_lower_bound} bounds the KL divergence between the two members' next-state laws by $2\Delta^2/p$ in either direction, as in~\eqref{eq:nlb-kl-bound}; its lower bounds on the two changing atoms ignore the reset mass, so they hold here too. For a row of $\vartheta_{j,\sigma,a}$, revealing whether the reset occurred, which has the same law under both members, can only increase the KL divergence, which is then at most
\[
\mathrm{KL}\bigl(\mathrm{Bern}(\tfrac12 + \Delta') \,\big\|\, \mathrm{Bern}(\tfrac12 - \Delta')\bigr)
\;=\; 2\Delta' \log\frac{1 + 2\Delta'}{1 - 2\Delta'}
\;\leq\; \frac{8\,\Delta'^2}{1 - 2\Delta'}
\;\leq\; 16\,\Delta'^2,
\]
using $\log((1 + x)/(1 - x)) \leq 2x/(1 - x)$ for $x \in [0, 1)$.

\paragraph{Assembly.}
Fix a member $(\theta, \vartheta)$ and one of its bits, and let $\mathbb{P}$ and $\mathbb{P}'$ be the laws of the transcript under this member and under the member in which only that bit is flipped. As in Step~6 of the proof of Proposition~\ref{prop:ap-nontabular_lower_bound}, the query rule has the same conditional law under both members, so the KL chain rule gives $\mathrm{KL}(\mathbb{P} \,\|\, \mathbb{P}') \leq (2\Delta^2/p)\,\mathbb{E}[N_{\mathrm{bit}}]$ for a state bit and $\mathrm{KL}(\mathbb{P} \,\|\, \mathbb{P}') \leq 16\Delta'^2\,\mathbb{E}[N_{\mathrm{bit}}]$ for an action bit, where $N_{\mathrm{bit}}$ is the number of queries to the rows that the bit changes and the expectation is under $\mathbb{P}$. The decoder of this bit returns the member's bit with probability at least $2/3$ under $\mathbb{P}$ and at most $1/3$ under $\mathbb{P}'$, so $\|\mathbb{P} - \mathbb{P}'\|_{\mathrm{TV}} \geq 1/3$, and Pinsker's inequality $\|\mathbb{P} - \mathbb{P}'\|_{\mathrm{TV}} \leq \sqrt{\mathrm{KL}(\mathbb{P} \,\|\, \mathbb{P}')/2}$ gives $\mathrm{KL}(\mathbb{P} \,\|\, \mathbb{P}') \geq 2/9$. Hence $\mathbb{E}[N_{\mathrm{bit}}] \geq p/(9\Delta^2)$ for a state bit and $\mathbb{E}[N_{\mathrm{bit}}] \geq 1/(72\Delta'^2)$ for an action bit. The groups are disjoint, so summing over the $d$ state bits and the $2d(|\mathcal A| - 1)$ action bits of the fixed member gives
\[
N \;\geq\; d\,\frac{p}{9\,\Delta^2} + 2d\,(|\mathcal A| - 1)\,\frac{1}{72\,\Delta'^2}
\;=\; \frac{d}{2^{12}\cdot 9\,p^3\,\varepsilon_Q^2} + \frac{d\,(|\mathcal A| - 1)}{2^{8}\cdot 9\,p^2\,\varepsilon_Q^2}.
\]

\paragraph{Final form.}
We have $d \geq |\mathcal S|/3$ for $|\mathcal S| \geq 2$, and the right-hand side equals $d\,(1 + 16(|\mathcal A| - 1)p)/(2^{12}\cdot 9\,p^3\varepsilon_Q^2)$ with $1 + 16(|\mathcal A| - 1)p \geq \frac45(1 + |\mathcal A|\,p)$ for $p \leq 1/4$, so $N \geq |\mathcal S|\,(p^{-3} + |\mathcal A|\,p^{-2})/(2^{10}\cdot 135\,\varepsilon_Q^2)$. Since $t_{\mathrm{mix}} < 2/p$, we have $p^{-3} > t_{\mathrm{mix}}^3/8$ and $p^{-2} > t_{\mathrm{mix}}^2/4$, which gives the second inequality of the proposition.
\end{proof}

The assumption $|\mathcal S| \geq 2$ is necessary: with a single state, $Q^\pi(s, a) = c(s, a) - \sum_{a'} \pi(a' \mid s)\, c(s, a')$ is determined by the costs, without any transitions.

\section{Proofs under linear function approximation} \label{ap:proofs-non-tabular}
\subsection{Proof of Theorem~\ref{thm:non-tabular-linear-convergence} (Linear convergence of inexact PMD under LFA)}
\begin{proof}
PMD is the update~\eqref{eq:pmd-update} with the LFA critic $\tilde Q^t(s,a) = \phi_{s,a}^\top w_t$. We apply Lemma~\ref{lem:master} after bounding the critic-error quantities $A_t, B_t$; all bounds below hold on the probability-one event of Assumption~\ref{as:L2-mismatch-bound}.

\paragraph{Bounding $A_t$ and $B_t$ via state--action concentrability.}
Define the scalar error function $e_t(s,a) := w_t^\top \phi_{s,a} - Q^{\pi^t}(s,a)$. Expanding the inner products against $\pi^{t+1} - \pi^t$ and $\pi^t - \pi^\star$ into expectations over state--action pairs,
\[
A_t = \Bigl|\mathbb{E}_{\nu^{t+1}}[e_t] - \mathbb{E}_{d^{t+1}\times\pi^t}[e_t]\Bigr|, \qquad B_t = \Bigl|\mathbb{E}_{d^\star\times\pi^t}[e_t] - \mathbb{E}_{\nu^\star}[e_t]\Bigr|,
\]
where $\nu^t := d^t \times \pi^t$ denotes the stationary state--action distribution of $\pi^t$.

We use the generic change-of-measure bound: for any two distributions $\nu,\nu'$ on $\mathcal S\times\mathcal A$
with $\nu' \ll \nu$ and any measurable $f$,
\begin{equation}\label{eq:com_bound}
\big|\mathbb{E}_{(s,a)\sim \nu'}[f(s,a)]\big|
\;\le\;
\sqrt{\mathbb{E}_{(s,a)\sim \nu}\!\left[\left(\frac{\nu'(s,a)}{\nu(s,a)}\right)^2\right]}\;
\sqrt{\mathbb{E}_{(s,a)\sim \nu}[f(s,a)^2]}.
\end{equation}

First, apply \eqref{eq:com_bound} with $\nu=\nu^t$ and $\nu'=\nu^{t+1}$, and use
Assumption~\ref{as:L2-mismatch-bound} to get
\[
\big|\mathbb{E}_{\nu^{t+1}}[e_t]\big|
\;\le\;
\sqrt{C_+}\,\sqrt{\mathbb{E}_{\nu^t}[e_t^2]}.
\]
Next, note that the mixed distribution $d^{t+1}\times \pi^t(\cdot\mid s)$ satisfies
\[
\frac{d^{t+1}(s)\pi^t(a\mid s)}{\nu^t(s,a)}
=
\frac{d^{t+1}(s)}{d^t(s)}
=
\sum_{a'} \pi^t(a'\mid s)\,\frac{\nu^{t+1}(s,a')}{\nu^t(s,a')},
\]
so by Jensen's inequality,
\[
\left(\frac{d^{t+1}(s)}{d^t(s)}\right)^2
\le
\sum_{a'} \pi^t(a'\mid s)\left(\frac{\nu^{t+1}(s,a')}{\nu^t(s,a')}\right)^2.
\]
Taking expectation over $s\sim d^t$ yields
\[
\mathbb{E}_{(s,a)\sim \nu^t}\!\left[\left(\frac{d^{t+1}(s)\pi^t(a\mid s)}{\nu^t(s,a)}\right)^2\right]
=
\mathbb{E}_{s\sim d^t}\!\left[\left(\frac{d^{t+1}(s)}{d^t(s)}\right)^2\right]
\le
\mathbb{E}_{(s,a)\sim \nu^t}\!\left[\left(\frac{\nu^{t+1}(s,a)}{\nu^t(s,a)}\right)^2\right]
\le C_+.
\]
Applying \eqref{eq:com_bound} with $\nu=\nu^t$ and $\nu'=d^{t+1}\times \pi^t(\cdot\mid s)$ gives
\[
\big|\mathbb{E}_{d^{t+1}\times \pi^t}[e_t]\big|
\;\le\;
\sqrt{C_+}\,\sqrt{\mathbb{E}_{\nu^t}[e_t^2]}.
\]
Therefore,
\begin{equation}\label{eq:E1_bound_sa}
A_t \le 2\sqrt{C_+}\,\sqrt{\mathbb{E}_{\nu^t}[e_t^2]}.
\end{equation}
An analogous argument using Assumption~\ref{as:L2-mismatch-bound} (with $\nu^\star$ in place of $\nu^{t+1}$) yields
\begin{equation}\label{eq:E2_bound_sa}
B_t \le 2\sqrt{C_\star}\,\sqrt{\mathbb{E}_{\nu^t}[e_t^2]}.
\end{equation}

Finally, decompose
\[
e_t(s,a)=\phi_{s,a}^\top(w_t-w_t^\star) + \big(\phi_{s,a}^\top w_t^\star - Q^{\pi^t}(s,a)\big).
\]
By the triangle inequality in $L_2(\nu^t)$,
\begin{align}
\sqrt{\mathbb{E}_{\nu^t}[e_t^2]}
\;&\le\;
\sqrt{\mathbb{E}_{\nu^t}\!\left[\big(\phi_{s,a}^\top(w_t-w_t^\star)\big)^2\right]}
+
\sqrt{\mathbb{E}_{\nu^t}\!\left[\big(\phi_{s,a}^\top w_t^\star - Q^{\pi^t}(s,a)\big)^2\right]} \\
&\leq
\|w_t-w_t^\star\|_{\Sigma_t}
+
\varepsilon_{\mathrm{app}}
\;\le\;
\varepsilon_{\mathrm{stat}}
+
\varepsilon_{\mathrm{app}},
\label{eq:e2_bound_sa}
\end{align}
where the last two inequalities follow from Assumptions~\ref{as:approx-error} and~\ref{as:stat-error}, respectively.
\paragraph{Applying the master lemma.}
Combining \eqref{eq:E1_bound_sa}, \eqref{eq:E2_bound_sa}, and \eqref{eq:e2_bound_sa} yields, for all $t \ge 0$,
\[
A_t \leq \bar A := 2\sqrt{C_+}\bigl(\varepsilon_{\mathrm{stat}} + \varepsilon_{\mathrm{app}}\bigr), \qquad B_t \leq \bar B := 2\sqrt{C_\star}\bigl(\varepsilon_{\mathrm{stat}} + \varepsilon_{\mathrm{app}}\bigr).
\]
Lemma~\ref{lem:master}(i) then gives the claimed non-asymptotic bound with the error floor $\bar\chi\bar A + \bar B = \varepsilon_{\mathrm{floor}}(\bar\chi)$, and Lemma~\ref{lem:master}(ii) gives the claimed asymptotic bound with the floor $(\chi+\zeta)\bar A + \bar B = \varepsilon_{\mathrm{floor}}(\chi+\zeta)$.
\end{proof}

\begin{algorithm}[tb]
\caption{Regression for $w_t^\star$}
\label{alg:end2end_w_oracle}
\begin{algorithmic}[1]
\REQUIRE Policy $\pi^t$; features $\phi_{s,a}\in\mathbb{R}^d$; current state of the trajectory;
burn-in length $m$; rollout horizon $H$; tail index $\tau\in\{0,\dots,H-1\}$;
number of rollouts $B=1$; number of SGD steps $K$; step size $\alpha$.
\ENSURE Parameter $\widehat w\in\mathbb{R}^d$.

\STATE Initialize $w_{t,0}\gets 0$
\FOR{$k=1,2,\dots,K$}
    \STATE Set $s_{k,0}$ to the current state of the trajectory
    \FOR{$j=0,1,\dots,m-1$}
        \STATE Sample $a_{k,j}\sim \pi^t(\cdot\mid s_{k,j})$
        \STATE Sample $s_{k,j+1}\sim P(\cdot\mid s_{k,j},a_{k,j})$
    \ENDFOR
    \STATE Sample $a_{k,m}\sim \pi^t(\cdot\mid s_{k,m})$
    \STATE Set $(s_k,a_k)\gets (s_{k,m},a_{k,m})$
    \STATE Continue the trajectory for a rollout of $H$ transitions from $(s_k,a_k)$ and compute $\widehat Q_k \gets \widehat Q^{\pi^t}_{H,\tau}(s_k,a_k)$ from it by the formula of Lemma~\ref{lem:ap-Q_estimation_burnin} with $B=1$
    \STATE $g_k \gets (w_{t,k-1}^\top \phi_{s_k,a_k}-\widehat Q_k)\,\phi_{s_k,a_k}$
    \STATE $w_{t,k} \gets w_{t,k-1}-\alpha g_k$
\ENDFOR
\STATE \textbf{Return} $\widehat w \gets \frac{1}{\lceil K/2\rceil}\sum_{k=\lfloor K/2\rfloor+1}^{K} w_{t,k}$
\end{algorithmic}
\end{algorithm}

\begin{algorithm}[tb]
\caption{Confidence boosting for $w_t$}
\label{alg:boost}
\begin{algorithmic}[1]
\REQUIRE Policy $\pi^t$; current state of the trajectory; number of runs $R$; inputs $m$, $H$, $\tau$, $B=1$, $K$, $\alpha$ of Algorithm~\ref{alg:end2end_w_oracle}.
\ENSURE Parameter $w_t\in\mathbb{R}^d$.

\FOR{$i=1,2,\dots,R$}
    \STATE Run Algorithm~\ref{alg:end2end_w_oracle}, continuing the trajectory, and let $w^{(i)}$ be its output
\ENDFOR
\STATE Let $\widehat\Sigma_t$ be the average of $\phi_{s,a}\phi_{s,a}^\top$ over the $RK$ pairs $(s,a)$ of the $R$ runs
\FOR{$i=1,2,\dots,R$}
    \STATE $r_i \gets \min\bigl\{r\ge0 : |\{j : \|w^{(j)}-w^{(i)}\|_{\widehat\Sigma_t}\le r\}| > R/2\bigr\}$
\ENDFOR
\STATE \textbf{Return} $w_t \gets w^{(i^\star)}$ for the smallest $i^\star\in\argmin_i r_i$
\end{algorithmic}
\end{algorithm}

\subsection{Proof of a helper lemma (tail-averaged SGD)}
\begin{lemma}[Estimating $w_t^\star$ with an unbiased label oracle]
\label{cor:avg_reward_qnpg_sample_complexity}
Fix $t$ and consider the regression target
\[
w_t^\star \in \argmin_{w\in\mathbb{R}^d}\;
\tfrac12\,\mathbb{E}_{(s,a)\sim \nu^t}\bigl[(w^\top\phi_{s,a} - Q^t(s,a))^2\bigr]
\]
with $|Q^t(s,a)| \leq G$ and approximation error $\mathbb{E}_{(s,a)\sim\nu^t}[(w_t^{\star\top}\phi_{s,a}-Q^t(s,a))^2]\le\varepsilon_{\mathrm{app}}^2$. Additionally suppose that:
\begin{enumerate}
\item[(i)] (Bounded features) $\|\phi_{s,a}\|_2\le L$ for all $(s,a)$.
\item[(ii)] (Uniform covariance lower bound) For all $t$,
$\Sigma_t := \mathbb{E}_{(s,a)\sim \nu^{t}}\bigl[\phi_{s,a}\phi_{s,a}^\top\bigr]
\succeq \mu I_d$ for some $\mu>0$.
\item[(iii)] (Unbiased label oracle) The label oracle returns $\widehat Q^t(s,a)$ with
\[
\mathbb{E}\!\left[\widehat Q^{t}(s,a)\mid s,a\right]=Q^{t}(s,a),
\qquad
\mathbb{E}\!\left[(\widehat Q^{t}(s,a))^2\mid s,a\right]\le G^2.
\]
\end{enumerate}

At each outer iteration $t$, compute $w_t$ by running $K$ steps of SGD on the quadratic loss
\[
f_t(w)
:=
\frac12\,
\mathbb{E}_{(s,a)\sim \nu^{t}}
\bigl[(w^\top \phi_{s,a}-Q^t(s,a))^2\bigr],
\]
with independent and identically distributed labeled pairs $(s,a,\widehat Q^t(s,a))$, where $(s,a)\sim\nu^t$ and the label satisfies (iii), constant step size $\alpha=\frac{1}{4L^2}$, initialization $w_{t,0}=0$, and tail-averaged output
$w_t:=\frac{1}{\lceil K/2\rceil}\sum_{k=\lfloor K/2\rfloor+1}^{K} w_{t,k}$.

Then, the SGD output satisfies the statistical-error bound
\[
\mathbb{E}\bigl[\|w_t-w_t^\star\|_{\Sigma_t}^2\bigr]
\;\le\;
\frac{256\,L^4G^2}{\mu^2K^2}\,\exp\!\Bigl(-\frac{\mu\,\lfloor K/2\rfloor}{4L^2}\Bigr)
+\frac{8}{K}\Bigl(dG^2+\frac{L^2}{\mu}\,\varepsilon_{\mathrm{app}}^2\Bigr).
\]
Consequently, for any desired level $\varepsilon_{\mathrm{stat}}\in(0,G]$, it suffices to choose
\[
K \;\ge\; \frac{16}{\varepsilon_{\mathrm{stat}}^2}\Bigl(dG^2+\frac{L^2}{\mu}\,\varepsilon_{\mathrm{app}}^2\Bigr)+\frac{8L^2}{\mu}\log\frac{16G^2}{\varepsilon_{\mathrm{stat}}^2}
\]
to guarantee $\mathbb{E}\bigl[\|w_t-w_t^\star\|_{\Sigma_t}^2\bigr]\le \varepsilon_{\mathrm{stat}}^2$.
\end{lemma}

\begin{proof}
With a constant step size, averaged SGD \citep{bach_non-strongly-convex_2013} forgets its initialization only at rate $1/K$; averaging only the second half of the iterates \citep{jain_parallelizing_2018} makes this error decay geometrically. We prove the bound directly by splitting the error of SGD into an initialization part and a noise part; bounding the noise separately for the label noise and for the approximation error keeps $L^2/\mu$ off the label-noise term.

Write $w^\star=w_t^\star$. For a labeled pair $(s,a,y)$ with $y=\widehat Q^t(s,a)$, let $x:=\phi_{s,a}$ and let $\xi:=(y-w^{\star\top}x)\,x$ be the residual; $(x_k,y_k,\xi_k)$ denotes the $k$-th pair. Since $\mathbb{E}[y\mid s,a]=Q^t(s,a)$, first-order optimality of $w^\star$ gives $\mathbb{E}[\xi]=\mathbb{E}[(Q^t(s,a)-w^{\star\top}x)\,x]=0$.

\paragraph{Projection bound.}
The minimizer satisfies the normal equation $\Sigma_t w^\star=\mathbb{E}[Q^t(s,a)\,x]$, so by the Cauchy--Schwarz inequality and $|Q^t(s,a)|\le G$,
\[
\|w^\star\|_{\Sigma_t}^2
=
\mathbb{E}\bigl[Q^t(s,a)\,x^\top w^\star\bigr]
\le G\,\|w^\star\|_{\Sigma_t},
\qquad\text{hence}\qquad
\|w^\star\|\le\frac{\|w^\star\|_{\Sigma_t}}{\sqrt{\mu}}\le\frac{G}{\sqrt{\mu}}.
\]

\paragraph{Initialization and noise parts.}
The update $w_{t,k}=w_{t,k-1}-\alpha(w_{t,k-1}^\top x_k-y_k)\,x_k$ gives $w_{t,k}-w^\star=(I_d-\alpha x_kx_k^\top)(w_{t,k-1}-w^\star)+\alpha\xi_k$. Hence $w_{t,k}-w^\star=b_k+v_k$, where
\[
b_k=(I_d-\alpha x_kx_k^\top)\,b_{k-1},\quad b_0=-w^\star,
\qquad\qquad
v_k=(I_d-\alpha x_kx_k^\top)\,v_{k-1}+\alpha\xi_k,\quad v_0=0.
\]
Let $\mathcal T(A):=\mathbb{E}[(I_d-\alpha xx^\top)A(I_d-\alpha xx^\top)]$. This linear map preserves the positive semidefinite order and is self-adjoint for the trace inner product. Since $\|x\|\le L$ and $\alpha L^2\le1$,
\[
\mathcal T(I_d)=I_d-2\alpha\Sigma_t+\alpha^2\,\mathbb{E}[\|x\|^2xx^\top]\preceq I_d-\alpha\Sigma_t\preceq(1-\alpha\mu)I_d,
\]
so $\mathrm{Tr}\,\mathcal T(A)=\mathrm{Tr}(A\,\mathcal T(I_d))\le(1-\alpha\mu)\,\mathrm{Tr}(A)$ for every $A\succeq0$. Since $b_{k-1}$ is independent of the $k$-th pair, $\mathbb{E}[b_kb_k^\top]=\mathcal T(\mathbb{E}[b_{k-1}b_{k-1}^\top])$, and therefore
\[
\mathbb{E}\|b_k\|^2\le(1-\alpha\mu)^k\,\|w^\star\|^2\le(1-\alpha\mu)^k\,G^2/\mu .
\]
For the noise part, $\mathbb{E}[v_k]=0$ by induction, so the cross terms vanish and $C_k:=\mathbb{E}[v_kv_k^\top]=\mathcal T(C_{k-1})+\alpha^2\,\mathbb{E}[\xi\xi^\top]$. Conditioning on $(s,a)$ separates the label noise from the approximation error:
\[
\mathbb{E}[\xi\xi^\top]
=
\underbrace{\mathbb{E}\bigl[\mathrm{Var}(y\mid s,a)\,xx^\top\bigr]}_{=:S_1}
+
\underbrace{\mathbb{E}\bigl[(Q^t(s,a)-w^{\star\top}x)^2\,xx^\top\bigr]}_{=:S_2},
\]
where $S_1\preceq G^2\Sigma_t$ because $\mathrm{Var}(y\mid s,a)\le\mathbb{E}[y^2\mid s,a]\le G^2$, and $\mathrm{Tr}(S_2)\le L^2\varepsilon_{\mathrm{app}}^2$ because $\|x\|\le L$.
By linearity, $C_k=C_k^{(1)}+C_k^{(2)}$ with $C_k^{(i)}=\mathcal T(C_{k-1}^{(i)})+\alpha^2S_i$ and $C_0^{(i)}=0$. Induction gives $C_k^{(1)}\preceq\alpha G^2I_d$, because $\mathcal T(\alpha G^2I_d)+\alpha^2G^2\Sigma_t\preceq\alpha G^2I_d$. The trace contraction gives $\mathrm{Tr}\,C_k^{(2)}\le(1-\alpha\mu)\,\mathrm{Tr}\,C_{k-1}^{(2)}+\alpha^2L^2\varepsilon_{\mathrm{app}}^2$, hence $\mathrm{Tr}\,C_k^{(2)}\le\alpha(L^2/\mu)\,\varepsilon_{\mathrm{app}}^2$. Together,
\[
\mathbb{E}\|v_k\|^2=\mathrm{Tr}(C_k)\le\alpha\Bigl(dG^2+\frac{L^2}{\mu}\,\varepsilon_{\mathrm{app}}^2\Bigr).
\]

\paragraph{Tail averaging.}
Let $u_k$ be either $b_k$ or $v_k$, and let $\mathcal K:=\{\lfloor K/2\rfloor+1,\dots,K\}$, with $q:=|\mathcal K|=\lceil K/2\rceil$. For $j\ge i$, the pairs after the $i$-th are independent of $u_i$ and their residuals have mean zero, so $\mathbb{E}[u_ju_i^\top]=(I_d-\alpha\Sigma_t)^{j-i}\,\mathbb{E}[u_iu_i^\top]$. Expanding the square, the indices $i\le j$ therefore contribute at most $2\,\mathrm{Tr}\bigl(\Sigma_t(I_d-\alpha\Sigma_t)^{j-i}\,\mathbb{E}[u_iu_i^\top]\bigr)$, which is nonnegative since $\Sigma_t(I_d-\alpha\Sigma_t)^{l}\succeq0$ for every $l\ge0$ (as $\alpha\Sigma_t\preceq I_d$). Summing over $j\ge i$ and using $\sum_{l\ge0}\Sigma_t(I_d-\alpha\Sigma_t)^l=I_d/\alpha$,
\[
\mathbb{E}\Bigl\|\frac1q\sum_{k\in\mathcal K}u_k\Bigr\|_{\Sigma_t}^2
\le
\frac{2}{q^2}\sum_{i\in\mathcal K}\sum_{l\ge0}\mathrm{Tr}\bigl(\Sigma_t(I_d-\alpha\Sigma_t)^l\,\mathbb{E}[u_iu_i^\top]\bigr)
=
\frac{2}{\alpha q^2}\sum_{i\in\mathcal K}\mathbb{E}\|u_i\|^2 .
\]

\paragraph{Conclusion.}
Summing the geometric series over the tail gives $\sum_{k\in\mathcal K}\mathbb{E}\|b_k\|^2\le(1-\alpha\mu)^{\lfloor K/2\rfloor}G^2/(\alpha\mu^2)$, and $\sum_{k\in\mathcal K}\mathbb{E}\|v_k\|^2\le q\,\alpha\,(dG^2+(L^2/\mu)\varepsilon_{\mathrm{app}}^2)$. Since $w_t-w^\star$ is the tail average of $b_k+v_k$ and $\|u+v\|_{\Sigma_t}^2\le2\|u\|_{\Sigma_t}^2+2\|v\|_{\Sigma_t}^2$,
\[
\mathbb{E}\|w_t-w^\star\|_{\Sigma_t}^2
\le
\frac{4G^2(1-\alpha\mu)^{\lfloor K/2\rfloor}}{\alpha^2\mu^2q^2}
+\frac{4}{q}\Bigl(dG^2+\frac{L^2}{\mu}\,\varepsilon_{\mathrm{app}}^2\Bigr),
\]
and $\alpha=1/(4L^2)$, $q\ge K/2$ and $1-\alpha\mu\le e^{-\alpha\mu}$ give the stated bound. For the choice of $K$, the first summand of the lower bound on $K$ makes the noise term at most $\varepsilon_{\mathrm{stat}}^2/2$. For the initialization term, let $\ell:=\log(16G^2/\varepsilon_{\mathrm{stat}}^2)\ge1$. The second summand gives $K\ge8(L^2/\mu)\,\ell$, so $256L^4/(\mu^2K^2)\le4$, and $\mu\lfloor K/2\rfloor/(4L^2)\ge\ell-\frac18$ by $\lfloor K/2\rfloor\ge(K-1)/2$ and $\mu\le L^2$. The initialization term is therefore at most $4e^{1/8}G^2e^{-\ell}<\varepsilon_{\mathrm{stat}}^2/2$.
\end{proof}

\subsection{Proof of Lemma~\ref{lem:w-sample-complexity} (LFA critic upper bound)}
\begin{lemma}[LFA critic upper bound, detailed form]
\label{lem:ap-w-sample-complexity}
Suppose Assumption~\ref{as:approx-error} holds, $\|\phi_{s,a}\|_2\le L$ for all $(s,a)$, and $\Sigma_t \succeq \mu I_d$ for all $t$, for some $\mu>0$. Fix an outer iteration $t$, let $\mathcal F_{t-1}$ denote the $\sigma$-algebra generated by the trajectory up to the start of this critic call, and let $w_t$ be the output of Algorithm~\ref{alg:boost}, whose runs of Algorithm~\ref{alg:end2end_w_oracle} use step size $\alpha=\frac{1}{4L^2}$. For any target critic accuracy $\varepsilon_{\mathrm{stat}}\in(0,1)$ and confidence $\delta\in(0,1)$, choose $H$, $\tau$, $B$, $m$ and $K$ as in Lemma~\ref{lem:ap-single-run} at accuracy $\varepsilon_{\mathrm{stat}}/20$, and $R=\lceil 2\log(2/\delta)\rceil$. Then
\[
\Pr\bigl(\|w_t-w_t^\star\|_{\Sigma_t}\le\varepsilon_{\mathrm{stat}} \,\big|\, \mathcal F_{t-1}\bigr) \;\ge\; 1-\delta \quad\text{almost surely.}
\]
Moreover, Algorithm~\ref{alg:boost} uses at most $RK(m+BH)$ transitions, so under the above choices the total number of transitions is
\[
RK(m+BH)
\;=\;
\widetilde{O}\!\left(
\Bigl(\frac{d\,t_{\mathrm{mix}}^3}{\varepsilon_{\mathrm{stat}}^2}+\frac{L^2}{\mu}\,t_{\mathrm{mix}}\Bigl(1+\frac{\varepsilon_{\mathrm{app}}^2}{\varepsilon_{\mathrm{stat}}^2}\Bigr)\Bigr)\log\frac{2}{\delta}
\right),
\]
where $\widetilde{O}(\cdot)$ hides absolute constants and polylogarithmic factors in
$t_{\mathrm{mix}},L,\mu,d,1/\varepsilon_{\mathrm{stat}}$.
\end{lemma}

The proof combines a bound in expectation for a single run of Algorithm~\ref{alg:end2end_w_oracle} with the selection step of Algorithm~\ref{alg:boost}. We first prove the former.

\begin{lemma}[One run of Algorithm~\ref{alg:end2end_w_oracle}]
\label{lem:ap-single-run}
Suppose Assumption~\ref{as:approx-error} holds, $\|\phi_{s,a}\|_2\le L$ for all $(s,a)$, and $\Sigma_t \succeq \mu I_d$ for all $t$, for some $\mu>0$. Fix an outer iteration $t$ and let $\widehat w$ be the output of a run of Algorithm~\ref{alg:end2end_w_oracle} with policy $\pi^t$ and step size
$\alpha=\frac{1}{4L^2}$. Let $\mathcal G$ denote the $\sigma$-algebra generated by the trajectory up to the start of this run; the policy $\pi^t$ and the current state of the trajectory are $\mathcal G$-measurable.

Then for any target critic accuracy $\varepsilon_{\mathrm{stat}}\in(0,1)$, the integer parameter choices, with $H\ge 4$,
\[
H \;=\; \Theta\!\left(t_{\mathrm{mix}}\,\log\!\frac{2t_{\mathrm{mix}}}{\varepsilon_{\mathrm{stat}}}\right),
\qquad
\tau \;=\; \left\lceil\frac{H}{2}\right\rceil,
\qquad
B \;=\; 1,
\]
\[
m \;=\; \Theta\!\left(t_{\mathrm{mix}}\,\log\!\frac{t_{\mathrm{mix}}\,(1+L^2/\mu)}{\varepsilon_{\mathrm{stat}}}\right),
\]
\[
K \;=\; \Theta\!\left(
\frac{d\,t_{\mathrm{mix}}^2\,\log(2t_{\mathrm{mix}}/\varepsilon_{\mathrm{stat}})}{\varepsilon_{\mathrm{stat}}^2}
+\frac{L^2}{\mu}\Bigl(\log\frac{2d\,t_{\mathrm{mix}}}{\varepsilon_{\mathrm{stat}}}+\frac{\min\{\varepsilon_{\mathrm{app}}^2,(1+2t_{\mathrm{mix}})^2\}}{\varepsilon_{\mathrm{stat}}^2}\Bigr)
\right)
\]
suffice to guarantee
\[
\mathbb{E}\!\left[\|\widehat w-w_t^\star\|_{\Sigma_t}^2 \,\middle|\, \mathcal G\right] \;\le\; \varepsilon_{\mathrm{stat}}^2 \quad\text{almost surely.}
\]

Moreover, Algorithm~\ref{alg:end2end_w_oracle} uses at most $K(m+BH)$ transitions, so
under the above choices (with $BH = H = \widetilde O(t_{\mathrm{mix}})$) the total number of transitions is
\[
K(m+BH)
\;=\;
\widetilde{O}\!\left(
\frac{d\,t_{\mathrm{mix}}^3}{\varepsilon_{\mathrm{stat}}^2}+\frac{L^2}{\mu}\,t_{\mathrm{mix}}\Bigl(1+\frac{\varepsilon_{\mathrm{app}}^2}{\varepsilon_{\mathrm{stat}}^2}\Bigr)
\right),
\]
where $\widetilde{O}(\cdot)$ hides absolute constants and polylogarithmic factors in
$t_{\mathrm{mix}},L,\mu,1/\varepsilon_{\mathrm{stat}}$.
\end{lemma}

\begin{proof}
Fix $t$.

\paragraph{Coupling with independent pairs.}
Throughout, we condition on $\mathcal G$, which fixes $\pi^t$ and the current state of the trajectory. Write $Z_k := (s_k, a_k, \widehat Q_k)$ for the $k$-th pair of Algorithm~\ref{alg:end2end_w_oracle} and its label, so that $\widehat w$ is a function of $Z_1,\dots,Z_K$. Let $Z_1',\dots,Z_K'$ be independent, where $Z_k' := (s_k', a_k', \widehat Q_k')$ has $(s_k', a_k') \sim \nu^t$ and a label $\widehat Q_k'$ computed in the same way from a rollout that starts at $(s_k', a_k')$; this is the sampling of Lemma~\ref{cor:avg_reward_qnpg_sample_complexity}. Conditional on the trajectory up to the start of the $k$-th burn-in, the law of $s_k$ is $P_{\pi^t}^m(\cdot\mid s_{k,0})$, which by~\eqref{eq:tmix-decay} is within $\delta_m := 2^{-\lfloor m/t_{\mathrm{mix}}\rfloor}$ of $d^{\pi^t}$ in total variation. Since $a_k \mid s_k \sim \pi^t(\cdot\mid s_k)$ and $\widehat Q_k$ is computed from the rollout that starts at $(s_k, a_k)$, the law of $Z_k$ satisfies the same TV bound relative to the law of $Z_k'$, and, by convexity of the total variation distance, so does its conditional law given only $Z_1,\dots,Z_{k-1}$, which averages these laws. Now replace $Z_K, Z_{K-1},\dots,Z_1$ in turn by independent draws distributed as $Z_K', Z_{K-1}',\dots,Z_1'$. The step that replaces $Z_k$ changes only the conditional law of the $k$-th coordinate given the first $k-1$, from that of $Z_k$ to the law of $Z_k'$, so the joint law moves in total variation by the average distance between these two laws, which is at most $\delta_m$. Hence the joint laws $\mathbb{P}$ of $(Z_1,\dots,Z_K)$ and $\mathbb{P}'$ of $(Z_1',\dots,Z_K')$ satisfy $\|\mathbb{P}-\mathbb{P}'\|_{\mathrm{TV}}\le K\delta_m$. Since the $Z_k$ take finitely many values, an optimal coupling of $\mathbb{P}$ and $\mathbb{P}'$ \citep[Proposition~4.7 and Remark~4.8]{levin_markov_2017} lets us realize $Z_1',\dots,Z_K'$ on an extended probability space so that $Z_k'\neq Z_k$ for some $k$ with probability at most $K\delta_m$. Let $\widehat w'$ be the output of the same SGD run on $Z_1',\dots,Z_K'$; it equals $\widehat w$ whenever $Z_k'=Z_k$ for all $k$.

\paragraph{Almost-sure bound on $\|\widehat w-w_t^\star\|_{\Sigma_t}$.}
With $B=1$, each label equals the quantity $Y_1$ in the proof of Lemma~\ref{lem:ap-Q_estimation_burnin}, so $|\widehat Q_k|\le 2H$ by the range bound there. Each SGD step is $w_{t,k}=(I_d-\alpha\phi\phi^\top)w_{t,k-1}+\alpha\widehat Q_k\phi$ with $\phi=\phi_{s_k,a_k}$, and $\|I_d-\alpha\phi\phi^\top\|_{\mathrm{op}}\le 1$ since $\alpha\|\phi\|_2^2\le 1/4$, so $\|w_{t,k}\|_2\le\|w_{t,k-1}\|_2+2\alpha LH=\|w_{t,k-1}\|_2+H/(2L)$. Hence $\|w_{t,k}\|_2\le kH/(2L)$, and the tail average satisfies $\|\widehat w\|_2\le KH/(2L)$. Moreover, $\|w_t^\star\|_2\le \|Q^{\pi^t}\|_\infty/\sqrt{\mu}\le L\,\|Q^{\pi^t}\|_\infty/\mu$ by the bound on $\|w^\star\|$ in the proof of Lemma~\ref{cor:avg_reward_qnpg_sample_complexity} and $\mu\le L^2$. For the Poisson bound $\|Q^{\pi^t}\|_\infty\le 1+2\,t_{\mathrm{mix}}$, the series for $Q^{\pi^t}$ in the proof of Lemma~\ref{lem:ap-Q_estimation_burnin} has a first term of absolute value at most $1$, and by~\eqref{eq:ap-reward-dev} the absolute values of its remaining terms sum to at most $\sum_{\ell\ge0}2^{-\lfloor \ell/t_{\mathrm{mix}}\rfloor}=2\,t_{\mathrm{mix}}$. Since $\Sigma_t\preceq L^2I_d$, almost surely
\[
\|\widehat w-w_t^\star\|_{\Sigma_t}
\;\le\;
L\bigl(\|\widehat w\|_2+\|w_t^\star\|_2\bigr)
\;\le\;
\frac{KH}{2}+\frac{(1+2\,t_{\mathrm{mix}})L^2}{\mu}.
\]

\paragraph{Decomposition.}
Define the conditional-mean label
\[
\widehat Q(s,a) \;:=\; \mathbb{E}\!\left[\widehat Q_{H,\tau}^{\pi^t}(s,a)\mid s,a\right],
\]
and the population minimizer
\[
\tilde w_t \in \argmin_w \mathbb{E}_{(s,a)\sim\nu^t}\!\left[(w^\top\phi_{s,a}-\widehat Q(s,a))^2\right].
\]
Split the expectation according to whether $Z_k'=Z_k$ for all $k$. If so, $\widehat w=\widehat w'$, and we write $\widehat w'-w_t^\star=(\widehat w'-\tilde w_t)+(\tilde w_t-w_t^\star)$ and use $\|u+v\|_{\Sigma_t}^2\le 2\|u\|_{\Sigma_t}^2+2\|v\|_{\Sigma_t}^2$; if not, which has probability at most $K\delta_m$, we use the almost-sure bound. This gives
\begin{equation}
\label{eq:three_term_decomp}
\mathbb{E}\!\left[\|\widehat w-w_t^\star\|_{\Sigma_t}^2\right]
\le
2\,\mathbb{E}\!\left[\|\widehat w'-\tilde w_t\|_{\Sigma_t}^2\right]
+2\,\|\tilde w_t-w_t^\star\|_{\Sigma_t}^2
+K\delta_m\Bigl(\frac{KH}{2}+\frac{(1+2\,t_{\mathrm{mix}})L^2}{\mu}\Bigr)^2.
\end{equation}

\paragraph{Bias term $\|\tilde w_t-w_t^\star\|_{\Sigma_t}^2$.}
Since $\phi^\top w_t^\star$ and $\phi^\top\tilde w_t$ are the $L^2(\nu^t)$ projections of $Q^{\pi^t}$ and $\widehat Q$ onto the span of the features, and projections do not increase distances,
\[
\|\tilde w_t-w_t^\star\|_{\Sigma_t}
\le
\sqrt{\mathbb{E}_{(s,a)\sim\nu^t}\!\left[(\widehat Q(s,a)-Q^{\pi^t}(s,a))^2\right]}
\le
\|\widehat Q-Q^{\pi^t}\|_\infty.
\]
By the proof of Lemma~\ref{lem:ap-Q_estimation_burnin} (bias bound~\eqref{eq:ap-bias-bound}),
\[
\|\widehat Q-Q^{\pi^t}\|_\infty
\;\le\;
12\,t_{\mathrm{mix}}\,2^{-\lfloor (\tau-1)/t_{\mathrm{mix}}\rfloor}.
\]
With the choice $\tau=\lceil H/2\rceil$ and $H=\Theta(t_{\mathrm{mix}}\,\log(2t_{\mathrm{mix}}/\varepsilon_{\mathrm{stat}}))$,
the right-hand side is $O(\varepsilon_{\mathrm{stat}})$, hence
$\|\widehat Q-Q^{\pi^t}\|_\infty\le\varepsilon_{\mathrm{stat}}/3$ and $\|\tilde w_t-w_t^\star\|_{\Sigma_t}^2 \le \varepsilon_{\mathrm{stat}}^2/9$ for a suitable constant in the
$\Theta(\cdot)$ choice of $H$.

\paragraph{SGD term $\mathbb{E}[\|\widehat w'-\tilde w_t\|_{\Sigma_t}^2]$.}
By definition of $\widehat Q$, the labels of Algorithm~\ref{alg:end2end_w_oracle} ($B = 1$) satisfy $\mathbb{E}[\widehat Q_{H,\tau}^{\pi^t}(s,a)\mid s,a]=\widehat Q(s,a)$. The SGD rate of Lemma~\ref{cor:avg_reward_qnpg_sample_complexity} depends on the label second moment, which we bound directly (rather than via the high-probability guarantee of Lemma~\ref{lem:ap-Q_estimation_burnin}):
\begin{align*}
\mathbb{E}[(\widehat Q_{H,\tau}^{\pi^t}(s,a))^2\mid s,a]
&\;\leq\;
\|\widehat Q\|_\infty^2 + \mathrm{Var}(\widehat Q_{H,\tau}^{\pi^t}(s,a)\mid s,a)
\;=\;
O(t_{\mathrm{mix}}^2) + O(H\,t_{\mathrm{mix}})\\
&\;=\;
O\bigl(t_{\mathrm{mix}}^2\log(2t_{\mathrm{mix}}/\varepsilon_{\mathrm{stat}})\bigr),
\end{align*}
where the last step uses $H=\Theta(t_{\mathrm{mix}}\log(2t_{\mathrm{mix}}/\varepsilon_{\mathrm{stat}}))$, the mean bound uses the bias and Poisson bounds above, and the variance bound follows from $\mathbb{E}[Y_i^2]\le 2(1+4\,t_{\mathrm{mix}})(\tau+\tau^2/(H-\tau))$ in the proof of Lemma~\ref{lem:ap-Q_estimation_burnin}, with $\tau=\lceil H/2\rceil$. Therefore, Lemma~\ref{cor:avg_reward_qnpg_sample_complexity} applies to $Z_1',\dots,Z_K'$, with target $\widehat Q$ in place of $Q^t$ (so that $w_t^\star$ there is $\tilde w_t$, and $|\widehat Q|\le G$ by Jensen's inequality), label second moment $G^2=O(t_{\mathrm{mix}}^2\log(2t_{\mathrm{mix}}/\varepsilon_{\mathrm{stat}}))$ with $G\ge1$, and $2\bar\varepsilon_{\mathrm{app}}^2+2\varepsilon_{\mathrm{stat}}^2/9$ in place of $\varepsilon_{\mathrm{app}}^2$. Here $\bar\varepsilon_{\mathrm{app}}:=\min\{\varepsilon_{\mathrm{app}},1+2\,t_{\mathrm{mix}}\}$ satisfies Assumption~\ref{as:approx-error} as well, since $w=0$ gives $\mathbb{E}_{\nu^t}[(Q^{\pi^t})^2]\le(1+2\,t_{\mathrm{mix}})^2$ by the Poisson bound above, and the minimality of $\tilde w_t$ together with $\|\widehat Q-Q^{\pi^t}\|_\infty\le\varepsilon_{\mathrm{stat}}/3$ gives
\begin{align*}
\mathbb{E}_{\nu^t}\bigl[(\tilde w_t^\top\phi-\widehat Q)^2\bigr]
&\le
\mathbb{E}_{\nu^t}\bigl[(w_t^{\star\top}\phi-\widehat Q)^2\bigr]\\
&\le
2\,\mathbb{E}_{\nu^t}\bigl[(w_t^{\star\top}\phi-Q^{\pi^t})^2\bigr]+2\|\widehat Q-Q^{\pi^t}\|_\infty^2
\le
2\bar\varepsilon_{\mathrm{app}}^2+\tfrac{2}{9}\varepsilon_{\mathrm{stat}}^2.
\end{align*}
At accuracy $\varepsilon_{\mathrm{stat}}/3\le G$, the choice of $K$ in Lemma~\ref{cor:avg_reward_qnpg_sample_complexity} is then of the order stated, which ensures
$\mathbb{E}[\|\widehat w'-\tilde w_t\|_{\Sigma_t}^2]\le \varepsilon_{\mathrm{stat}}^2/9$.

\paragraph{Burn-in term.}
Since $K$ and $H$ are bounded by polynomials in $t_{\mathrm{mix}}$, $1+L^2/\mu$ and $1/\varepsilon_{\mathrm{stat}}$ (using $d\le L^2/\mu$, from $d\mu\le\mathrm{Tr}(\Sigma_t)\le L^2$, and $\bar\varepsilon_{\mathrm{app}}\le 1+2\,t_{\mathrm{mix}}$), a suitable constant in the $\Theta(\cdot)$ choice of $m$ gives $m\ge t_{\mathrm{mix}}\bigl\lceil\log_2\!\bigl(2K(KH/2+(1+2\,t_{\mathrm{mix}})L^2/\mu)^2/\varepsilon_{\mathrm{stat}}^2\bigr)\bigr\rceil$. Then $\delta_m=2^{-\lfloor m/t_{\mathrm{mix}}\rfloor}$ makes the last term of~\eqref{eq:three_term_decomp} at most $\varepsilon_{\mathrm{stat}}^2/2$. In particular, $\delta_m\le1/(2K)$, since $\varepsilon_{\mathrm{stat}}<1$.

\paragraph{Sample complexity.}
Plugging the three bounds into \eqref{eq:three_term_decomp} gives
$\mathbb{E}[\|\widehat w-w_t^\star\|_{\Sigma_t}^2]\le 4\varepsilon_{\mathrm{stat}}^2/9+\varepsilon_{\mathrm{stat}}^2/2<\varepsilon_{\mathrm{stat}}^2$, conditionally on $\mathcal G$.
Finally, each SGD step uses $m$ transitions for its burn-in and $H$ transitions for its rollout ($B = 1$), so the per-step cost is $m + H = \widetilde O(t_{\mathrm{mix}})$ and the total number of transitions is $K(m+BH) = \widetilde O\bigl(d\,t_{\mathrm{mix}}^3/\varepsilon_{\mathrm{stat}}^2+(L^2/\mu)\,t_{\mathrm{mix}}(1+\varepsilon_{\mathrm{app}}^2/\varepsilon_{\mathrm{stat}}^2)\bigr)$.
\end{proof}

\begin{proof}[Proof of Lemma~\ref{lem:ap-w-sample-complexity}]
Let $w^{(1)},\dots,w^{(R)}$ be the outputs of the $R$ runs of Algorithm~\ref{alg:end2end_w_oracle} in Algorithm~\ref{alg:boost}, and let $\mathcal G_{i-1}$ be the $\sigma$-algebra generated by the trajectory up to the start of run $i$, so that $\mathcal G_0=\mathcal F_{t-1}$. All runs follow $\pi^t$. Call run $i$ \emph{accurate} if $\|w^{(i)}-w_t^\star\|_{\Sigma_t}\le\varepsilon_{\mathrm{stat}}/5$, and write $\|u\|_{\widehat\Sigma_t}:=\sqrt{u^\top\widehat\Sigma_t u}$.

\paragraph{Most runs are accurate.}
Let $\mathbf 1_i$ indicate that run $i$ is not accurate; it is $\mathcal G_i$-measurable. Lemma~\ref{lem:ap-single-run} at accuracy $\varepsilon_{\mathrm{stat}}/20$, applied with $\mathcal G=\mathcal G_{i-1}$, and Markov's inequality give $\mathbb{E}[\mathbf 1_i\mid\mathcal G_{i-1}]\le(\varepsilon_{\mathrm{stat}}/20)^2/(\varepsilon_{\mathrm{stat}}/5)^2=1/16$. Hence $\mathbb{E}[4^{\mathbf 1_i}\mid\mathcal G_{i-1}]\le1+3/16\le e^{3/16}$, and iterating conditional expectations gives $\mathbb{E}[4^{\sum_i\mathbf 1_i}\mid\mathcal F_{t-1}]\le e^{3R/16}$. By Markov's inequality,
\[
\Pr\Bigl(\textstyle\sum_{i=1}^R\mathbf 1_i\ge R/2\,\Big|\,\mathcal F_{t-1}\Bigr)\;\le\;e^{3R/16}\,4^{-R/2}\;=\;e^{-R(\log2-3/16)}\;\le\;e^{-R/2}\;\le\;\delta/2.
\]
So, outside an event of conditional probability at most $\delta/2$, more than half of the runs are accurate. The runs need not be independent.

\paragraph{The empirical covariance is accurate.}
List the $RK$ pairs of all runs in the order in which they are visited, with features $x_1,\dots,x_{RK}$, and let $\mathcal H_{l-1}$ be the $\sigma$-algebra generated by the trajectory up to the start of the burn-in that precedes the $l$-th pair. Set $\Phi_l:=(\mu/L^2)\,\Sigma_t^{-1/2}x_lx_l^\top\Sigma_t^{-1/2}$ and $\Phi_{\le l}:=\Phi_1+\dots+\Phi_l$. Then $0\preceq\Phi_l\preceq I_d$, because $\|\Sigma_t^{-1/2}x_l\|_2^2\le L^2/\mu$, and $\Phi_{\le RK}=(\mu RK/L^2)\,\Sigma_t^{-1/2}\widehat\Sigma_t\Sigma_t^{-1/2}$. As in the proof of Lemma~\ref{lem:ap-single-run}, the conditional law of the $l$-th pair given $\mathcal H_{l-1}$ is within $2^{-\lfloor m/t_{\mathrm{mix}}\rfloor}$ of $\nu^t$ in total variation. For every unit vector $u$, the function $(s,a)\mapsto(u^\top\Sigma_t^{-1/2}\phi_{s,a})^2$ takes values in $[0,L^2/\mu]$ and has mean $1$ under $\nu^t$, so
\[
\Bigl\|\frac{L^2}{\mu}\,\mathbb{E}[\Phi_l\mid\mathcal H_{l-1}]-I_d\Bigr\|_{\mathrm{op}}\;\le\;\frac{L^2}{\mu}\,2^{-\lfloor m/t_{\mathrm{mix}}\rfloor}\;\le\;\frac14,
\]
where the last step uses $2^{-\lfloor m/t_{\mathrm{mix}}\rfloor}\le1/(2K)$, from the choice of $m$ in that proof, and $K\ge32(L^2/\mu)\log(4d)\ge2L^2/\mu$, which a suitable constant in the choice of $K$ there guarantees. Hence $\frac{3\mu}{4L^2}I_d\preceq\mathbb{E}[\Phi_l\mid\mathcal H_{l-1}]\preceq\frac{5\mu}{4L^2}I_d$.

We now follow the iterative Golden--Thompson argument of Ahlswede and Winter \citep[Section~1.3.1]{tropp_user-friendly_2012}, which applies to our dependent sequence because each conditional bound below is a deterministic multiple of $I_d$. Convexity of $z\mapsto e^{\pm z/2}$ on $[0,1]$ gives $e^{\pm\Phi_l/2}\preceq I_d+(e^{\pm1/2}-1)\,\Phi_l$. With the bounds on $\mathbb{E}[\Phi_l\mid\mathcal H_{l-1}]$ above, $1+z\le e^z$, $e^{1/2}\le17/10$ and $e^{-1/2}\le5/8$, this gives
\[
\mathbb{E}\bigl[e^{\Phi_l/2}\bigm|\mathcal H_{l-1}\bigr]\preceq e^{7\mu/(8L^2)}\,I_d,
\qquad
\mathbb{E}\bigl[e^{-\Phi_l/2}\bigm|\mathcal H_{l-1}\bigr]\preceq e^{-9\mu/(32L^2)}\,I_d.
\]
The Golden--Thompson inequality gives $\operatorname{Tr}e^{\pm\Phi_{\le l}/2}\le\operatorname{Tr}\bigl(e^{\pm\Phi_{\le l-1}/2}\,e^{\pm\Phi_l/2}\bigr)$, and $e^{\pm\Phi_{\le l-1}/2}\succeq0$ is $\mathcal H_{l-1}$-measurable, so taking conditional expectations and iterating over $l=RK,\dots,1$ yields
\[
\mathbb{E}\bigl[\operatorname{Tr}e^{\Phi_{\le RK}/2}\bigm|\mathcal F_{t-1}\bigr]\le d\,e^{7\mu RK/(8L^2)},
\qquad
\mathbb{E}\bigl[\operatorname{Tr}e^{-\Phi_{\le RK}/2}\bigm|\mathcal F_{t-1}\bigr]\le d\,e^{-9\mu RK/(32L^2)}.
\]
By Markov's inequality, the event $\lambda_{\max}(\Phi_{\le RK})\ge2\mu RK/L^2$ has conditional probability at most $d\,e^{-\mu RK/(8L^2)}$, and the event $\lambda_{\min}(\Phi_{\le RK})\le\mu RK/(2L^2)$ has conditional probability at most $d\,e^{-\mu RK/(32L^2)}$. Hence $\frac12\Sigma_t\preceq\widehat\Sigma_t\preceq2\Sigma_t$ fails with conditional probability at most $2d\,e^{-\mu RK/(32L^2)}\le\delta/2$. Here we used $K\ge32(L^2/\mu)\log(4d)$, together with $R\ge2\log(2/\delta)\ge1+\log(1/\delta)$ and $\log(4d)\ge1$, which give $R\log(4d)\ge\log(4d/\delta)$.

\paragraph{The selected run has error at most $\varepsilon_{\mathrm{stat}}$.}
On the intersection of the two events above, which has conditional probability at least $1-\delta$, more than half of the runs are accurate and $\frac12\|u\|_{\Sigma_t}^2\le\|u\|_{\widehat\Sigma_t}^2\le2\|u\|_{\Sigma_t}^2$ for every $u\in\mathbb{R}^d$. The last bound holds for all $u$ at once, so it applies to the data-dependent differences $w^{(i)}-w^{(j)}$. Any two accurate runs satisfy $\|w^{(i)}-w^{(j)}\|_{\widehat\Sigma_t}\le\sqrt2\cdot2\varepsilon_{\mathrm{stat}}/5$, so every accurate run has $r_i\le2\sqrt2\,\varepsilon_{\mathrm{stat}}/5$, and hence $r_{i^\star}\le2\sqrt2\,\varepsilon_{\mathrm{stat}}/5$. The ball of radius $r_{i^\star}$ around $w^{(i^\star)}$ in $\|\cdot\|_{\widehat\Sigma_t}$ contains more than half of the runs, so it contains an accurate run $w^{(j)}$, and
\[
\|w_t-w_t^\star\|_{\Sigma_t}\le\|w^{(i^\star)}-w^{(j)}\|_{\Sigma_t}+\|w^{(j)}-w_t^\star\|_{\Sigma_t}\le\sqrt2\cdot\frac{2\sqrt2\,\varepsilon_{\mathrm{stat}}}{5}+\frac{\varepsilon_{\mathrm{stat}}}{5}=\varepsilon_{\mathrm{stat}}.
\]
This is the deterministic selection argument of \citet[Proposition~8]{hsu_loss_2016}; the two previous steps remove the independence that their probability bounds require, through conditional arguments along a single trajectory.

\paragraph{Sample complexity.}
Each run uses $K(m+BH)$ transitions, so Algorithm~\ref{alg:boost} uses $RK(m+BH)$. At accuracy $\varepsilon_{\mathrm{stat}}/20$, the parameters of Lemma~\ref{lem:ap-single-run} change only by constant factors, so with $R=\lceil2\log(2/\delta)\rceil$ the total is as stated.
\end{proof}

\subsection{Proof of Theorem~\ref{thm:non_tabular_pmd_sample_complexity_generative} (LFA sample complexity)}

\begin{theorem}[LFA sample complexity, detailed form]
\label{thm:ap-non_tabular_pmd_sample_complexity_generative}
Suppose Assumption~\ref{as:L2-mismatch-bound} and the hypotheses of Lemma~\ref{lem:w-sample-complexity} hold, and $D_0^\star<\infty$. Recall the initial Lyapunov value $\Psi_0 = \Delta_0 + D_0^\star/(\eta_{-1}\bar\chi)$. Fix a target suboptimality $\varepsilon\in(0,1)$ with $\varepsilon \ge 8(\bar\chi\sqrt{C_+}+\sqrt{C_\star})\varepsilon_{\mathrm{app}}$ and a confidence $\delta\in(0,1)$, and set $\varepsilon_{\mathrm{stat}} := \varepsilon/(8(\bar\chi\sqrt{C_+}+\sqrt{C_\star}))$ and $T := \lceil \bar\chi\log\max\{e, 2\Psi_0/\varepsilon\}\rceil$. Run the update~\eqref{eq:pmd-update} with $\tilde Q^t(s,a) = \phi_{s,a}^\top w_t$ for $T$ steps under the hypotheses of Lemma~\ref{lem:master}(i), where $w_t$ is the output of Algorithm~\ref{alg:boost} with the parameters of Lemma~\ref{lem:ap-w-sample-complexity} at accuracy $\varepsilon_{\mathrm{stat}}$ and confidence $\delta/T$. Then, with probability at least $1-\delta$,
\[
\Delta_T \;\le\; \varepsilon.
\]
Moreover, each outer iteration uses at most $RK(m+BH)$ transitions, with $R=\lceil 2\log(2T/\delta)\rceil$, hence the total number of transitions is at most $T\,RK(m+BH)$, which under these choices is
\[
T\,RK(m+BH)
\;=\;
\widetilde{O}\!\left(
\Bigl(\frac{d\,t_{\mathrm{mix}}^3\,(\bar\chi\sqrt{C_{+}}+\sqrt{C_{\star}})^2\;\bar\chi}{\varepsilon^{2}}+\frac{L^2}{\mu}\,t_{\mathrm{mix}}\,\bar\chi\Bigr)\log\frac{2T}{\delta}
\right),
\]
where $\widetilde{O}(\cdot)$ hides polylogarithmic factors in $t_{\mathrm{mix}}, L, \mu, d, \bar\chi, C_+, C_\star, 1/\varepsilon, \Psi_0$.
\end{theorem}

\begin{proof}
As for Theorem~\ref{thm:ap-pmd_sample_complexity}, the proof has two ingredients: a union bound converting the per-call guarantee of Lemma~\ref{lem:ap-w-sample-complexity} into critic accuracy along the PMD trajectory, and a deterministic composition with Theorem~\ref{thm:non-tabular-linear-convergence}.

\paragraph{Critic accuracy along the trajectory.}
For each $t=0,1,\dots,T-1$, the policy $\pi^t$ and the state of the trajectory at the start of critic call $t$ are $\mathcal F_{t-1}$-measurable, so Lemma~\ref{lem:ap-w-sample-complexity} at confidence $\delta/T$ gives
\[
\Pr\bigl(\|w_t-w_t^\star\|_{\Sigma_t}>\varepsilon_{\mathrm{stat}} \,\big|\, \mathcal F_{t-1}\bigr)\;\le\;\delta/T \quad\text{almost surely.}
\]
Taking expectations and a union bound over $t\in\{0,\dots,T-1\}$ then gives
\[
\Pr\bigl(\exists\, t<T : \|w_t-w_t^\star\|_{\Sigma_t}>\varepsilon_{\mathrm{stat}}\bigr)\;\le\;\delta.
\]
The critic calls share one trajectory, but the union bound needs no independence between them. On the complementary good event, Assumption~\ref{as:stat-error} holds along the realized PMD trajectory up to time $T$, which is all that the bound of Theorem~\ref{thm:non-tabular-linear-convergence} at $t=T$ uses.

\paragraph{Composition with Theorem~\ref{thm:non-tabular-linear-convergence}.}
Condition on the good event, intersected with the probability-one event of Assumption~\ref{as:L2-mismatch-bound}. Theorem~\ref{thm:non-tabular-linear-convergence} gives $(\Delta_T-\varepsilon_{\mathrm{floor}}(\bar\chi))_+\le(1-1/\bar\chi)^T\Psi_0$. Since $\varepsilon\ge 8(\bar\chi\sqrt{C_+}+\sqrt{C_\star})\varepsilon_{\mathrm{app}}$ gives $\varepsilon_{\mathrm{app}}\le\varepsilon_{\mathrm{stat}}$, the error floor satisfies
\[
\varepsilon_{\mathrm{floor}}(\bar\chi)
=
2(\bar\chi\sqrt{C_+}+\sqrt{C_\star})\big(\varepsilon_{\mathrm{stat}}+\varepsilon_{\mathrm{app}}\big)
\;\le\;
4(\bar\chi\sqrt{C_+}+\sqrt{C_\star})\,\varepsilon_{\mathrm{stat}}
\;=\;
\frac{\varepsilon}{2}.
\]
Using $(1-1/\bar\chi)^T\le\exp(-T/\bar\chi)$, the choice of $T$ ensures $(1-1/\bar\chi)^T\Psi_0\le\varepsilon/2$. Therefore $\Delta_T\le\varepsilon$ with probability at least $1-\delta$.

\paragraph{Sample complexity.}
Each outer iteration runs Algorithm~\ref{alg:boost} once, at accuracy $\varepsilon_{\mathrm{stat}}$ and confidence $\delta/T$. By Lemma~\ref{lem:ap-w-sample-complexity}, this uses $RK(m+BH)=\widetilde O\bigl((d\,t_{\mathrm{mix}}^3/\varepsilon_{\mathrm{stat}}^2+(L^2/\mu)\,t_{\mathrm{mix}})\log(2T/\delta)\bigr)$ transitions, since $\varepsilon_{\mathrm{app}}\le\varepsilon_{\mathrm{stat}}$ absorbs the approximation term. Substituting $\varepsilon_{\mathrm{stat}}$ and multiplying by $T=\widetilde\Theta(\bar\chi)$ gives the stated total.
\end{proof}

\subsection{Proof of Proposition~\ref{prop:nontabular_lower_bound} (LFA critic lower bound)}

We first restate the proposition with the explicit construction and constants used in the proof.

\begin{proposition}[LFA critic lower bound, detailed form]
\label{prop:ap-nontabular_lower_bound}
For every $d \geq 1$, $p \in (0, 1/4)$, and $\varepsilon_{\mathrm{stat}} \in (0, 1/(256\,p)]$, there exists a family of one-action average-reward MDPs on $2d$ states with costs in $[0, 1]$ satisfying Assumption~\ref{as:ergodic}, whose $1/4$-mixing time satisfies $t_{\mathrm{mix}} \in [1/(16p),\, 3/p]$, and admitting an exactly realizable linear differential $Q$-function $Q^\theta(s,a) = w_\theta^{\star\top}\phi(s,a)$ with $\|\phi(s,a)\|_2 \leq 1$ and stationary covariance $\Sigma = \mathbb{E}_{\nu^\theta}[\phi\phi^\top] = (1/d)\,I_d$, such that any (possibly adaptive) generative-model algorithm that outputs $\widehat w \in \mathbb{R}^d$ with
\[
\Pr_\theta\!\left(\|\widehat w - w_\theta^\star\|_\Sigma \leq \varepsilon_{\mathrm{stat}}\right) \;\geq\; 2/3 \qquad \text{for every } \theta \in \{0, 1\}^d
\]
must use
\[
N \;\geq\; \frac{1}{2^{21}\cdot 27}\,\frac{d\, t_{\mathrm{mix}}^3}{\varepsilon_{\mathrm{stat}}^2} \;=\; \Omega\!\left(\frac{d\, t_{\mathrm{mix}}^3}{\varepsilon_{\mathrm{stat}}^2}\right)
\]
generative-model transitions.
\end{proposition}
Here, an algorithm is adaptive if it chooses each query based on the samples seen so far.

\begin{proof}
The proof has six steps: construct the direct-sum reset family, verify its mixing properties, solve the centered Poisson equation, compute the per-coordinate gap and per-query KL, then assemble Assouad's inequality.

\paragraph{Construction.} Fix $d \geq 1$ and $p \in (0, 1/4)$. Let $\mathcal S = [d] \times \{-1, +1\}$ with one action $\mathcal A = \{a\}$, and parameterize the family by $\theta \in \{0, 1\}^d$ via local transition rates $q_j(\theta) := p + \theta_j\, \Delta$, where $\Delta \in (0, p/2]$ will be fixed in Step~6. The transition kernel is
\[
P_\theta\bigl((j, \sigma),\,\cdot\bigr) \;:=\; (1 - p)\,R_\theta\bigl((j, \sigma),\,\cdot\bigr) \;+\; p\, u(\cdot),
\]
where $u$ is the uniform distribution on $\mathcal S$ (so $u(j, \sigma) = 1/(2d)$) and $R_\theta$ is the local sign-flip kernel
\[
R_\theta\bigl((j, \sigma),\,(j, \sigma)\bigr) \;:=\; 1 - q_j, \qquad R_\theta\bigl((j, \sigma),\,(j, -\sigma)\bigr) \;:=\; q_j,
\]
with $R_\theta((j, \sigma), \cdot) = 0$ on all other states. Cost $c(j, \sigma) := (1 + \sigma)/2 \in \{0, 1\}$ and features $\phi(j, \sigma) := \sigma\, e_j \in \mathbb{R}^d$.

\paragraph{Step 1: Stationary distribution.} The uniform distribution $u$ is invariant under $P_\theta$. Indeed, the reset part moves every state to $u$, and for the local part,
\[
(u R_\theta)(k, \sigma) \;=\; u(k, \sigma)(1 - q_k) + u(k, -\sigma)\, q_k \;=\; \tfrac{1}{2d}(1 - q_k) + \tfrac{1}{2d}\, q_k \;=\; \tfrac{1}{2d} \;=\; u(k, \sigma).
\]
Hence $u P_\theta = (1 - p)\, u + p\, u = u$. Since $P_\theta(s, s') \geq p\, u(s') = p/(2d) > 0$ for all $s, s'$, the chain is irreducible and aperiodic, so $d^\theta = u$ is the unique stationary distribution.

\paragraph{Step 2: Mixing time.} Decomposing $P_\theta = (1 - p)\, R_\theta + p\,\mathbf 1u$, where the kernel $\mathbf 1u$ moves every state to $u$, and using $u\, R_\theta = u$,
\[
P_\theta^k \;=\; (1 - p)^k\, R_\theta^k \;+\; \bigl(1 - (1 - p)^k\bigr)\, \mathbf 1u,
\]
so $\mu_0 P_\theta^k - u = (1-p)^k(\mu_0 R_\theta^k - u)$ and $\sup_{\mu_0} \|\mu_0 P_\theta^k - u\|_{\mathrm{TV}} \leq (1-p)^k$. Setting this $\leq 1/4$ and using $\log(1/(1-p)) \geq p$ gives the upper bound
\begin{equation} \label{eq:nlb-tmix-upper}
t_{\mathrm{mix}} \;\leq\; \lceil \log 4/p \rceil \;\leq\; (\log 4 + 1)/p \;<\; 3/p.
\end{equation}

For the lower bound, define the coordinate-local function $f_j(k, \sigma) := \mathbf{1}\{k = j\}\, \sigma$, which has $\|f_j\|_\infty \leq 1$ and $\mathbb{E}_u[f_j] = 0$. Direct computation shows $f_j$ is an eigenfunction of $P_\theta$ with eigenvalue $\lambda_j := (1 - p)(1 - 2 q_j)$. Starting from $(j, +1)$,
\[
\bigl\|\delta_{(j, +1)}\, P_\theta^k - u\bigr\|_{\mathrm{TV}} \;\geq\; \tfrac{1}{2}\, \mathbb{E}_{(j, +1)}[f_j(s_k)] \;=\; \tfrac{1}{2}\,\lambda_j^k.
\]
With $q_j \leq p + \Delta \leq 3p/2$, $\lambda_j \geq (1-p)(1-3p) \geq 1 - 4p$ for $p \leq 1/16$. For $k \leq 1/(16p)$, $(1-4p)^k \geq \exp(-4pk/(1-4p)) \geq \exp(-1/3) > 1/2$. Hence for $p \leq 1/16$ and $k \leq 1/(16p)$,
\[
\bigl\|\delta_{(j, +1)}\, P_\theta^k - u\bigr\|_{\mathrm{TV}} \;>\; 1/4.
\]
For $p \in [1/16, 1/4)$, the trivial bound $t_{\mathrm{mix}} \geq 1$ implies $t_{\mathrm{mix}} \geq 1/(16p)$. Combining the upper and lower bounds,
\[
1/(16p) \;\leq\; t_{\mathrm{mix}} \;\leq\; 3/p.
\]

\paragraph{Step 3: Poisson equation and linear realizability.} Since $d^\theta = u$, the average cost is $\rho(\theta) = \mathbb{E}_u[c] = 1/2$. The centered Poisson equation $\rho + h^\theta = c + P_\theta h^\theta$ with $\sum_s u(s) h^\theta(s) = 0$ admits the ansatz $h^\theta(j, \sigma) = \beta_j(\theta)\, \sigma$ (the centering is automatic since $\sum_\sigma \sigma = 0$). Computing $(P_\theta h^\theta)(j, \sigma) = (1 - p)(1 - 2 q_j)\, \beta_j\, \sigma$ (the reset term vanishes since $\mathbb{E}_u[h^\theta] = 0$) and substituting,
\[
\tfrac{1}{2} + \beta_j\, \sigma \;=\; \tfrac{1 + \sigma}{2} \;+\; (1 - p)(1 - 2 q_j)\, \beta_j\, \sigma \quad\Longrightarrow\quad \beta_j(\theta) \;=\; \frac{1}{2\bigl(p + 2 q_j(1 - p)\bigr)}.
\]
With $q_j \in [p, 3p/2]$, the denominator lies in $[5p, 8p]$, so $\beta_j(\theta) \in [1/(8p),\, 1/(5p)] = \Theta(1/p)$.

Since there is one action, $Q^\theta(j, \sigma, a) = h^\theta(j, \sigma) = \beta_j(\theta)\, \sigma = w_\theta^{\star\top}\phi(j, \sigma)$ with $w_\theta^\star := (\beta_1(\theta), \dots, \beta_d(\theta))^\top$. The approximation error is exactly zero. The covariance matrix is
\[
\Sigma \;=\; \mathbb{E}_{\nu^\theta}[\phi\phi^\top] \;=\; \frac{1}{2d}\sum_{j=1}^d \sum_{\sigma}(\sigma e_j)(\sigma e_j)^\top \;=\; \frac{1}{d}\, I_d,
\]
giving $\|\phi\|_2 \leq 1$ ($L = 1$) and $\mu = 1/d$.

\paragraph{Step 4: Per-coordinate gap.} Write $\beta(q) := 1/(2p + 4q(1-p))$. Since $1/\beta(p) = p(6 - 4p) \in [5p, 6p]$ and $1/\beta(p + \Delta) \in [5p, 8p]$ for $\Delta \leq p/2$,
\[
|\beta(p + \Delta) - \beta(p)| \;=\; 4(1-p)\,\Delta\,\beta(p)\,\beta(p + \Delta) \;\geq\; \frac{3\,\Delta}{(6p)(8p)} \;=\; \frac{1}{16}\,\frac{\Delta}{p^2},
\]
where we used $4(1-p) \geq 3$ for $p \leq 1/4$. Define $\delta_w := |\beta(p + \Delta) - \beta(p)| \geq \Delta/(16 p^2)$.

\paragraph{Step 5: Per-query KL.} Fix $\theta, \theta'$ differing only in coordinate $j$, with $q_j = p$ and $q_j' = p + \Delta$. A query from any state $(j', \sigma')$ with $j' \neq j$ has identical next-state law under both hypotheses (only $q_j$ changes), so the per-query KL is zero.

For a query from $(j, \sigma)$, the only changing atoms are $(j, \sigma)$ and $(j, -\sigma)$:
\begin{align*}
P_q\bigl((j, \sigma),\,(j, \sigma)\bigr) &\;=\; P_{\mathrm{stay}}(q) := (1 - p)(1 - q) + p/(2d), \\
P_q\bigl((j, \sigma),\,(j, -\sigma)\bigr) &\;=\; P_{\mathrm{flip}}(q) := (1 - p)\, q + p/(2d),
\end{align*}
By the standard bound of the KL divergence by the chi-square divergence,
\[
\mathrm{KL}(P_q \,\|\, P_{q'}) \;\leq\; (1-p)^2\, \Delta^2\!\left(\frac{1}{P_{\mathrm{stay}}(q')} + \frac{1}{P_{\mathrm{flip}}(q')}\right).
\]
For $p \leq 1/4$ and $\Delta \leq p/2$, $P_{\mathrm{flip}}(q') \geq (1-p)p \geq 3p/4$ (so $1/P_{\mathrm{flip}}(q') \leq 4/(3p)$) and $P_{\mathrm{stay}}(q') \geq (1-p)(1 - p - \Delta) \geq (3/4)(5/8) = 15/32$, hence $1/P_{\mathrm{stay}}(q') \leq 32/15 \leq 8/(15p)$ (using $p \leq 1/4$, so $1/p \geq 4$). Combining,
\[
\frac{1}{P_{\mathrm{stay}}(q')} + \frac{1}{P_{\mathrm{flip}}(q')} \;\leq\; \left(\frac{8}{15} + \frac{4}{3}\right)\frac{1}{p} \;=\; \frac{28}{15\,p} \;<\; \frac{2}{p}.
\]
Since $(1-p)^2 \leq 1$,
\begin{equation} \label{eq:nlb-kl-bound}
\mathrm{KL}\bigl(P_\theta(\cdot \mid j, \sigma) \,\big\|\, P_{\theta'}(\cdot \mid j, \sigma)\bigr) \;\leq\; 2\,\frac{\Delta^2}{p}.
\end{equation}

\paragraph{Step 6: Assouad assembly.} Suppose, for contradiction, an algorithm using at most $N$ generative-model queries satisfies the success condition. Let $\mathbb{P}_\theta$ be the law under $\theta$ of the transcript $(U, S_1, Y_1, \dots, S_N, Y_N, \widehat w)$, where $U$ is the algorithm's internal randomness, $S_n$ is the queried state at round $n$, and $Y_n$ is the observed next state. The query rule $S_n$ is measurable with respect to $(U, S_1, Y_1, \dots, S_{n-1}, Y_{n-1})$ and has the same conditional law under every $\theta$; only the next-state law $P_\theta(\cdot \mid S_n)$ depends on $\theta$. Hence the KL chain rule gives
\[
\mathrm{KL}\bigl(\mathbb{P}_\theta \,\big\|\, \mathbb{P}_{\theta'}\bigr) \;=\; \sum_{n=1}^N \mathbb{E}_\theta\!\left[\mathrm{KL}\bigl(P_\theta(\cdot \mid S_n) \,\big\|\, P_{\theta'}(\cdot \mid S_n)\bigr)\right].
\]
Letting $N_j := \#\{n : S_n \in \{j\} \times \{-1, +1\}\}$ be the number of queries to coordinate $j$ (so $\sum_j N_j \leq N$) and using~\eqref{eq:nlb-kl-bound} with zero cross-coordinate KL,
\[
\mathrm{KL}(\mathbb{P}_\theta \,\|\, \mathbb{P}_{\theta^{(j)}}) \;\leq\; 2\,\frac{\Delta^2}{p}\,\mathbb{E}_\theta[N_j],
\]
where $\theta^{(j)}$ is $\theta$ with bit $j$ flipped. Since $\beta(q) = 1/(2p + 4q(1-p))$ is \emph{decreasing} in $q$, we have $w_{\theta, j}^\star = \beta(p) > \beta(p + \Delta) = w_{\theta^{(j)}, j}^\star$ when $\theta_j = 0$. Accordingly, define the bit decoder
\[
\widehat\theta_j \;:=\; \mathbf{1}\!\left\{\widehat w_j < (\beta(p) + \beta(p + \Delta))/2\right\}.
\]
If $\widehat\theta_j \neq \theta_j$, then $|\widehat w_j - w_{\theta, j}^\star| \geq \delta_w/2$, so each Hamming error contributes $\delta_w^2/(4d)$ to $\|\widehat w - w_\theta^\star\|_\Sigma^2 = (1/d)\|\widehat w - w_\theta^\star\|_2^2$:
\begin{equation} \label{eq:nlb-hamming-loss}
d_H(\widehat\theta, \theta) \;\leq\; \frac{4 d}{\delta_w^2}\, \|\widehat w - w_\theta^\star\|_\Sigma^2.
\end{equation}

Choose $\Delta := 128\, p^2 \varepsilon_{\mathrm{stat}}$. Then $\delta_w \geq \Delta/(16 p^2) = 8 \varepsilon_{\mathrm{stat}}$, so $\delta_w^2 \geq 64\, \varepsilon_{\mathrm{stat}}^2$. The constraint $\Delta \leq p/2$ becomes $\varepsilon_{\mathrm{stat}} \leq 1/(256\, p)$, which is the stated regime. Under the success hypothesis $\Pr_\theta(\|\widehat w - w_\theta^\star\|_\Sigma \leq \varepsilon_{\mathrm{stat}}) \geq 2/3$ and using \eqref{eq:nlb-hamming-loss},
\[
\mathbb{E}_\theta[d_H(\widehat\theta, \theta)] \;\leq\; \tfrac{2}{3} \cdot \tfrac{4 d\, \varepsilon_{\mathrm{stat}}^2}{64\, \varepsilon_{\mathrm{stat}}^2} + \tfrac{1}{3} \cdot d \;=\; \tfrac{2}{3} \cdot \tfrac{d}{16} + \tfrac{d}{3} \;=\; \tfrac{3 d}{8}.
\]
Averaging over $\theta \in \{0, 1\}^d$, the average Hamming risk satisfies $R_H(\widehat\theta) \leq 3 d/8$.

We use the standard Hamming-risk form of Assouad's inequality \citep{yu_assouad_1997}: pairing each $\theta$ with $\theta^{(j)}$ for each $j$ and applying Pinsker plus Jensen,
\[
R_H(\widehat\theta) \;\geq\; \frac{d}{2}\!\left[1 - \sqrt{\frac{1}{2 d\, 2^{d-1}}\sum_{j=1}^d \sum_{\theta : \theta_j = 0}\mathrm{KL}(\mathbb{P}_\theta \,\|\, \mathbb{P}_{\theta^{(j)}})}\right].
\]
By the per-query KL bound and $\sum_j N_j \leq N$, $\sum_{j} \sum_{\theta : \theta_j = 0} \mathbb{E}_\theta[N_j] \leq 2^d\, N$, so the expression under the square root is bounded by $2\, N\, \Delta^2/(d\, p)$, giving
\[
R_H(\widehat\theta) \;\geq\; \frac{d}{2}\!\left[1 - \sqrt{2\,\frac{N\, \Delta^2}{d\, p}}\right].
\]
If $N \leq d\, p/(128\, \Delta^2)$, the radicand is $\leq 1/64$, hence $R_H(\widehat\theta) \geq (d/2)(1 - 1/8) = 7 d/16$. But we showed $R_H(\widehat\theta) \leq 3 d/8 < 7 d/16$, a contradiction. Therefore
\[
N \;>\; \frac{d\, p}{128\, \Delta^2} \;=\; \frac{d\, p}{128 \cdot 128^2\, p^4\, \varepsilon_{\mathrm{stat}}^2} \;=\; \frac{1}{2^{21}}\,\frac{d}{p^3\, \varepsilon_{\mathrm{stat}}^2}.
\]
Using $1/p^3 \geq t_{\mathrm{mix}}^3/27$ from \eqref{eq:nlb-tmix-upper},
\[
N \;\geq\; \frac{1}{2^{21}\cdot 27}\,\frac{d\, t_{\mathrm{mix}}^3}{\varepsilon_{\mathrm{stat}}^2}. \qedhere
\]
\end{proof}

\subsection{Sufficient conditions for the covariance condition of Lemma~\ref{lem:w-sample-complexity}} \label{ap:covariance-condition}
For $\pi \in \Pi$, let $\Sigma_\pi := \mathbb{E}_{(s,a)\sim\nu^\pi}[\phi_{s,a}\phi_{s,a}^\top]$, so that $\Sigma_t = \Sigma_{\pi^t}$.

\paragraph{Spanning features.} Suppose $\operatorname{span}\{\phi_{s,a_s} : s \in \mathcal S\} = \mathbb{R}^d$ for every $(a_s)_{s\in\mathcal S} \in \mathcal A^{\mathcal S}$. If $u^\top \Sigma_\pi u = \sum_{s} d^\pi(s) \sum_{a} \pi(a\mid s)\,(u^\top \phi_{s,a})^2 = 0$, then, since $d^\pi > 0$ (Lemma~\ref{lem:uniform-d-bound}), $u^\top \phi_{s,a} = 0$ whenever $\pi(a \mid s) > 0$; choosing one such $a_s$ in every state gives $u \perp \operatorname{span}\{\phi_{s,a_s}\}_{s} = \mathbb{R}^d$, so $u = 0$. Hence $\Sigma_\pi \succ 0$ for every $\pi \in \Pi$. The map $\pi \mapsto \Sigma_\pi$ is continuous, as $\pi \mapsto d^\pi$ is (see the proof of Lemma~\ref{lem:uniform-d-bound}), and $\Pi$ is compact, so $\mu := \min_{\pi\in\Pi} \lambda_{\min}(\Sigma_\pi) > 0$ and $\Sigma_\pi \succeq \mu I_d$ for all $\pi$. When $d \le |\mathcal S|$, the spanning condition holds whenever the features are in general position, that is, whenever any $d$ of them are linearly independent, and hence for almost every choice of features.

\paragraph{$\varepsilon$-exploration.} Without the spanning condition, enforce $\pi^t(a \mid s) \ge \varepsilon/|\mathcal A|$ as in Remark~\ref{rmk:eps-exploration}. Then $\nu^t(s,a) \ge \underline d\,\varepsilon/|\mathcal A|$, so $\Sigma_t \succeq (\underline d\,\varepsilon/|\mathcal A|)\sum_{s,a}\phi_{s,a}\phi_{s,a}^\top$, and the condition holds with $\mu = (\underline d\,\varepsilon/|\mathcal A|)\,\lambda_{\min}\bigl(\sum_{s,a}\phi_{s,a}\phi_{s,a}^\top\bigr)$, which is positive whenever $\operatorname{span}\{\phi_{s,a} : (s,a) \in \mathcal S \times \mathcal A\} = \mathbb{R}^d$.

\paragraph{Hard instances.} The family of Proposition~\ref{prop:nontabular_lower_bound} has features $\phi(j,\sigma)=\sigma e_j$ of norm $1$ and covariance $\Sigma = I_d/d$ (see the proof of Proposition~\ref{prop:ap-nontabular_lower_bound}). Lemma~\ref{lem:w-sample-complexity} assumes $\|\phi_{s,a}\|_2\le L$ and $\Sigma_t\succeq\mu I_d$, which force $d\mu\le\mathrm{Tr}(\Sigma_t)\le L^2$. Multiplying the features of the family by $\sqrt{d\mu}$ and dividing $w_\theta^\star$ by $\sqrt{d\mu}$ gives features of norm $\sqrt{d\mu}\le L$ and covariance $\mu I_d$, and leaves the differential $Q$-function unchanged; an estimate $\widehat w$ for the rescaled family corresponds to $\sqrt{d\mu}\,\widehat w$ for the original one, with the same error in the covariance norm. The lower bound therefore holds for every $L$ and $\mu$ allowed by Lemma~\ref{lem:w-sample-complexity}, so the condition does not make the estimation problem easier. Finally, a generative model simulates a single trajectory with one query per transition, so the lower bound also applies to the single-trajectory oracle of Lemma~\ref{lem:w-sample-complexity}.

\section{Mixing time details} \label{ap:mixing-assumption-details}

Under Assumption~\ref{as:ergodic}, for each fixed policy $\pi$ the induced Markov chain with kernel $P_\pi$ is finite-state, irreducible, and aperiodic, and therefore admits a unique stationary distribution $d^\pi$ and converges to it at a geometric rate in total variation. In particular, every such $\pi$ has a finite mixing time
\begin{equation} \label{eq:per-policy-tmix}
t_{\mathrm{mix}}(\pi) \;:=\; \min\!\left\{k \geq 1 : \sup_{\mu_0 \in \mathcal{P}(\mathcal S)} \bigl\|\mu_0\, P_\pi^k - d^\pi\bigr\|_{\mathrm{TV}} \leq \tfrac{1}{4}\right\} \;<\; \infty.
\end{equation}
A self-contained proof is given in Lemma~\ref{lem:per-policy-mixing} below.

Lemma~\ref{lem:uniform-mix} upgrades \eqref{eq:per-policy-tmix} to a single constant $t_{\mathrm{mix}}$ that works \emph{uniformly} over all policies $\pi \in \Pi$; its proof is given at the end of this appendix. The argument is non-constructive, so we first record a standard condition under which $t_{\mathrm{mix}}$ is explicit.

This condition is uniform minorization (see, e.g., \citet[Assumption~3.1]{murthy_performance_2023}). One convenient MDP-level formulation is: there exist $\alpha \in (0, 1)$ and a reference distribution $\nu \in \mathcal{P}(\mathcal S)$ such that
\begin{equation} \label{eq:doeblin-controlled}
P(s' \mid s, a) \;\geq\; \alpha\, \nu(s') \qquad\text{for all } s, s' \in \mathcal S,\; a \in \mathcal A.
\end{equation}
Since $P_\pi(\cdot \mid s) = \sum_a \pi(a \mid s)\, P(\cdot \mid s, a)$, \eqref{eq:doeblin-controlled} implies $P_\pi(\cdot \mid s) \geq \alpha\, \nu$ for every $\pi$ and $s$, and consequently yields uniform mixing with an explicit bound on $t_{\mathrm{mix}}$. Define the residual kernel $R_\pi(\cdot \mid s) := (P_\pi(\cdot \mid s) - \alpha\, \nu)/(1 - \alpha)$; then $P_\pi = \alpha\, \mathbf{1}\, \nu + (1 - \alpha)\, R_\pi$, and for any distributions $\mu_0, \mu_0'$,
\[
\|\mu_0 P_\pi - \mu_0' P_\pi\|_{\mathrm{TV}} \;=\; (1 - \alpha)\, \|(\mu_0 - \mu_0')\, R_\pi\|_{\mathrm{TV}} \;\leq\; (1 - \alpha)\, \|\mu_0 - \mu_0'\|_{\mathrm{TV}}.
\]
Taking $\mu_0' = d^\pi$ and iterating yields $\|\mu_0\, P_\pi^k - d^\pi\|_{\mathrm{TV}} \leq (1 - \alpha)^k$. Requiring $(1 - \alpha)^k \leq 1/4$ gives $t_{\mathrm{mix}} \leq \lceil \log 4 / \log(1/(1-\alpha)) \rceil = O(1/\alpha)$ as $\alpha \to 0$.

The same condition also makes $\bar\chi$ explicit. Stationarity gives $d^\pi(s') = \sum_{s} d^\pi(s)\, P_\pi(s' \mid s) \geq \alpha\, \nu(s')$ for every $\pi$ and $s'$. If $\nu$ has full support, the constant of Lemma~\ref{lem:uniform-d-bound} therefore satisfies $\underline d \geq \alpha \min_{s} \nu(s)$, and since $d^\star \leq 1$,
\[
\sup_{t \geq 0} \chi_t \;\leq\; \frac{1}{\underline d} \;\leq\; \frac{1}{\alpha\, \min_{s} \nu(s)},
\]
so $\bar\chi = 1/(\alpha \min_s \nu(s))$ is a valid choice in Lemma~\ref{lem:master}(i). Under~\eqref{eq:doeblin-controlled}, both $t_{\mathrm{mix}}$ and $\bar\chi$ are thus explicit.

\begin{lemma}[Finite mixing time for a fixed policy]
\label{lem:per-policy-mixing}
Let $\pi$ be a policy such that $P_\pi$ is irreducible and aperiodic. Then $t_{\mathrm{mix}}(\pi)$ defined by~\eqref{eq:per-policy-tmix} is finite.
\end{lemma}

\begin{proof}
Fix $\pi$. Since $P_\pi$ is finite-state, irreducible, and aperiodic, it is \emph{primitive}; hence there exists an integer $m_\pi \geq 1$ such that $P_\pi^{m_\pi}(s' \mid s) > 0$ for all $s, s' \in \mathcal S$ (see \citet[Proposition~1.7]{levin_markov_2017}). Let
\[
\varepsilon_\pi \;:=\; \min_{s, s' \in \mathcal S} P_\pi^{m_\pi}(s' \mid s) \;>\; 0, \qquad \alpha_\pi \;:=\; \min\{|\mathcal S|\, \varepsilon_\pi,\, 1/2\} \in (0, 1/2],
\]
and let $\nu$ be the uniform distribution on $\mathcal S$. Then $P_\pi^{m_\pi}(\cdot \mid s) \geq \alpha_\pi\, \nu$ for every $s$. Defining the residual kernel $R_\pi(\cdot \mid s) := (P_\pi^{m_\pi}(\cdot \mid s) - \alpha_\pi\, \nu)/(1 - \alpha_\pi)$, we obtain $P_\pi^{m_\pi} = \alpha_\pi\, \mathbf{1}\, \nu + (1 - \alpha_\pi)\, R_\pi$, and by the same non-expansive argument as above,
\[
\sup_{\mu_0}\, \bigl\|\mu_0\, P_\pi^{q\, m_\pi} - d^\pi\bigr\|_{\mathrm{TV}} \;\leq\; (1 - \alpha_\pi)^q \qquad \forall\, q \geq 0.
\]
Choosing $q$ large enough that $(1 - \alpha_\pi)^q \leq 1/4$, namely $q \geq \lceil \log 4 / \log(1/(1-\alpha_\pi)) \rceil$, yields $t_{\mathrm{mix}}(\pi) \leq q\, m_\pi < \infty$.
\end{proof}

\begin{proof}[Proof of Lemma~\ref{lem:uniform-mix}, first part]
For $\pi \in \Pi$ and $k \geq 1$, write $\mathrm{TV}_\pi(k) := \sup_{\mu_0 \in \mathcal{P}(\mathcal S)} \|\mu_0\, P_\pi^k - d^\pi\|_{\mathrm{TV}}$. We use three facts.
\emph{(i) Monotonicity.} Since $d^\pi$ is stationary and $P_\pi$ is non-expansive in total variation, $\mathrm{TV}_\pi(k+1) \leq \mathrm{TV}_\pi(k)$.
\emph{(ii) Continuity.} For fixed $k$, the map $\mu_0 \mapsto \|\mu_0\, P_\pi^k - d^\pi\|_{\mathrm{TV}}$ is convex, so its supremum over the simplex is attained at a vertex, $\mathrm{TV}_\pi(k) = \max_{s \in \mathcal S} \|P_\pi^k(\cdot \mid s) - d^\pi\|_{\mathrm{TV}}$. Both $\pi \mapsto P_\pi^k$ and $\pi \mapsto d^\pi$ are continuous on $\Pi$ (the latter is shown in the proof of Lemma~\ref{lem:uniform-d-bound}), so $\pi \mapsto \mathrm{TV}_\pi(k)$ is continuous as a finite maximum of continuous functions.
\emph{(iii) Pointwise decay.} For each $\pi$, the proof of Lemma~\ref{lem:per-policy-mixing} gives $\mathrm{TV}_\pi(q\, m_\pi) \leq (1 - \alpha_\pi)^q \to 0$, so there is $k_\pi \geq 1$ with $\mathrm{TV}_\pi(k_\pi) < 1/4$.

By (ii), $U_\pi := \{\pi' \in \Pi : \mathrm{TV}_{\pi'}(k_\pi) < 1/4\}$ is an open neighborhood of $\pi$. Since $\Pi = \mathcal{P}(\mathcal A)^{\mathcal S}$ is compact, finitely many $U_{\pi_1}, \dots, U_{\pi_n}$ cover it. Let $t_{\mathrm{mix}} := \max_i k_{\pi_i}$. Any $\pi' \in \Pi$ lies in some $U_{\pi_i}$, and (i) gives $\mathrm{TV}_{\pi'}(t_{\mathrm{mix}}) \leq \mathrm{TV}_{\pi'}(k_{\pi_i}) < 1/4$.
\end{proof}

\paragraph{Proof of~\eqref{eq:tmix-decay}.}
With $\mathrm{TV}_\pi(k)$ as in the proof above, let $\overline{\mathrm{TV}}_\pi(k) := \sup_{\mu_0, \mu_0' \in \mathcal{P}(\mathcal S)} \|\mu_0\, P_\pi^k - \mu_0'\, P_\pi^k\|_{\mathrm{TV}}$. Then $\mathrm{TV}_\pi \leq \overline{\mathrm{TV}}_\pi \leq 2\,\mathrm{TV}_\pi$, and $\overline{\mathrm{TV}}_\pi$ is submultiplicative, $\overline{\mathrm{TV}}_\pi(k + l) \leq \overline{\mathrm{TV}}_\pi(k)\, \overline{\mathrm{TV}}_\pi(l)$ \citep[Lemmas~4.10--4.11]{levin_markov_2017}. The first part of Lemma~\ref{lem:uniform-mix} gives $\overline{\mathrm{TV}}_\pi(t_{\mathrm{mix}}) \leq 2\,\mathrm{TV}_\pi(t_{\mathrm{mix}}) \leq 1/2$, so $\mathrm{TV}_\pi(\ell\, t_{\mathrm{mix}}) \leq \overline{\mathrm{TV}}_\pi(t_{\mathrm{mix}})^\ell \leq 2^{-\ell}$ for every integer $\ell \geq 0$. Since $\mathrm{TV}_\pi$ is nonincreasing by (i), $\mathrm{TV}_\pi(k) \leq 2^{-\lfloor k/t_{\mathrm{mix}} \rfloor}$ for all $k \geq 0$, which is~\eqref{eq:tmix-decay}. The factor $1/2$ per block, rather than the $1/4$ at $k = t_{\mathrm{mix}}$, is the price of a single bound valid for every $k$.

\paragraph{Extension to unichain MDPs.}
Assumption~\ref{as:ergodic} can be weakened to the unichain condition (each induced kernel $P_\pi$ has a unique recurrent class) while retaining aperiodicity, at the cost of additional constants, and the mechanism is visible in our analysis. \emph{For the convergence results}, whenever the iterates remain interior (automatic under KL, where the update is multiplicative), each path mismatch $\chi_t$ stays finite. What is lost is the a priori uniform bound $\chi_t \le \underline d^{-1}$ of Lemma~\ref{lem:uniform-d-bound}, so a uniform path bound $\bar\chi$ must be imposed or verified rather than inherited; once it is supplied, the recursion of Lemma~\ref{lem:master} applies verbatim. \emph{For the sample bounds}, uniform mixing is the assumption of the mixing unichain setting of \citet[Assumption~1]{li_stochastic_2024}, who note that aperiodicity ensures it. The proof of Lemma~\ref{lem:uniform-mix} indeed carries over when every $P_\pi$, $\pi\in\Pi$, is unichain with an aperiodic recurrent class, with two changes. For continuity, every $v$ with $vP_\pi = v$ satisfies $v = vP_\pi^k \to (v\mathbf 1)\,d^\pi$ by the pointwise decay, so the linear system in the proof of Lemma~\ref{lem:uniform-d-bound} still has the unique solution $d^\pi$ and is therefore nonsingular, although $d^\pi$ may now vanish on transient states. For the pointwise decay, every state reaches the recurrent class, which is aperiodic, so $P_\pi^{m_\pi}(s'\mid s) > 0$ for some $m_\pi\ge1$, every $s\in\mathcal S$ and every $s'$ in the recurrent class, and the proof of Lemma~\ref{lem:per-policy-mixing} applies with $\nu$ uniform on that class. With $t_{\mathrm{mix}}$ also absorbing the time to reach the recurrent class, the sample bounds are stated exactly as before. In both halves the change is in constants, not in the form of the results. By contrast, the multichain case (some induced kernel $P_\pi$ has more than one recurrent class) requires additional analysis, as in \citet{lee_policy_2026}.

\section{Numerical experiments} \label{ap:numerical-experiments}

\paragraph{Code availability.}
All code for reproducing the experiments in this appendix is available at
\begin{center}
\url{https://github.com/enesarda22/avg-reward-pmd-experiments}
\end{center}

We report three experiments. The first two (Sections~\ref{ap:exp-exact}
and~\ref{ap:exp-inexact}) are proof-of-concept diagnostics on a shared
three-state MDP, targeting Corollary~\ref{cor:exact-pmd} and
Theorem~\ref{thm:avg-pmd-linear-inexact-limsup} respectively: they show
that the asymptotic rate is governed by the tail mismatch $\chi$ rather
than the worst-case $\bar\chi$, and that the error floor is independent
of the step-size schedule, empirically realizing the decoupling argument
of Section~\ref{sec:preliminaries}. The third
(Section~\ref{ap:exp-rate-gap}) extends this story to four diverse MDPs,
showing that the $\chi \ll \bar\chi$ separation is broadly observed in
practice. All experiments run on a single CPU laptop in under five minutes total.

\subsection{Setup} \label{ap:exp-setup}

\paragraph{MDP family.}
We use a three-state ergodic average-\emph{cost} MDP with
$\mathcal S=\{0,1,2\}$ and $\mathcal A=\{0,1\}$. The transition kernel is
\begin{align*}
P(0,0)&=(1-p,\,0,\,p), &
P(0,1)&=(0,\,1-p,\,p),\\
P(1,0)&=(0,\,1-p,\,p), &
P(1,1)&=(1-p,\,0,\,p),\\
P(2,0)&=(0.19,\,0.80,\,0.01), &
P(2,1)&=(0.95,\,0.04,\,0.01),
\end{align*}
with $p=0.05$. The small self-loop at state $2$ makes the chain aperiodic
under all policies. Costs are $c(2,0)=0$, $c(2,1)=0.2$, and
$c(s,a)=0$ for $s\in\{0,1\}$ in the \emph{multiple-minimizer} instance; we
also consider a \emph{unique-minimizer} instance obtained by adding a
tie-breaker $\varepsilon>0$ via $c_\text{unique}(0,0)=\varepsilon$,
$c_\text{unique}(1,1)=\varepsilon$ with the other entries unchanged. The
initial policy is nearly deterministic and placed at high probability on the
costly action in state $2$:
$\pi^0(0\mid 0)=\pi^0(1\mid 1)=0.999$,
$\pi^0(1\mid 0)=\pi^0(0\mid 1)=0.001$,
$\pi^0(0\mid 2)=0.010$, $\pi^0(1\mid 2)=0.990$.

\paragraph{Algorithm.}
We run KL-PMD,
\[
\pi^{t+1}(a\mid s)\;\propto\;\pi^{t}(a\mid s)\exp\!\bigl(-\eta_t\,\widehat Q^t(s,a)\bigr),
\]
where $\widehat Q^t$ is either the exact differential action-value
$Q^{\pi^t}$ (Section~\ref{ap:exp-exact}) or a noisy approximation
(Section~\ref{ap:exp-inexact}). At each iteration we compute $Q^{\pi^t}$
exactly by solving the Poisson equation for $h^{\pi^t}$ in
\eqref{eq:ap-poisson}.

\paragraph{Diagnostics.}
We track the suboptimality gap $\Delta_t:=\rho(\pi^t)-\rho^\star$ and the
stationary-distribution mismatch
$\chi_t:=\|d^{\pi^\star}/d^{\pi^t}\|_\infty$ with $\pi^\star$ a fixed
reference optimal policy; the ``tail'' estimate of $\chi$ is the median of
$\chi_t$ over the last $20\%$ of the run. On suboptimality plots we overlay
the non-asymptotic upper bound $(1-1/\bar\chi)^t\Psi_0$ from
Lemma~\ref{lem:master}(i) and the asymptotic-rate reference
$(1-1/\chi)^t\Psi_0$ from Lemma~\ref{lem:master}(ii); the prefactor
$\Psi_0=\Delta_0+D_0^\star/(\eta_{-1}\bar\chi)$ is the initial Lyapunov
value, with $D_0^\star=D^{d^\star}(\pi^\star,\pi^0)$ the state-weighted KL
divergence and $\eta_{-1}=\eta_0/b$ the largest step size consistent with
the non-asymptotic hypothesis at $t=0$.

\begin{figure}[t]
  \centering
  \includegraphics[width=0.98\linewidth]{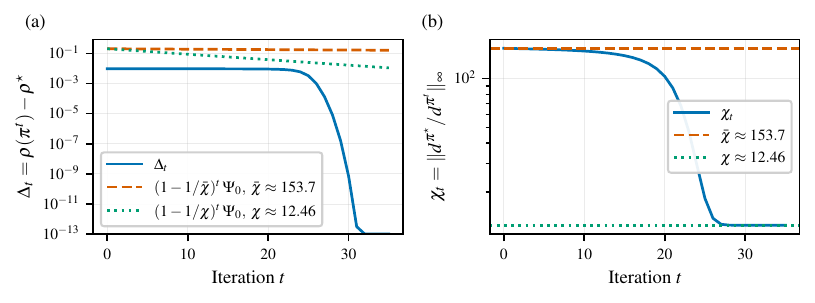}
  \caption{Exact KL-PMD on the multiple-minimizer instance.
  (a)~Suboptimality $\Delta_t$ with the non-asymptotic upper bound
  $(1-1/\bar\chi)^t\Psi_0$ (Lemma~\ref{lem:master}(i)) and asymptotic-rate
  reference $(1-1/\chi)^t\Psi_0$ (Lemma~\ref{lem:master}(ii)) overlaid.
  (b)~Mismatch $\chi_t$ with horizontal references at $\bar\chi\approx 153.7$
  and the tail estimate $\chi\approx 12.5$. The observed contraction is
  strictly tighter than both theoretical rates.}
  \label{fig:app:pmd-multi}
\end{figure}

\begin{figure}[t]
  \centering
  \includegraphics[width=0.98\linewidth]{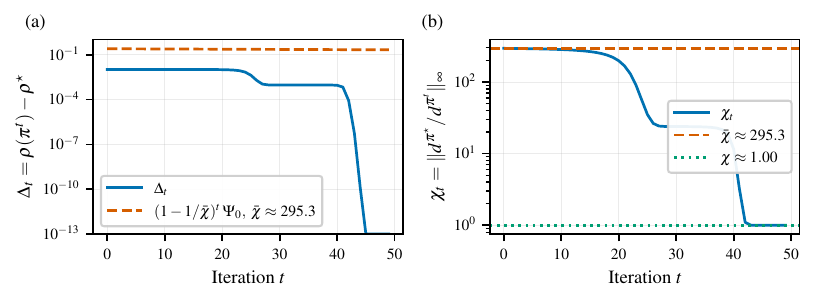}
  \caption{Exact KL-PMD on the unique-minimizer instance
  ($\varepsilon=10^{-3}$). Same diagnostics as
  Figure~\ref{fig:app:pmd-multi}. Here $\chi_t\to 1$ (tail $\chi\approx 1$)
  so the asymptotic reference collapses to zero and is omitted from
  panel~(a); the observed trajectory exhibits the two-phase convergence
  predicted by Corollary~\ref{cor:exact-pmd}, passing through an intermediate
  plateau at the tie-breaker scale before collapsing below $10^{-12}$.}
  \label{fig:app:pmd-unique}
\end{figure}

\subsection{Exact PMD: the asymptotic rate is governed by \texorpdfstring{$\chi$}{chi}, not \texorpdfstring{$\bar\chi$}{chi-bar}}
\label{ap:exp-exact}

The step-size schedule is the \emph{mixed super-geometric}
\begin{equation} \label{eq:app-schedule}
\eta_t \;=\; \eta_0\, b^{\,t}\, \alpha^{\,t^2},
\qquad \eta_0=0.1,\quad b=1.05,\quad \alpha=1.005,
\end{equation}
whose ratio $\eta_t/\eta_{t-1}=b\,\alpha^{2t-1}$ is bounded below by $b$ for
all $t\ge 1$, with the manual initialization $\eta_{-1} := \eta_0/b$ giving $\eta_0/\eta_{-1}=b$ at $t=0$ (satisfying the floor condition of
Lemma~\ref{lem:master}(i)) and diverges as $t\to\infty$ (satisfying
Lemma~\ref{lem:master}(ii)). A plain geometric schedule
$\eta_t=\eta_0 b^t$ has constant ratio and only satisfies the former.

Figures~\ref{fig:app:pmd-multi} and~\ref{fig:app:pmd-unique} illustrate the
two regimes of Corollary~\ref{cor:exact-pmd}. In the multiple-minimizer
instance (Figure~\ref{fig:app:pmd-multi}), PMD quickly collapses onto action $0$
in state $2$, but the action choices in states $0$ and $1$ are unconstrained
at optimality, so the iterates settle at some optimum that need not match
the reference $\pi^\star$ used to compute $\chi_t$. Consequently $\chi_t$
stabilizes at a tail value $\chi\approx 12.5$, much smaller than the
transient worst case $\bar\chi\approx 153.7$, and the observed $\Delta_t$
decays strictly below both the non-asymptotic and the asymptotic reference
lines. In the unique-minimizer instance (Figure~\ref{fig:app:pmd-unique}),
the tie-breaker forces $\pi^t\to\pi^\star$ and hence $\chi_t\to 1$, driving
the asymptotic rate $1-1/\chi\to 0$: Corollary~\ref{cor:exact-pmd}'s
superlinear regime. Panel~(a) exhibits the corresponding two-phase
convergence, with $\Delta_t$ first settling at the tie-breaker scale and
then collapsing superlinearly once $\chi_t$ approaches $1$. In both
instances the observed rate is at least as fast as the asymptotic rate
predicts.

\subsection{Inexact PMD: the error floor is independent of the step size}
\label{ap:exp-inexact}

We now replace the exact differential value with
$\widehat Q^t(s,a)=Q^{\pi^t}(s,a)+\xi^t(s,a)$, where
$\xi^t(s,a)\stackrel{\text{i.i.d.}}{\sim}\mathrm{Unif}([-\varepsilon_Q,\varepsilon_Q])$
so that $\|\widehat Q^t-Q^{\pi^t}\|_\infty\le\varepsilon_Q$, exactly
satisfying Assumption~\ref{as:uniform-critic-error}. The step-size schedule
is the plain geometric $\eta_t=\eta_0 b^t$ with $b=1.05$ (above the
floor $\bar\chi/(\bar\chi-1)\approx 1.003$ required by
Lemma~\ref{lem:master}(i)). We use the unique-minimizer instance with the
tie-breaker reduced to $\varepsilon=10^{-4}$ so that all three noise regimes
considered below are non-trivial.

Figure~\ref{fig:app:pmd-inexact}(a) directly addresses the concern flagged
in Section~\ref{sec:preliminaries}: aggressive step sizes do not amplify
critic errors at the floor. Fixing $\varepsilon_Q=10^{-2}$ and sweeping the
initial step size $\eta_0$ across two orders of magnitude, the three
trajectories enter the rate-dominated phase at very different iterations
but settle to an indistinguishable tail value around $3\times 10^{-5}$. Step
size controls the rate, not the floor. Figure~\ref{fig:app:pmd-inexact}(b)
fixes $\eta_0$ and sweeps $\varepsilon_Q$ across the same range; the tail
floors are monotonically ordered by the noise magnitude, consistent with
Theorem~\ref{thm:avg-pmd-linear-inexact-limsup}. The separation is modest
because the three-state MDP admits a deterministic optimum: once the policy
is sharply concentrated, the only pathway for critic noise to contribute to
$\Delta_t$ is by flipping the tie-broken actions in states $0$ and $1$, a
discrete event whose frequency saturates once $\varepsilon_Q$ exceeds the
tie-breaker scale. The theoretical upper bound
$2(\bar\chi+1)\varepsilon_Q$ is loose for this small MDP but its qualitative
dependence on $\varepsilon_Q$ is visible.

\begin{figure}[t]
  \centering
  \includegraphics[width=0.98\linewidth]{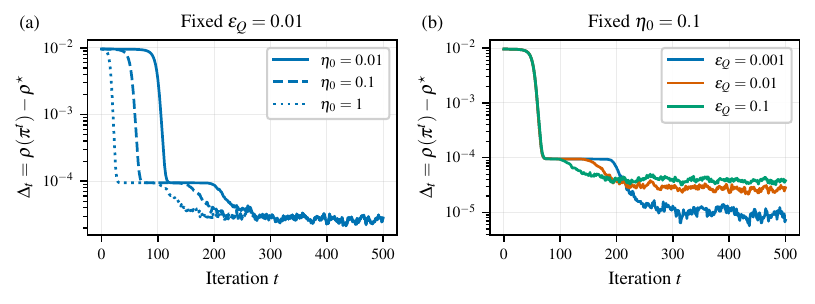}
  \caption{Inexact KL-PMD under uniform critic noise, on the unique-minimizer
  instance with tie-breaker $\varepsilon=10^{-4}$; each curve is the mean of
  $N=100$ seeds.
  (a)~Step-size independence. At fixed $\varepsilon_Q=10^{-2}$, three
  geometric schedules ($b=1.05$, $\eta_0\in\{10^{-2},10^{-1},1\}$)
  converge at very different rates but to essentially the same tail floor.
  (b)~Floor depends on critic error. At fixed $\eta_0=10^{-1}$, sweeping
  $\varepsilon_Q\in\{10^{-3},10^{-2},10^{-1}\}$ gives monotonically ordered
  tail floors.}
  \label{fig:app:pmd-inexact}
\end{figure}

\subsection{Empirical $\chi \ll \bar\chi$ across diverse MDPs} \label{ap:exp-rate-gap}

The 3-state experiments of Sections~\ref{ap:exp-exact}
and~\ref{ap:exp-inexact} dissect the $\chi$-vs-$\bar\chi$ separation in a
controlled setting where the Poisson equation, the optimal policy, and
both rate constants are computable in closed form. This subsection
shows that the same separation arises across considerably more diverse
problems. The figure referenced in the main body
(Figure~\ref{fig:rate-gap-main}) is the result of the experiment
described here.

\paragraph{MDP setups.} We use four MDP families, chosen to span structural variety. Three are ergodic, so $\chi_t$ is finite by Lemma~\ref{lem:uniform-d-bound}. Access-control is unichain and so also illustrates the extension of Appendix~\ref{ap:mixing-assumption-details}: the PMD iterates keep every action probability positive, which keeps $\chi_t$ finite, and the path bound $\bar\chi$ is verified along each run, as that extension requires.
\begin{itemize}[leftmargin=*]
\item \emph{Random Dirichlet MDPs} ($|\mathcal S|=20$, $|\mathcal A|=4$):
for each $(s,a)$, the next-state distribution $P(\cdot\mid s,a)$ is
sampled from Dirichlet$(0.1)$ (yielding full-support kernels) and the
cost $c(s,a)$ is sampled uniformly on $[0,1]$. We sample $80$ such MDPs and run PMD on each from the worst-cost policy, which plays the most costly action in every state, with zero probabilities floored at $10^{-300}$ so that the KL update and $D_0^\star$ are well defined. The variability across the ensemble then comes from the sampled MDPs themselves. This adversarial start gives these well-mixing MDPs a sizeable transient mismatch; it is also why their trajectories stay flat for the first $\sim\!20$ iterations, the time the growing step sizes need to lift the floored probabilities.
\item \emph{Access-control queuing} \citep[Example~10.2]{sutton_reinforcement_2020}:
$10$ identical servers, customers with priorities $\{1,2,4,8\}$
arriving uniformly, each busy server completes with probability $0.06$
per step; the agent admits or rejects each arrival. We use a single
fixed instance and run PMD from $80$ random Dirichlet$(0.5)$ initial
policies.
\item \emph{$5\times 5$ GridWorld}: a slip-prone gridworld
(slip probability $0.1$) in cost-minimization form, with the
goal-state transition acting as a teleport back to the start so the
chain is ergodic. We use a single fixed instance and run PMD from
$80$ random Dirichlet$(0.5)$ initial policies.
\item \emph{Jack's car rental} \citep[Example~4.2]{sutton_reinforcement_2020}:
the canonical two-location car-rental MDP with $441$ states and $11$
overnight-move actions, instantiated via the \texttt{gym-classics}
\texttt{JacksCarRental-v0} environment \citep{daley2021gym}. The Poisson rentals/returns
make every transition full-support, so the chain is ergodic under every
policy without any smoothing. We run PMD from $80$ random
Dirichlet$(0.5)$ initial policies.
\end{itemize}

\paragraph{Schedule and algorithm.} All four ensembles use the mixed
super-geometric schedule~\eqref{eq:app-schedule} with $\eta_0 = 10$,
$b = 1.1$ and $\alpha = 1.001$, again with $\eta_{-1} := \eta_0/b$. Its ratio
$\eta_t/\eta_{t-1}$ diverges, as Lemma~\ref{lem:master}(ii) requires, and is
at least $b$, so, as with doubling (Section~\ref{sec:master}),
Lemma~\ref{lem:master}(i) holds with $\bar\chi$ replaced by
$\max\{b/(b-1),\bar\chi\} = \max\{11,\bar\chi\}$; the ensembles below have
median $\bar\chi$ between $14$ and $372$. The KL-PMD update is
computed in log space, so the large late steps
($\eta_{60}\approx 1.1\times 10^5$) cause no overflow. We use the exact
differential value $Q^{\pi^t}$ at every iteration via the Poisson solve
described in Section~\ref{ap:exp-setup}, so this is the exact-PMD setting of
Corollary~\ref{cor:exact-pmd} (no critic noise).

\paragraph{Diagnostic.} For each run we record the
stationary-distribution mismatch
$\chi_t = \|d^{\pi^\star}/d^{\pi^t}\|_\infty$ over $T=60$ iterations
and report the \emph{normalized trajectory} $\chi_t/\bar\chi$, where
$\bar\chi = \max_{0\le t\le T}\chi_t$ is the per-run peak. The tail value $\chi$ of a run is the median of $\chi_t$ over its final $12$ iterations. This
normalization makes the separation between $\bar\chi$ (the constant
driving the non-asymptotic rate) and the asymptotic constant $\chi$
visually comparable across MDPs of very different absolute scales.
Figure~\ref{fig:rate-gap-main} plots the per-ensemble median of
$\chi_t/\bar\chi$ with the inter-quartile range shaded.

\paragraph{Result.} Across all four families the median trajectory
collapses by roughly $10\text{--}1000\times$ within the first
$\sim 20\text{--}40$ iterations and remains essentially flat thereafter. The
transient $\bar\chi$ regime that drives
Theorem~\ref{thm:avg-pmd-linear-inexact-limsup}'s non-asymptotic bound
is therefore a brief warm-up across problem types of widely varying
state-action sizes, transition structures, and optimal-policy
geometry; the asymptotic regime of Corollary~\ref{cor:exact-pmd}, where
$\chi$ governs convergence, is the operationally relevant one.

\paragraph{Per-MDP convergence and mismatch.}
Figure~\ref{fig:app:rate-gap-panels} expands the body figure into the
diagnostic format of Figures~\ref{fig:app:pmd-multi}
and~\ref{fig:app:pmd-unique}: for each MDP family we plot the median
suboptimality $\Delta_t$ (panel~(a)) and the median mismatch $\chi_t$
(panel~(b)), with the inter-quartile range over the ensemble shaded
in both. Panel~(a) overlays the non-asymptotic rate reference
$(1-1/\bar\chi)^t\Psi_0$ (Lemma~\ref{lem:master}(i)) and the
asymptotic-rate reference $(1-1/\chi)^t\Psi_0$
(Lemma~\ref{lem:master}(ii)), with
$\Psi_0:=\Delta_0+D_0^\star/(\eta_{-1}\bar\chi)$ the initial Lyapunov
value, evaluated at the ensemble-median $\bar\chi$ and tail $\chi$
(the asymptotic reference is omitted when $\chi\!\approx\!1$ collapses it to
zero, the superlinear regime of Corollary~\ref{cor:exact-pmd}); panel~(b)
overlays horizontal references at the same $\bar\chi$ and tail $\chi$
values. The observed $\Delta_t$ is below the non-asymptotic reference in every panel. The asymptotic reference must be read differently: Corollary~\ref{cor:exact-pmd} is a limit statement, $\limsup_t \Delta_t^{1/t} \le 1 - 1/\chi$, which constrains the tail slope of $\Delta_t$ and asserts nothing at any finite $t$. In three of the four families $\chi \approx 1$, and the median $\Delta_t$ reaches the numerical floor of $10^{-15}$ within $40$ iterations. In GridWorld, where $\chi \approx 7.5$, the median $\Delta_t$ stays below the asymptotic reference throughout and reaches the floor at $t = 58$, after which floating-point error dominates. In all four cases the empirical decay is significantly faster than the worst-case rate of Theorem~\ref{thm:avg-pmd-linear-inexact-limsup}.

\begin{figure}[H]
  \centering
  \includegraphics[width=0.88\linewidth]{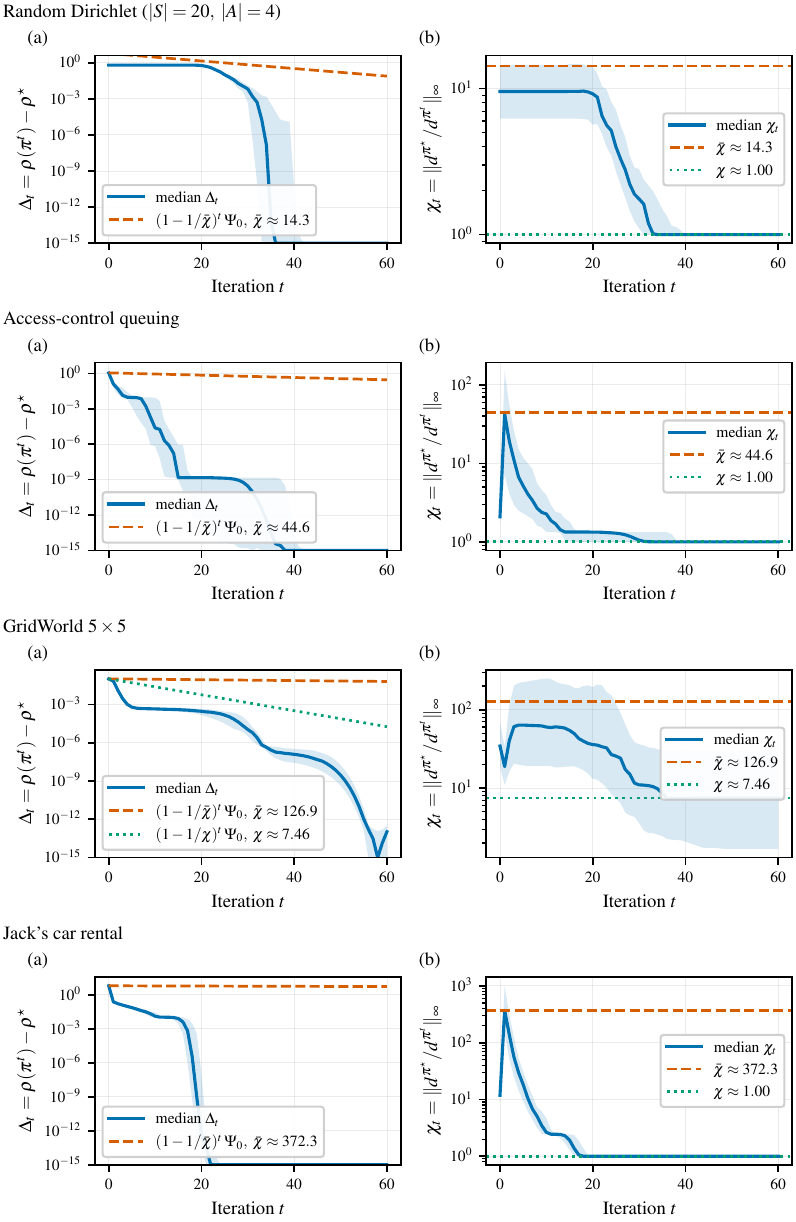}
  \caption{Per-MDP diagnostics for the rate-gap experiment (Section~\ref{ap:exp-rate-gap}): panel~(a) median $\Delta_t$ with the non-asymptotic rate reference $(1-1/\bar\chi)^t\Psi_0$ (Lemma~\ref{lem:master}(i)) and the asymptotic-rate reference $(1-1/\chi)^t\Psi_0$ (Lemma~\ref{lem:master}(ii)), both at ensemble-median values, where $\Psi_0:=\Delta_0+D_0^\star/(\eta_{-1}\bar\chi)$. The asymptotic reference is a limit statement plotted from $t=0$, not a finite-time bound. Panel~(b) median $\chi_t$ with horizontal references at $\bar\chi$ and tail $\chi$. IQR shaded.}
  \label{fig:app:rate-gap-panels}
\end{figure}

\end{document}